\documentclass[Afour,sageh,times]{sagej}

\usepackage{moreverb,url}

\usepackage[colorlinks,bookmarksopen,bookmarksnumbered,citecolor=green,urlcolor=blue]{hyperref}

\usepackage{graphicx,subfigure}
\usepackage{amsmath}
\newcommand{\norm}[1]{\left\lVert#1\right\rVert}
\usepackage{amsfonts}

\usepackage{amsthm}
\newtheorem{theorem}{Theorem}
\newtheorem{mydef}{Definition}
\newtheorem{assumption}{Assumption}
\usepackage{chngcntr}
\usepackage{apptools}
\AtAppendix{\counterwithin{lemma}{section}}

\newtheorem{lemma}{Lemma}
\setcitestyle{square}
\usepackage{mathtools}
\usepackage{amssymb}
\usepackage{balance}

\usepackage[font=small,labelfont=bf,  justification=justified,  format=plain]{caption}
\newcommand\BibTeX{{\rmfamily B\kern-.05em \textsc{i\kern-.025em b}\kern-.08em
T\kern-.1667em\lower.7ex\hbox{E}\kern-.125emX}}

\def\volumeyear{2016}
\usepackage{cleveref}
\crefname{figure}{Fig.}{Fig.}
\Crefname{figure}{Figure}{Figure}
\crefname{equation}{}{}
\Crefname{equation}{Equation}{Equation}
\Crefname{subsection}{subsection}{subsections}
\begin{document}
\runninghead{Grover, Liu and Sycara}

\title{The Before, During and After of  Multi-Robot Deadlock}

\author{Jaskaran Grover\affilnum{1}, Changliu Liu \affilnum{1} and Katia Sycara\affilnum{1}}

\affiliation{\affilnum{1}The authors are with the Robotics Institute, Carnegie Mellon University, Pittsburgh, PA 15213, USA. Email: {\tt\small \{jaskarag, cliu6,sycara\}@andrew.cmu.edu}. This work was funded by the DARPA Cooperative Agreement No.: HR00111820051, AFOSR awards FA9550-18-1-0097 and FA9550-15-1-0442.}

\corrauth{Jaskaran Grover (jaskarag@cs.cmu.edu)}

\begin{abstract}
Collision  avoidance for multirobot systems is a well studied problem. Recently, control barrier functions (CBFs) have been proposed for synthesizing controllers that guarantee collision avoidance and goal stabilization for multiple robots. However, it has been noted that reactive control synthesis methods (such as CBFs) are prone to \textit{deadlock}, an equilibrium of system dynamics that causes the robots to stall before reaching their goals. In this paper, we analyze the closed-loop dynamics of robots using CBFs, to characterize controller parameters, initial conditions and goal locations that invariably lead the system to deadlock.  Using tools from duality theory, we derive geometric properties of robot configurations of an $N$ robot system once it is in deadlock and we justify them using the mechanics interpretation of KKT conditions. Our key deductions are that 1) system deadlock is characterized by a force-equilibrium on robots and 2) deadlock occurs to ensure safety when safety is at the brink of being violated. These deductions allow us to interpret deadlock as a subset of the state space, and we show that this set is non-empty and located on the boundary of the safe set. By exploiting these properties, we analyze the number of admissible robot configurations in deadlock and develop a provably-correct decentralized algorithm for deadlock resolution to safely deliver the robots to their goals. This algorithm is validated in simulations as well as experimentally on Khepera-IV robots.
\end{abstract}
\maketitle
\section{Introduction}
Multirobot systems have been studied extensively for solving a variety of complex tasks such as target search \cite{kantor2003distributed}, \cite{grover2020parameter},  sensor coverage \cite{cortes2004coverage}, environmental exploration \cite{burgard2005coordinated} and perimeter guarding \cite{feng2019optimal}. Global coordinated behaviors result from executing local control laws on individual robots interacting with their neighbors \cite{ogren2002control}, \cite{olfati2007consensus}. Typically, the local controllers running on these robots are a combination of a task-based controller responsible for completion of a primary objective and a reactive collision avoidance controller for ensuring safety. However, augmenting the primary task based control with a hand-engineered safety control no longer guarantees that the original task will be satisfied \cite{borrmann2015control}. 
This problem becomes all the more pronounced when the number of robots increases. Motivated by this drawback, our paper focuses on an algorithmic analysis of the performance-safety trade-offs that result from augmenting a task-based controller with collision avoidance constraints as done using CBF based quadratic programs (QPs) \cite{ames2017control}. Although CBF-QPs mediate between safety and performance in a rigorous way, yet ultimately they are distributed local controllers. Such approaches exhibit a lack of look-ahead, which causes the robots to be trapped in \textit{deadlocks} as noted in \cite{petti2005safe,o1989deadlockfree,wang2017safety}. 

In deadlock, the robots stop while still being away from their goals and persist in this state unless intervened.
\begin{figure}[t]
    \centering
    \includegraphics[trim={5.5cm 17.5cm 7.7cm 2.8cm},clip,width=.7\linewidth]{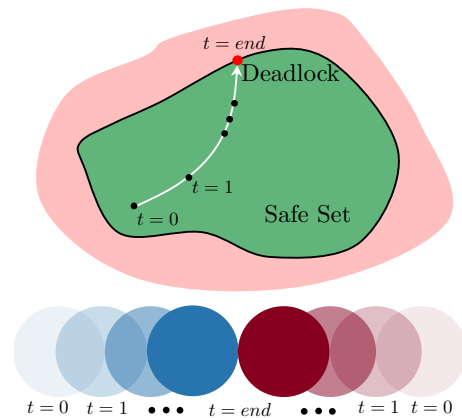}
    \setlength{\belowcaptionskip}{0pt}
    \caption{(lower) In the work space, two robots moving towards each other fall in deadlock. (upper) In the state space, the system state converges to the boundary of safe set.}
    \label{fig:deadlock_cartoon}
\end{figure}
This occurs because robots reach a state where conflict becomes inevitable, \textit{i.e.} a control favoring goal stabilization will violate safety (see the red dot in \cref{fig:deadlock_cartoon}). Hence, the only feasible strategy is to remain static. Although small perturbations can steer the system away from deadlock, there is no guarantee that deadlock will not relapse. To circumvent these issues, this work addresses the following technical questions:
\begin{enumerate}
    \item What initial conditions/controller parameters invariably lead a system to deadlock?
    \item What are the salient geometric properties of a system in deadlock?
    \item What are all the admissible configurations of robots in deadlock?
    \item How can we leverage this information to provably exit deadlock using decentralized controllers?
\end{enumerate}
To address these questions, we first recall technical definitions for CBF-QPs in \cref{CBFReview}. We divide our analysis intro three sections consistent with the chronology of deadlock incidence. 
\begin{enumerate}
    \item \textbf{Before Deadlock} (\cref{before deadlock}) In this section, we analyze the closed-loop dynamics of the multirobot system before the incidence of deadlock to characterize initial conditions and controller parameters that result in deadlock. We show that geometric symmetry in initial conditions and goals locations necessarily results in deadlock and that heterogeneity in robots' controller parameters is not sufficient to avoid deadlock.
    \item \textbf{During Deadlock} (\cref{in deadlock}) In this section, we use duality theory to derive geometric properties of robots' configurations once the system is in deadlock.  We propose a novel set-theoretic definition of deadlock to interpret it as a subset of the system's state space with an eye towards devising mitigative controllers to exit this set. We show that system deadlock is  characterized by a force equilibrium on robots and that the deadlock set is non-empty and located on the boundary of the system's safe set. Using graph enumeration, we show that the number of geometric configurations of robots that are admissible in deadlock increases combinatorially with the number of robots.
    \item \textbf{After Deadlock} (\cref{deadlock resolution}) In this section, we design provably-correct decentralized controllers to make the robots exit the deadlock set. We demonstrate this strategy on two and three robots in simulation, and experimentally on Khepera-IV robots.
\end{enumerate}
Finally, we conclude in \cref{conclusions} with a summary of our work and directions for future research questions. \\ \\ 
\textbf{Remark 1:} In the narrative, we follow the technical statement of a theorem/lemma with an intuitive explanation to convey the high-level idea. The proofs are deferred to the appendix to sustain the flow of essential ideas in the main body of the paper.
\section{Prior Work}
\label{PriorWork}
Several existing  methods provide  inspiration  for  the  results  presented  here. Among these, two are especially relevant: in the first category, we describe prior methods for collision avoidance and in the second, we focus on deadlock resolution.
\subsection{Prior Work on Avoidance Control}
Avoidance control is a well-studied problem with immediate applications for planning collision-free motions for multirobot systems. Classical avoidance control assumes a worst case scenario with no cooperation between robots \cite{leitmann1977avoidance,leitmann1983note}
Cooperative collision avoidance is explored in \cite{stipanovic2007cooperative,hokayem2010coordination} where avoidance control laws are computed using value functions.  Velocity obstacles have been proposed in \cite{fiorini1998motion} for motion planning in dynamic environments. They select avoidance maneuvers outside of robot's velocity obstacles to avoid static and moving obstacles by means of a tree-search. While this method is prone to undesirable oscillations, the authors in \cite{van2008reciprocal,van2011reciprocal,wilkie2009generalized} proposed reciprocal velocity obstacles that are immune to such oscillations. Recent work by
\cite{douthwaite2018comparative} gave a comparative analysis of the variants of velocity obstacles and empirically showed incidence of deadlock in symmetric situations. Additionally, control barrier function based controllers have been used in \cite{borrmann2015control,wang2017safety} to mediate between safety and performance using QPs. \cite{pan2020augmenting} presents a multi-robot feedback control
policy augmented with a global planner for robust safe navigation in a
complex workspace. 
\subsection{Prior Work on Deadlock Resolution}
The importance of coordinating motions of multiple robots while simultaneously ensuring safety, performance and deadlock prevention has been acknowledged in works as early as in \cite{o1989deadlockfree}. Here, authors proposed scheduling algorithms to asynchronously coordinate motions of two manipulators to ensure that their trajectories remain collision free and deadlock free.  In the context of mobile robots, \cite{yamaguchi1999cooperative} identified the presence of deadlocks in a cooperative  scenario using mobile robot troops. To the best of our knowledge, \cite{jager2001decentralized} were the first to propose algorithms for deadlock resolution specifically for multiple mobile robots. Their strategy for collision avoidance modifies planned paths by inserting idle times and resolves deadlocks by asking the trajectory planners of each robot to plan an alternative trajectory until deadlock is resolved. \cite{4414232} show that without sharing any information
on the global situation of the multi-robot system, deadlock cannot be avoided. Further, they proposed a random move + priority-based rule for deadlock resolution.  Authors in \cite{li2005motion} proposed coordination graphs to resolve deadlocks in robots navigating through narrow corridors. \cite{wang2017safety,rodriguez2016guaranteed} added perturbation terms to their controllers for avoiding deadlock.  Authors in \cite{douthwaite2019velocity,douthwaite2018comparative,douthwaite2019multi} gave empirical simulation based evidence of incidence of deadlock in symmetric situations and prevented deadlock occurrence by adding small perturbations, whereas our work gives analytical proofs why those geometric arrangements are prone to deadlock with extensions to heterogeneous robots. \cite{7875176,zhou2017collision} proposed a distributed algorithm to avoid collisions and deadlocks by stopping robots preemptively before these events occur. \cite{alonso2018reactive} presents an algorithm for mission and motion planning for small teams of robots performing a task around moving obstacles guaranteeing safety and providing an algorithm and conditions for deadlock resolution.

In these works, controllers for task-completion and collision avoidance are highly coupled. However, in our work, we assume that a nominal task based control is already given and must be followed as much as possible while maintaining safety using the framework of CBF-QPs. Thus, the mediation between safety and task performance is hierarchical with safety constraints taking precedence over task-satisfaction. Due to this nature, our analysis is generalizable to deadlocks resulting from an arbitrary task-based control and not restricted to just the go-to-goal task. Further, differently from prior work, we characterize analytical properties of the system dynamics to characterize the reasons behind deadlock incidence. We demonstrate that intuitive explanations for geometric properties of a system in deadlock are indeed recovered using duality. Moreover, our analysis can be extended to reveal bottlenecks of any optimization-based controller synthesis method such as velocity obstacles since the underlying formalism is obtained from duality theory. Additionally, to our knowledge, we are the first to use graph enumeration and KKT conditions to highlight the combinatorial nature of geometric configurations admissible in deadlock. For deadlock resolution, we do not consider additive perturbations since they lack formal guarantees. Instead, we exploit the derived geometric properties  to guide the design of a provably correct controller that ensures safety, performance and deadlock resolution. 













\section{Avoidance Control with CBFs: Review}
\label{CBFReview}
 In this section, we review CBF based QPs used for synthesizing controllers that mediate between safety (collision avoidance) and performance (goal-stabilization) for multirobot systems. We refer the reader to \cite{wang2017safety} for a comprehensive treatment on this subject, since our work builds on top of their approach. Assume we have $N$ robots, where each robot follows single-integrator dynamics (we refer the reader to \cite{grover2019deadlock} where we considered double-integrators):
\begin{align}
\dot{\boldsymbol{p}}_i = \boldsymbol{u}_i.
\end{align} 
Here $\boldsymbol{p}_i=(x_i,y_i)\in\mathbb{R}^2$ is the position of robot $i$, $\boldsymbol{u}_i\in\mathbb{R}^2$ is its velocity (\textit{i.e.} control input) and $i \in \{1,2,\cdots,N\} = [\mathbb{N}]$. In this paper, we ignore actuator limit constraints $\vert\boldsymbol{u}_i\vert \leq \alpha_i$ but the reader can refer to our related work \cite{grover2019deadlock} where those constraints were considered. The problem of goal stabilization with collision-avoidance requires that each robot $i$ must reach a goal position $\boldsymbol{p}_{d_i}$ while avoiding collisions with every other robot $j$, $\forall j\in [\mathbb{N}]\backslash i$. Assume that there is a user-prescribed proportional  controller $\hat{\boldsymbol{u}}_i(\boldsymbol{p}_i)=-k_{p_i}(\boldsymbol{p}_i-\boldsymbol{p}_{d_i})$ that generates movement towards goal. Here $k_{p_i}$ is the controller gain for robot $i$. For posing the collision-free requirement, a pairwise function is formulated that maps the joint state space of $i$ and $j$ to a real-valued safety index \textit{i.e.} $h:\mathbb{R}^2 \times \mathbb{R}^2\longrightarrow\mathbb{R}$:
\begin{align}
\label{hdef}
h_{ij}= \norm{\Delta \boldsymbol{p}_{ij}}^2-D_s^2,
\end{align}
where $D_s$ is a desired safety margin. Robots $i\mbox{ and }j$ are considered to be collision-free or safe if their positions $(\boldsymbol{p}_i,\boldsymbol{p}_j)$ are such that $h_{ij}(\boldsymbol{p}_i,\boldsymbol{p}_j)\geq0$ (\textit{i.e.} at-least $D_s$ distance apart). We define ``pairwise safe set" as the 0-level superset of $h_{ij}$ \textit{i.e.}  $\mathcal{C}_{ij}\coloneqq \{(\boldsymbol{p}_i,\boldsymbol{p}_j)\in\mathbb{R}^4\mid h_{ij}(\boldsymbol{p}_i,\boldsymbol{p}_j)\geq0$\}. The boundary of the pairwise safe set is \vspace{-0.2cm}
\begin{align}
\label{safetysetdef}
    \partial \mathcal{C}_{ij}=\{(\boldsymbol{p}_i,\boldsymbol{p}_j) \in \mathbb{R}^2 \vert h_{ij}(\boldsymbol{p}_i,\boldsymbol{p}_j)=0\}.
\end{align}
\begin{figure}[t]
	\setlength{\belowcaptionskip}{0pt}
	\centering     
	\subfigure[Positions ]{\label{fig:deadlock_pos}\includegraphics[width=\linewidth]{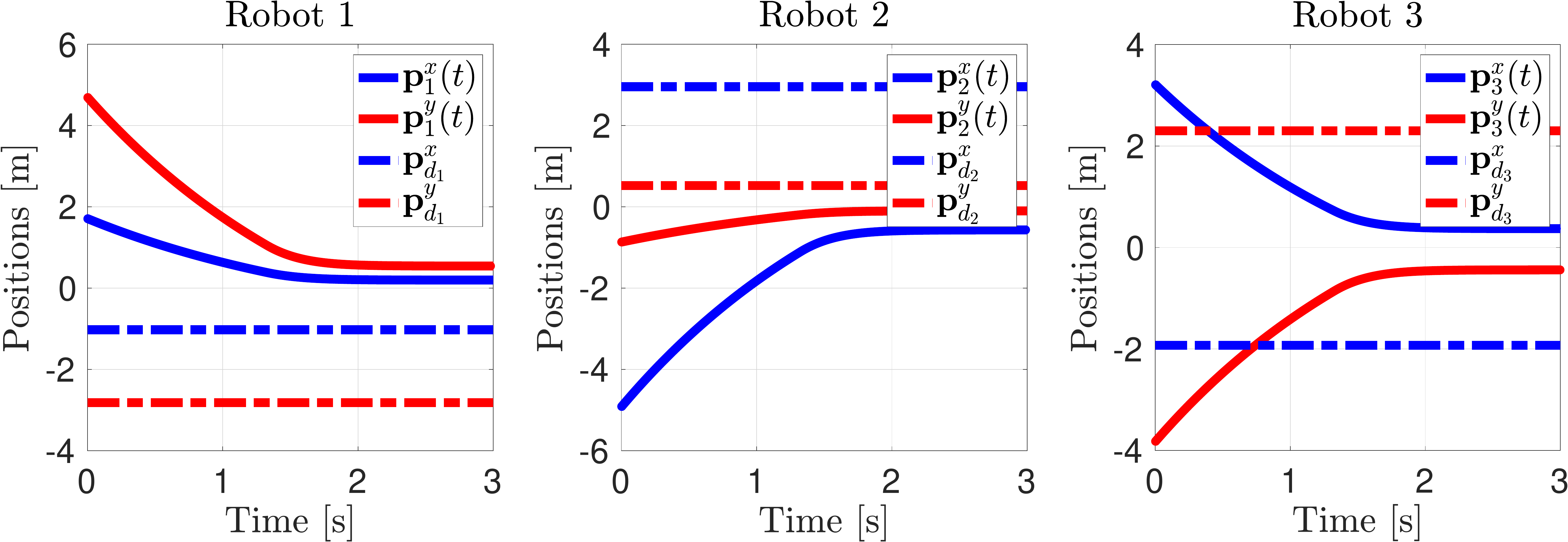}}
	\subfigure[Velocities (control inputs) ]{\label{fig:deadlock_acc}\includegraphics[width=\linewidth]{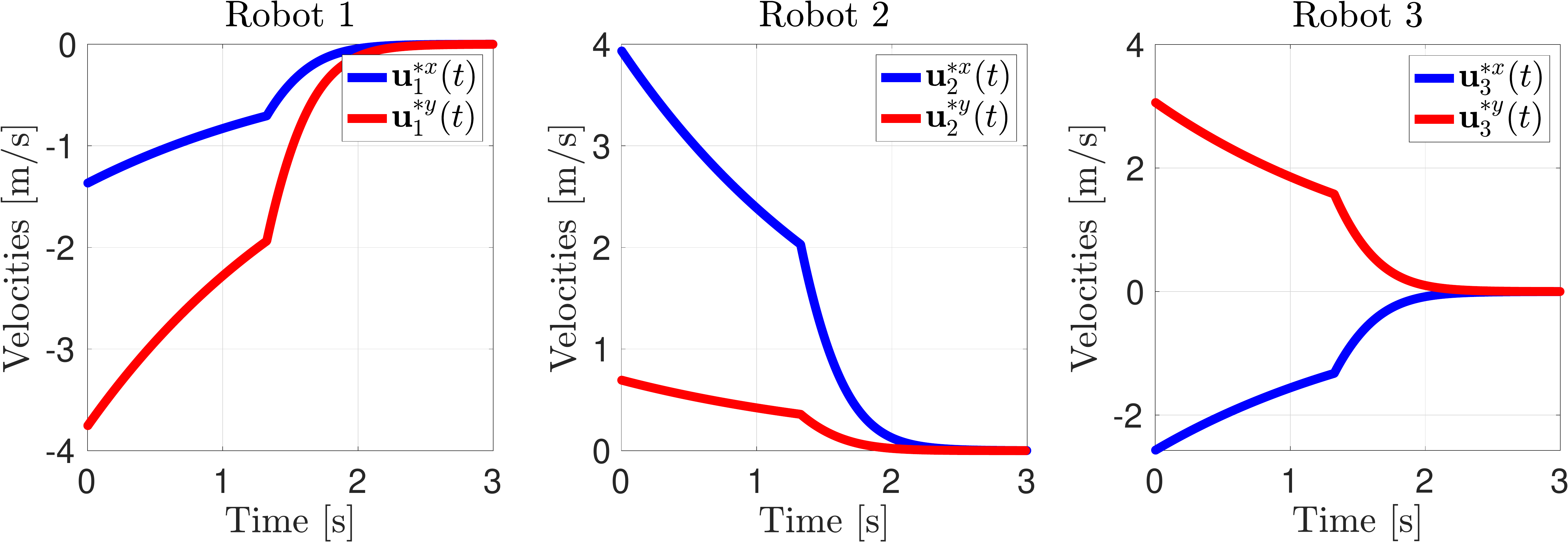}}
	\caption{Positions and velocities of robots falling in deadlock. Note that $\lim_{t\rightarrow \infty}p_{x,y}\neq p_{d_{x,y}}$ yet  $\lim_{t\rightarrow \infty}v_{x,y}=0$. Simulation results at   \url{https://youtu.be/FShai_JESII} and experimental demonstration at \url{https://youtu.be/e6eOeuh7Uec}.}
	\label{fig:deadlock_traj}
\end{figure}
\hspace{-0.2cm} In other words, the system state is on the boundary of the safe set when the robots are grazing each other \textit{i.e.} separated by the safety margin $D_s$. Assuming that initial positions of robots $i,j$ are collision-free \textit{i.e.} $h_{ij}(\boldsymbol{p}_i(0),\boldsymbol{p}_j(0))\geq0$, we would like to synthesize controls $\boldsymbol{u}_i ,\boldsymbol{u}_j$ that ensure their future positions are also collision-free \textit{i.e.} $h_{ij}(\boldsymbol{p}_i(t),\boldsymbol{p}_j(t))\geq0,\mbox{ }\forall t>0$. This can be achieved by ensuring that (\cite{ames2017control}): 
\begin{align}
\label{constraint}
\frac{dh_{ij}}{dt} \geq - \gamma h_{ij},
\end{align}
where $\gamma>0$ (a user-defined hyperparameter). Substituting  \eqref{hdef} in  \eqref{constraint}, we get 
\begin{align}
-\Delta\boldsymbol{p}^T_{ij}\Delta \boldsymbol{u}_{ij} \leq \frac{\gamma}{2}h_{ij}.
\end{align}
This constraint is distributed on robots $i \mbox{ and }j$ as:
\begin{align}
\label{decentralizedconstraints}
-\Delta\boldsymbol{p}_{ij}^T\boldsymbol{u}_i \leq  \frac{\gamma}{4}h_{ij} \mbox{ and }
\Delta\boldsymbol{p}_{ij}^T\boldsymbol{u}_j \leq  \frac{\gamma}{4}h_{ij}
\end{align}
Therefore, any $\boldsymbol{u}_i$ and $\boldsymbol{u}_j$ that satisfy  \cref{decentralizedconstraints} are guaranteed to ensure collision free trajectories for robots $i\mbox{ and }j$ in the multirobot system.  Since robot $i$ wants to avoid collisions with $N-1$ robots, there are $N-1$ collision avoidance constraints. To mediate between safety and goal stabilization objective, a QP is posed that computes a controller closest to the prescribed control $\hat{\boldsymbol{u}}_i(\boldsymbol{p}_i)$ and satisfies the $N-1$ constraints, as shown:
\begin{align}
\label{optimization_formulation_2}
	\begin{aligned}
		\boldsymbol{u}^*_i&= \underset{\boldsymbol{u}_i}{\arg\min}
		& & \norm{\boldsymbol{u}_i - \hat{\boldsymbol{u}}_i(\boldsymbol{p}_i)}^2 \\
		& \text{subject to}
		& & \boldsymbol{a}_{ij}^T\boldsymbol{u}_i \leq b_{ij}  \mbox{ } \forall j\in [\mathbb{N}]\backslash i
	\end{aligned}
\end{align}
where using  \eqref{decentralizedconstraints}, we define
\begin{align}
\label{abdefs}
\boldsymbol{a}_{ij} \coloneqq -\Delta \boldsymbol{p}_{ij} \mbox{,   }
    b_{ij} \coloneqq  \frac{\gamma}{4}h_{ij} =\frac{\gamma}{4}  (\norm{\Delta \boldsymbol{p}_{ij}}^2-D_s^2)
\end{align}
Each robot $i$ locally solves this QP to determine its $\boldsymbol{u}^*_i$, which guarantees collision avoidance of robot $i$ with $N-1$ robots while encouraging motion towards its goal. While provably safe, this approach does not guarantee that stabilizing to  goals will be accomplished because goal stabilization is expressed as a cost function unlike safety which is expressed as a hard constraint. As an example, \cref{fig:deadlock_traj} shows the result of  executing  \cref{optimization_formulation_2} on three robots. Notice from \cref{fig:deadlock_pos} that the positions of robots have converged, but \textbf{not} to their respective goals. However, the control inputs from \cref{optimization_formulation_2} (\cref{fig:deadlock_acc}) have already converged to zero. An experimental realization can be seen at \url{https://youtu.be/e6eOeuh7Uec} where three Khepera robots fall in deadlock after using these controllers. Manual intervention is needed to resolve deadlock.  Based on these observations, deadlock is defined as (\cite{grover2020does} ):
\begin{mydef} 
\label{def:deadlock_definition}
Robot $i$ is in deadlock if $\boldsymbol{u}^*_i=0$ and the prescribed nominal control $\hat{\boldsymbol{u}}_i \neq \boldsymbol{0} \iff \boldsymbol{p}_i \neq \boldsymbol{p}_{d_i}$
\end{mydef}
In simpler terms, this definition states that for a robot to be in deadlock, the output from the QP based controller is zero, even though the prescribed controller reports non-zero velocity because the robot is not at its intended destination. We do not analyze ``livelock" in this paper because in our simulations we did not observe incidence of oscillatory behavior. Therefore, the focus is on deadlock. In the next section, we examine the closed-loop dynamics of the system \textit{before deadlock} as a function of the robots' initial conditions, goal locations and controller parameters $k_{p_i}$ to characterize what causes deadlock.
\section{Before Deadlock}
\label{before deadlock}
In this section, we analytically investigate the solutions of \eqref{optimization_formulation_2} $\forall i \in \{1,2,\cdots,N\}$ because these solutions govern the instantaneous dynamics of the robots. However, owing to the fact that \eqref{optimization_formulation_2} is an  optimization problem, the controls $\boldsymbol{u}^*_i$ are not available as explicit function of the robots' states. This prohibits straightforward use of tools like Lyapunov theory to analyze long-term behavior of the robots' motion. To circumvent this issue, we use duality theory, specifically KKT conditions of \eqref{optimization_formulation_2} to derive an explicit expression for $\boldsymbol{u}^*_i$ as a function of the robots' positions. These conditions are necessary and sufficient for a global optimum of this QP. The Lagrangian is
\begin{align}
L(\boldsymbol{u}_i,\boldsymbol{\mu}_i) =  \norm{\boldsymbol{u}_i - \hat{\boldsymbol{u}}_i}^2_2  + \sum_{j \in [\mathbb{N}]\backslash i}\mu_{ij}(\boldsymbol{a}^T_{ij}\boldsymbol{u}_i-b_{ij}) \nonumber
\end{align}
Let $(\boldsymbol{u}^*_i,\boldsymbol{\mu}^*_i)$ be the optimal primal-dual solution to  \eqref{optimization_formulation_2}. The KKT conditions are \cite{boyd2004convex}:
\begin{enumerate}
	\item Stationarity: $\nabla_{\boldsymbol{u}_i}L(\boldsymbol{u}_i,\boldsymbol{\mu}_i)\vert_{(\boldsymbol{u}^*_i,\boldsymbol{\mu}^*_i)} = 0$
	\begin{align}
	\label{stationarity1}
	\implies \boldsymbol{u}^*_i = \hat{\boldsymbol{u}}_{i} - \frac{1}{2}\sum_{j\in [\mathbb{N}]\backslash i}\mu^*_{ij}\boldsymbol{a}_{ij}.
	\end{align}
\item Primal Feasibility 
\begin{align}
\label{primal_feasibility1}
\boldsymbol{a}^T_{ij}\boldsymbol{u}^*_i \leq b_{ij} \mbox{  }	 \forall j \in [\mathbb{N}]\backslash i
\end{align}
\item Dual Feasibility 
\begin{align}
\label{dual_feasibility1}
{\mu^*_{ij}} \geq 0 \mbox{  }	 \forall j \in [\mathbb{N}]\backslash i
\end{align}
\item Complementary Slackness 
\begin{align}
	\label{complimentarty slackness1}
	\mu^*_{ij} \cdot (\boldsymbol{a}^T_{ij}\boldsymbol{u}^*_i -b_{ij}) = 0 
	 \mbox{   }\forall j \in [\mathbb{N}]\backslash i
\end{align}
\end{enumerate}
We define the set of active and inactive constraints as
\begin{align}
\label{activeinactive}
	\mathcal{A}(\boldsymbol{u}^*_i) = \{[\mathbb{N}]\backslash i \mid \boldsymbol{a}^T_{ij}\boldsymbol{u}^*_i = b_{ij} \} \\
	\mathcal{IA}(\boldsymbol{u}^*_i) = \{[\mathbb{N}]\backslash i \mid \boldsymbol{a}^T_{ij}\boldsymbol{u}^*_i < b_{ij} \} 
\end{align}
\textcolor{black}{These two sets define a partition of the set of non-ego robots. As we will explain in \cref{in deadlock}, intuitively, the active set is those set of robots that the ego robot ``worries" about for a collision and therefore ``actively" repels to stay safe.} Using complementary slackness from \eqref{complimentarty slackness1}, we deduce
\begin{align}
\label{cs3}
	\mu^*_{ij} = 0 \mbox{ $\forall j $} \in \mathcal{IA}(\boldsymbol{u}^*_i)
\end{align}  
Therefore, we can restrict the summation in \eqref{stationarity1}  to only the set of active constraints \textit{i.e.}
\begin{align}
\label{kkt_general}
	\boldsymbol{u}^*_i = \hat{\boldsymbol{u}}_i - \frac{1}{2}\sum_{j \in \mathcal{A}(\boldsymbol{u}^*_i)}\mu^*_{ij}\boldsymbol{a}_{ij}
\end{align}
\textcolor{black}{This equation says that the overall control returned by \cref{optimization_formulation_2} is the sum of the task-based nominal control
$\hat{\boldsymbol{u}}_i$ and the resultant of collision-avoidance velocities from the active robots. While we state this as a fact here, we will prove this in \cref{in deadlock} using a geometric argument}. We will use this representation of controls to explore scenarios in which deadlock occurs. Our aim is to characterize the parameters and the geometry of robot arrangements which necessarily result in deadlock instead of characterizing every possible geometric arrangement where deadlock will occur, since this latter problem is combinatorially complex in the number of robots (see \cref{complexity}). Therefore, to convey the idea, we focus on the simpler cases of two and three robots. 
\subsection{Closed-loop Analysis for Two Robots}
\label{tworobotdeadlock}
Consider two robots positioned on the line connecting their goals at $t=0$ as shown in \cref{fig:deadlock_in_two_cartoon}. We will demonstrate that even if the prescribed control of one robot is more aggressive than the other, \eqref{optimization_formulation_2} will necessarily result in deadlock because of the collinearity of initial conditions and goals.
Let $\boldsymbol{p}_i(t) \in \mathbb{R}^2$ denote the position of robot $i \in \{1,2\}$ at time $t$. At $t=0$,
\begin{align}
\label{initial_positions_two_robots}
    \boldsymbol{p}_1(0) &= (x_0,y_0) \nonumber \\
    \boldsymbol{p}_2(0) &= \boldsymbol{p}_1(0) + D_{init}\hat{\boldsymbol{e}}_\alpha ,
\end{align}
where $\hat{\boldsymbol{e}}_{\alpha}$ is a unit vector oriented at $\alpha$ relative to $X_{w}$ axis and $D_{init} \in \mathbb{R}^{+}$ is the initial distance between the robots (see  \cref{fig:deadlock_in_two_cartoon}). The desired goal positions are
\begin{align}
\label{goal_positions_two_robots}
    \boldsymbol{p}_{d_1} &= \boldsymbol{p}_1(0) + D_{G_1}\hat{\boldsymbol{e}}_{\alpha} \nonumber \\
    \boldsymbol{p}_{d_2} &= \boldsymbol{p}_2(0) - D_{G_2}\hat{\boldsymbol{e}}_{\alpha}.
\end{align}
Here $D_{G_i}$ is the distance of  $i$ from its goal $\boldsymbol{p}_{d_i}$ at $t=0$. We assume that $D_{G_1},D_{G_2}>D_{init}$ to encourage incidence of deadlock. We formally state the deadlock incidence result as follows:
\begin{figure}
    \centering
    \includegraphics[trim={2.8cm 16.7cm 5.3cm 3.8cm},clip,width=.8\linewidth]{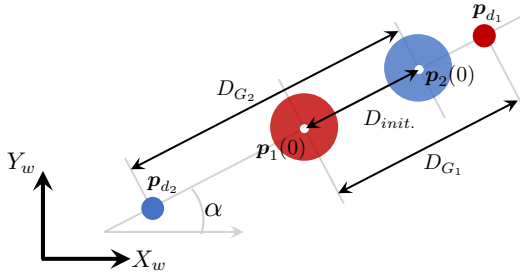}
    \caption{Geometric configuration of two robots and their goals. Notice that everything is collinear along $\hat{\boldsymbol{e}}_{\alpha}$.}
    \label{fig:deadlock_in_two_cartoon}
\end{figure}
\begin{theorem}
\label{tworobottheorem}
Given the initial positions of robots as in  \eqref{initial_positions_two_robots} and desired goals as in  \eqref{goal_positions_two_robots}, the controls generated by the QP in \eqref{optimization_formulation_2} will cause the robots to fall in deadlock $\forall D_{G_1},D_{G_2}>D_{init.}$ and $\forall k_{p_1},k_{p_2} \in \mathbb{R}^{+}$.
\end{theorem}
\begin{proof}
    We give the sketch of the proof to convey the overall idea and break it down into shorter lemmas.  Recall from \eqref{kkt_general} that the control input on robot $i$ at every time step, depends on the collision avoidance constraints that are active. Therefore, to begin with, we analyze the controls returned by \eqref{optimization_formulation_2} at $t=0$. In \textbf{Lemma \ref{lemma1}}, we show that if the initial distance between the robots is greater than a certain critical distance, then the collision avoidance constraints of both robots are inactive, and as a result $\boldsymbol{u}_i^*(0)=\boldsymbol{\hat{u}}_i(0) \forall i$. We then extend this analysis to future time $t>0$ and show that there exist three sequential phases of motion which are dependent on which robot's collision avoidance constraint is active/inactive:
\begin{enumerate}
    \item  Phase 1 corresponds to the duration in which the collision avoidance constraints of both robots are inactive. We show that this phase culminates in a finite time and the robots move closer to each other during this phase (\textbf{Lemma \ref{lemma2}}).
    \item Next, phase 2 begins where one robot's constraint becomes active while the other's is still inactive. This is due to the  heterogeneity $k_{p_1}\neq k_{p_2}$ and $D_{G_1}\neq D_{G_2}$. We show that phase 2 also culminates and the robots move further closer to each other  (\textbf{Lemma \ref{lemma3}}).
   \item Finally, phase 3 begins, where both constraints are active. We show that the distance between the robots converges to the safety margin $D_s$ and the robots stop moving (\textbf{Lemma \ref{lemma4}}), while still away from their goals, thus proving that they have fallen in \textit{deadlock}. \qed
\end{enumerate}
\end{proof}

For the special case of two robots, the control for robot $i$ can be written by adapting  \eqref{kkt_general},
\begin{align}
\label{generalu2robot}
\boldsymbol{u}^*_{i} &= \hat{\boldsymbol{u}}_{i} - \frac{1}{2}\mu^*_{ij}\boldsymbol{a}_{ij},
\end{align}
where $j \neq i$ and $i \in \{1,2\}$. In \cref{cs3}, we showed that the value of  Lagrange multiplier $\mu_{ij}$ depends on whether the collision avoidance constraint $\boldsymbol{a}_{ij}^T\boldsymbol{u}_i - b_{ij} \leq 0 $ is active ($=$) or inactive ($<$) at $\boldsymbol{u}_i = \boldsymbol{u}^*_i$. 
Thus, we have a check flag for evaluating the feasibility of a candidate control $\tilde{\boldsymbol{u}}_i$:
\begin{align}
    f_{ij}(\tilde{\boldsymbol{u}}_i)&=\boldsymbol{a}_{ij}^T\tilde{\boldsymbol{u}}_i -b_{ij},
\end{align}
$\forall i \in \{1,2\}$ and $j \neq i$. \textcolor{black}{If for a candidate control $\tilde{\boldsymbol{u}}_i$, $f_{ij}(\tilde{\boldsymbol{u}}_i)>0$, then $\tilde{\boldsymbol{u}}_i$ violates robot $i's$ constraint with $j$ in \cref{optimization_formulation_2}.} 
 \textcolor{black}{Let us check this flag for the nominal control $\boldsymbol{\hat{u}}_i$. Suppose that $f_{ij}(\boldsymbol{\hat{u}}_i) < 0$ \textit{i.e.} $\boldsymbol{\hat{u}}_i$ ensures that the constraint of $i$ with $j$ is feasible. Then, looking at \eqref{optimization_formulation_2}, notice that its optimum will be attained at $\boldsymbol{u}^*_i = \boldsymbol{\hat{u}}_i$ itself with zero cost (as long as there are only two robots in the system). This is because $\boldsymbol{\hat{u}}_i$ by itself satisfies the constraint and the cost function penalizes deviation from $\boldsymbol{\hat{u}}_i$. Further, $f_{ij}(\boldsymbol{u}^*_i)=f_{ij}(\boldsymbol{\hat{u}}_i)  < 0$ implies that the constraint of $i$ with $j$ is inactive.
}  
Let us look at the robots geometrically to see when this happens. Define critical distance for robot $i$ as (to be derived in \textbf{Lemma \ref{lemma1}}):
\begin{align}
\label{generaldcrit}
\beta^i_{+}=\frac{2D_{G_i}k_{p_i}}{\gamma} + \sqrt{\bigg(\frac{2D_{G_i}k_{p_i}}{\gamma}\bigg)^2 + D_s^2}  
 \end{align}
$\forall i \in \{1,2\}$. We now show that if the initial distance between the robots $i$ and $j$ \textit{i.e.} $D_{init.}$, is greater than $i$'s critical distance $\beta^i_{+}$, then $i's$ constraint with $j$ is inactive at $t=0$. 
\begin{lemma}
\label{lemma1}
If at $t=0$, $D_{init}>\beta^i_{+}$, then $\boldsymbol{u}^*_i(0) = \boldsymbol{\hat{u}}_i(0)$. 
\end{lemma}
\hspace{-0.4cm}\textbf{Intuition:} This result says that if at $t=0$, robot $i$ is far from $j$ (specifically at-least $\beta^i_+$ distance away), then $i$ can simply use it's prescribed nominal control $\boldsymbol{\hat{u}}_i$ at $t=0$, and not worry about a collision with $j$ at-least until the next instant. 
\begin{proof}
\textcolor{black}{See appendix (\ref{lemma1_appendix}).}
\end{proof}
\begin{assumption}
\label{ass2}
WLOG assume that $D_{init}>\beta^1_+>\beta^2_+$
\end{assumption}
From assumption \eqref{ass2}, and the result of \textbf{Lemma} \ref{lemma1}, it follows that the collision avoidance constraints of both robots are inactive at $t=0$. Thus, $\boldsymbol{u}^*_i(0) = \boldsymbol{\hat{u}}_i(0)$ for $i=\{1,2\}$. If we give these velocities to robots for a small time $\Delta t$, the distance between the robots and the critical distance will change. So at the next instant, we will compare the \textit{updated inter-robot distance}  with the \textit{updated critical distances} to decide whether $\boldsymbol{u}^*_i(\Delta t) = \boldsymbol{\hat{u}}_i(\Delta t)$ and so on. Therefore, one can assume that these constraints remain inactive for some finite time, which is precisely the duration of phase 1.
\subsubsection{\textbf{Phase 1:}}
\textcolor{black}{We define phase 1 as the period in which the constraints of both robots are inactive \textit{i.e.} both their flags are strictly negative when evaluated using their nominal controllers. In other words,}
\begin{align}
\label{constraints_time1}
    &f_{12}(\boldsymbol{\hat{u}}_1(t))=\boldsymbol{a}^T_{12}(t)\boldsymbol{\hat{u}}_1(t) - b_{12}(t) <0  \\
    &f_{21}(\boldsymbol{\hat{u}}_2(t))=\boldsymbol{a}^T_{21}(t)\boldsymbol{\hat{u}}_2(t) - b_{21}(t) <0  
\end{align}
Mathematically, the duration of phase 1 is given by,
\begin{align}
\label{time_def}
    t_1 \coloneqq \sup_{t>0}\{t|f_{12}(\boldsymbol{\hat{u}}_1(t))<0,f_{21}(\boldsymbol{\hat{u}}_2(t))<0\},
\end{align}
\textcolor{black}{\textit{i.e.} the maximum time until which the nominal controllers of both robots remain feasible. So until $t_1$, we have
\begin{align}
\label{robot12inactive}
   \dot{\boldsymbol{p}}_1 = \boldsymbol{u}^*_{1} = \hat{\boldsymbol{u}}_{1} \mbox{ and  } 
   \dot{\boldsymbol{p}}_2 =\boldsymbol{u}^*_{2} = \hat{\boldsymbol{u}}_{2}
\end{align}
Now given the initial positions and desired goals as in \cref{fig:deadlock_in_two_cartoon}, it cannot be the case that both robots continue to use their nominal controllers indefinitely  to move to their goals because that would result in a collision. Thus, a time will come when using the nominal control is no longer feasible for at-least one robot. This is precisely the time $t_1$ defined in Def. \eqref{time_def} and we prove its existence in the next lemma. }
\begin{lemma}
\label{lemma2}
$\exists$ a finite time $t_1$ as described in Def. \eqref{time_def}, until which the collision avoidance constraints of both robots are simultaneously inactive. 
\end{lemma}
\hspace{-0.4cm}\textbf{Intuition:} This result states that for sometime, both robots may continue to use their nominal controllers and move without ``worrying" about collisions with the other. Nevertheless, eventually at $t_1$, the inter-robot distance will fall below a critical threshold for one robot, at which point, its collision avoidance constraint will get activated. For example, in \cref{fig:phases}, the dark green curve (the inter-robot distance) begins to fall below the blue curve (the critical distance of robot one \cref{time_crit_robot_1}) at $t_1$. Thus, robot one can no longer use its nominal control starting from $t_1$.
\begin{proof}
\textcolor{black}{See appendix \ref{lemma2_appendix}. }
\end{proof}
The critical distance for robot one (blue curve in \cref{fig:phases}) is derived in the appendix \ref{lemma2_appendix} and is given by:
\begin{align}
\label{time_crit_robot_1}
\beta^1_{+}(t)=\frac{2D_{G_1}k_{p_1}e^{-k_{p_1}t}}{\gamma} + \sqrt{\bigg(\frac{2D_{G_1}k_{p_1}e^{-k_{p_1}t}}{\gamma} \bigg)^2 + D_s^2}    
\end{align}
Thus, phase 1 ends at $t_1$ marking the beginning of phase 2.  
\subsubsection{\textbf{Phase 2:}}
\textcolor{black}{We define this phase as the period in which the constraint of robot one has become active \textit{i.e.}  $\boldsymbol{a}^T_{12}\boldsymbol{u}^*_{1}=b_{12}$ while that of robot two is inactive \textit{i.e.} $\boldsymbol{a}^T_{21}\boldsymbol{u}^*_{2}<b_{21}$. So for $t\geq t_1$, the dynamics of the robots are}
\begin{align}
   &\dot{\boldsymbol{p}}_1=\boldsymbol{u}^*_{1} = \hat{\boldsymbol{u}}_{1}- \frac{1}{2}\mu_{12}\boldsymbol{a}_{12}\hspace{0.1cm},\mbox{where}\hspace{0.1cm} \mu_{12} = 2\frac{\boldsymbol{a}^T_{12}\boldsymbol{\hat{u}}_{1}-b_{12}}{\norm{\boldsymbol{a}_{12}}^2} \nonumber \\
   &\dot{\boldsymbol{p}}_2=\boldsymbol{u}^*_{2} = \hat{\boldsymbol{u}}_{2} \nonumber
\end{align}
\textcolor{black}{ Here we derived $\boldsymbol{u}^*_{1}$ using \cref{generalu2robot}.} With some calculation, we can show that the components of velocities in $\boldsymbol{u}^*_{\{1,2\}}$ perpendicular to $\hat{\boldsymbol{e}}_{\alpha}$ vanish. Now, owing to the inevitability of a collision should robot two choose to continue using $\hat{\boldsymbol{u}}_{2}$ forever, eventually, the constraint of robot two will also switch from inactive to active (\ref{lemma3_appendix}). Therefore, similar to the definition of $t_1$, we define the duration of phase 2 until the time for which the constraint of robot two stays inactive \textit{i.e.}
\begin{align}
\label{time_def2}
    t_2 \coloneqq \sup_{t>t_1}\{t|f_{21}(\boldsymbol{\hat{u}}_2(t))<0\},
\end{align}
\begin{figure}[t]
    \centering
    \includegraphics[trim={0cm 0.0cm 0cm 0cm},clip,width=1\linewidth]{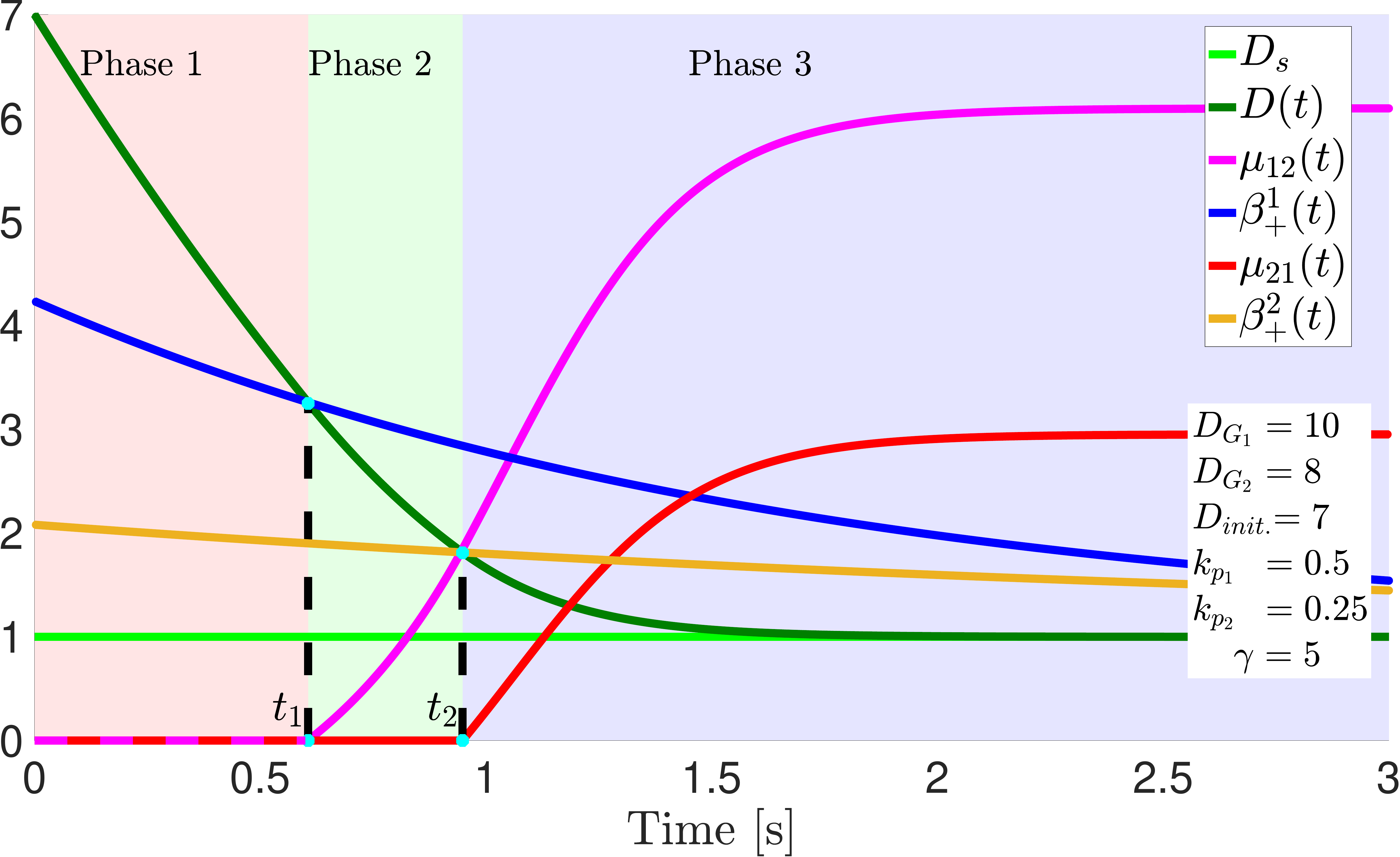}
    \caption{Simulation results for two-robot deadlock. Simulation parameters shown in bottom right panel. Note at $t=t_1$, $D(t)=\beta^1_+(t)$ and $\mu_{12}(t)$ switches on, similarly at $t=t_2$,  $\mu_{21}(t)$ switches on. Finally, $D(t) \longrightarrow D_s$.}
    \label{fig:phases}
\end{figure}
\begin{lemma}
\label{lemma3}
$\exists$ a finite time $t_2$ as in Def. \eqref{time_def2}, until which the constraint  of robot two stays inactive.
\end{lemma}
\textcolor{black}{\hspace{-0.4cm}\textbf{Intuition:} Eventually there will come a time $t_2$ when the inter-robot distance will fall just below the critical threshold of robot two as well, just the way it did for robot one at $t_1$. So at this time, robot two's collision avoidance constraint will become active (see \cref{fig:phases} where the dark green curve (the inter-robot distance) intersects the yellow curve (the critical distance of robot two \cref{time_crit_robot_2}). }
\begin{proof}
\textcolor{black}{See appendix \ref{lemma3_appendix}.}
\end{proof}
The critical distance for robot two (yellow curve in \cref{fig:phases}) is defined similar to \cref{time_crit_robot_1} as:
\begin{align}
\label{time_crit_robot_2}
\beta^2_{+}(t)=\frac{2D_{G_2}k_{p_2}e^{-k_{p_2}t}}{\gamma} + \sqrt{\bigg(\frac{2D_{G_2}k_{p_2}e^{-k_{p_2}t}}{\gamma} \bigg)^2 + D_s^2}    
\end{align}
\subsubsection{\textbf{Phase 3:}}
\textcolor{black}{This phase begins at $t_2$ and in this  phase, the constraints of both robots are active \textit{i.e.} $\boldsymbol{a}^T_{12}\boldsymbol{u}^*_{1}=b_{12}$ and $\boldsymbol{a}^T_{21}\boldsymbol{u}^*_{2}=b_{21}$. So for $t\geq t_2$, the dynamics of the robots are}
\begin{align}
   &\dot{\boldsymbol{p}}_1=\boldsymbol{u}^*_{1} = \hat{\boldsymbol{u}}_{1}- \frac{1}{2}\mu_{12}\boldsymbol{a}_{12}\hspace{0.1cm},\mbox{where}\hspace{0.1cm} \mu_{12} = 2\frac{\boldsymbol{a}^T_{12}\boldsymbol{\hat{u}}_{1}-b_{12}}{\norm{\boldsymbol{a}_{12}}^2} \nonumber \\
   &\dot{\boldsymbol{p}}_2=\boldsymbol{u}^*_{2} = \hat{\boldsymbol{u}}_{2}- \frac{1}{2}\mu_{21}\boldsymbol{a}_{21}\hspace{0.1cm},\mbox{where}\hspace{0.1cm} \mu_{21} = 2\frac{\boldsymbol{a}^T_{21}\boldsymbol{\hat{u}}_{2}-b_{21}}{\norm{\boldsymbol{a}_{21}}^2}. \nonumber
\end{align}
We can show that the velocities of both robots \textit{i.e.} $\boldsymbol{u}^*_{\{1,2\}}$ have no component perpendicular to $\hat{\boldsymbol{e}}_{\alpha}$. Recall that $\boldsymbol{a}_{12}=-\boldsymbol{a}_{21}=\Delta\boldsymbol{p}_{21}$ is parallel to $\hat{\boldsymbol{e}}_{\alpha}$. Additionally,  $b_{12}=b_{21}=\frac{\gamma}{4}(\norm{\Delta \boldsymbol{p}_{21}}-D_s^2)$. Using these,  we can simplify the dynamics of the robots to
\begin{align}
\label{phasethreedynamics}
   & \dot{\boldsymbol{p}}_1 =+\gamma\frac{\norm{\Delta \boldsymbol{p}_{21}}^2-D_s^2}{4\norm{\Delta \boldsymbol{p}_{21}}^2}\Delta \boldsymbol{p}_{21} \\
    &\dot{\boldsymbol{p}}_2 =-\gamma\frac{\norm{\Delta \boldsymbol{p}_{21}}^2-D_s^2}{4\norm{\Delta \boldsymbol{p}_{21}}^2}\Delta \boldsymbol{p}_{21} 
\end{align}
\begin{align}
\label{relativedynamicsphase3}
 \implies \dot{\Delta \boldsymbol{p}}_{21}=-\gamma\frac{\norm{\Delta \boldsymbol{p}_{21}}^2-D_s^2}{2\norm{\Delta \boldsymbol{p}_{21}}^2}\Delta \boldsymbol{p}_{21} 
\end{align}
\textcolor{black}{Given this inter-robot dynamics, we show in \ref{lemma4_appendix} that the asymptotic solution of \cref{relativedynamicsphase3} proves that $\norm{\Delta \boldsymbol{p}_{21}(t) }$ \textit{i.e.} the inter-robot distance converges to $D_s$ and additionally, both robots come to a halt while still being away from their goals. This establishes the incidence of deadlock.}
\begin{lemma}
\label{lemma4}
The distance between robots converges to the safety margin $D_s$ at which point they stop moving and fall in deadlock.
\end{lemma}
\begin{proof} \textcolor{black}{See appendix \ref{lemma4_appendix}.}
\end{proof}

\textbf{Takeaway: } We have demonstrated that two robots fall in deadlock given initial conditions \eqref{initial_positions_two_robots} and goals \eqref{goal_positions_two_robots}. This is because the \textit{geometric arrangement of initial positions and goals prevents both $\hat{\boldsymbol{u}}$ and the QP to output velocities perpendicular to the $\hat{\boldsymbol{e}}_{\alpha}$ direction} in all the three phases. With the control always along $\boldsymbol{u}_{\parallel} \parallel \hat{\boldsymbol{e}}_{\alpha}$, the only way to reach the goals would involve intersection of the robots \textit{i.e.} collisions. This is why, the QP based controller decides to prevent collisions from happening by forcing the robots to stall, resulting in deadlock. 
\subsection{Closed-Loop Analysis for Three-Robots}
\label{threerobotdeadlock}
We describe another scenario prone to deadlock, which consists of three robots positioned on the vertices of an equilateral triangle, and required to stabilize to their respective antipodal positions. For simplicity, we assume that all prescribed controller gains are \textit{identical}, as are the distances of the robots' initial positions to their goals. Let $\boldsymbol{p}_i(t) \in \mathbb{R}^2$ be the position of robot $i \in \{1,2,3\}$ given as
\begin{align}
\label{initial_positions_three_robots}
    \boldsymbol{p}_1(0) &= (x_0,y_0) \nonumber \\
    \boldsymbol{p}_2(0) &= \boldsymbol{p}_1(0) + D_{init}\hat{\boldsymbol{e}}_{\alpha} \nonumber \\
    \boldsymbol{p}_3(0) &= \boldsymbol{p}_2(0) + D_{init}\hat{\boldsymbol{e}}_{\alpha + \frac{2\pi}{3}}
\end{align}
Assume that the goal of each robot is diametrically opposite to its initial position \textit{i.e.}
\begin{align}
\label{goal_positions_three_robots}
    \boldsymbol{p}_{d_i} &= \boldsymbol{p}_i(0) + \frac{\sqrt{3}D_G}{D_{init.}}(\boldsymbol{c}-\boldsymbol{p}_i(0))
\end{align}
for $i \in \{1,2,3\}$ where $\boldsymbol{c} = \frac{1}{3}\sum_{i=1}^3\boldsymbol{p}_i(0)$ is the centroid of the equilateral triangle formed by the initial positions and $D_G$ is the distance of each robot from its goal. This scenario is prone to deadlock, because of (1) the geometric symmetry in the initial positions of robots and goals and (2) identical controller gains.
\begin{assumption}
\label{ass3}
$\sqrt{3}D_{G}>D_{init.}$
\end{assumption}
This assumption will be needed in \textbf{Lemma \ref{lemma8}} to establish the inevitability of deadlock. We now formally state the deadlock incidence result as follows.
\begin{theorem}
\label{threerobottheorem}
Given the initial positions of robots as in  \eqref{initial_positions_three_robots} and desired goals as in  \eqref{goal_positions_three_robots}, the controls generated by \eqref{optimization_formulation_2} will cause the robots to fall in deadlock.
\end{theorem}
\begin{proof}
Our proof draws upon the ideas from deadlock in the two robot case. The overall sketch is as follows:
\begin{enumerate}
\item First, we analyze the controls returned by the QP at $t=0$. We show in \textbf{Lemma \ref{lemma5}} that if the initial distance between robots is greater than a certain critical distance, then at $t=0$, all collision avoidance constraints of all robots are inactive, therefore $\boldsymbol{u}^*_i(0) =\hat{\boldsymbol{u}}_i(0) \mbox{  }\forall i \in \{1,2,3\}$.
\item Next, we show that the future motion of robots can be broken into two successive phases. In phase 1, the collision avoidance constraints of all robots stay inactive, so they use $\hat{\boldsymbol{u}}_i(t) \mbox{  }\forall i \in \{1,2,3\}$. We prove in \textbf{Lemma \ref{lemma6}} that using $\hat{\boldsymbol{u}}_i(t)$, the robots move on the vertices of an equilateral triangle. Next, in \textbf{Lemma \ref{lemma7}}, we demonstrate that there exists a finite time when phase 1 culminates (see \cref{fig:phases3}). This time is identical for all robots because of the symmetry that follows from \textbf{Lemma \ref{lemma6}}.
\item Next, phase 2 begins during which all constraints become active.  We show in \textbf{Lemma \ref{lemma8}} that the distance between robots converges to $D_s$, their velocities converge to zero while the robots are still away from their goals, thus establishing \textit{deadlock}. \qed
\end{enumerate}
\end{proof}
\hspace{-0.35cm}Recall from  \eqref{kkt_general} that the control for robot $i$ is 
\begin{align}
\label{generaluthree}
\boldsymbol{u}^*_i = \hat{\boldsymbol{u}}_i - \frac{1}{2}\sum_{j\in \{1,2,3\}\backslash i}{\mu_{ij}\boldsymbol{a}_{ij}}
\end{align} 
\textcolor{black}{Here, the value of $\mu_{ij}$ depends on whether the constraint $\boldsymbol{a}_{ij}^T\boldsymbol{u}^*_i - b_{ij}\leq0 $ is active ($=0$) or inactive ($<0$). 
So like the two robot case, we check the value of the flag $f_{ij}(\tilde{\boldsymbol{u}}_i)=\boldsymbol{a}_{ij}^T\tilde{\boldsymbol{u}}_i -b_{ij}$ for the nominal controller \textit{i.e.} $\tilde{\boldsymbol{u}}_i =\hat{\boldsymbol{u}}_i$. We show that if the initial distance between \textit{any two} robots is greater than the critical distance $\beta_+$ \cref{crit3}, then at $t=0$, all robots have all their constraints inactive. Consequently, the optimal control returned by \cref{optimization_formulation_2} will be  $\hat{\boldsymbol{u}}_i(0)$. The critical distance is defined as (derived in \ref{lemma5_appendix}):}
\begin{align}
\label{crit3}
    \beta_+ =\frac{\sqrt{3}D_Gk_p}{\gamma} + \sqrt{\bigg(\frac{\sqrt{3}D_Gk_p}{\gamma}\bigg)^2 + D_s^2}
\end{align}
\begin{lemma}
\label{lemma5}
If  $D_{init.}>\beta_{+}$, then $\boldsymbol{u}^*_i(0) = \boldsymbol{\hat{u}}_i(0)$ $\forall i$
\end{lemma}
\hspace{-0.4cm}\textbf{Intuition:} This result states that, if at $t=0$,  $i$ is away from $j$ by at-least $\beta_+$, then  $i$ can simply use  $\boldsymbol{\hat{u}}_i(0)$  and not ``worry" about a collision with $j$ at-least until the next  instant. This holds for all robots because $\beta^+$ is identical for them all due to the assumed symmetry and their identical controller gains.\textcolor{black}{
\begin{proof} See appendix \ref{lemma5_appendix}. \end{proof}}
\begin{figure}[t]
    \centering
    \includegraphics[trim={0cm 0.0cm 0cm 0cm},clip,width=1\linewidth]{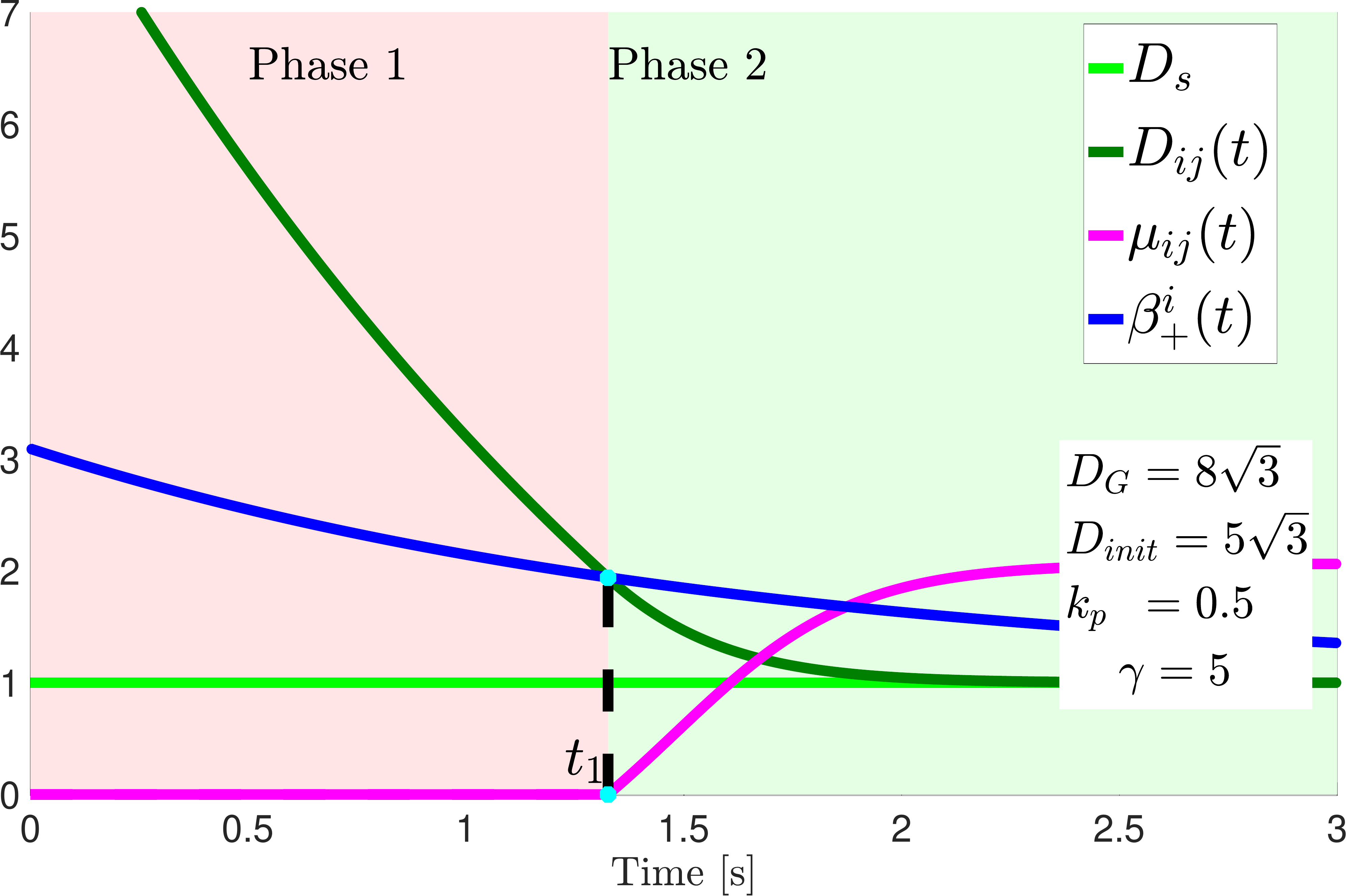}
    \caption{Simulation results for three-robot deadlock. Simulation parameters shown in bottom right panel. Note at $t=t_1$, $D_{ij}(t)=\beta^i_+(t)$ and $\mu_{ij}(t)$ switches on. Finally, $D(t) \longrightarrow D_s$.}
    \label{fig:phases3}
\end{figure}
\begin{assumption}
\label{ass4}
WLOG assume at $t=0$, $D_{init.}>\beta_+$.
\end{assumption}
From assumption \eqref{ass4} and \textbf{Lemma} \ref{lemma5}, it follows that at $t=0$, \textbf{all} collision avoidance constraints of \textbf{all} robots are inactive. Therefore, it is safe to assume that these constraints also remain inactive for a finite duration after $t=0$, which is precisely the duration of phase 1.
\subsubsection{\textbf{Phase 1:} } In this phase, the robots use their nominal controls. Thus, their dynamics are given by:
\begin{align}
\label{threerobotpositions}
\dot{\boldsymbol{p}}_i&=\boldsymbol{u}^*_i = \hat{\boldsymbol{u}}_i = -k_{p}(\boldsymbol{p}_i-\boldsymbol{p}_{d_i}) \nonumber \\
    \implies \boldsymbol{p}_i(t) &= e^{-k_pt}\boldsymbol{p}_i(0) + (1-e^{-k_pt})\boldsymbol{p}_{d_i} 
\end{align}
Next we show that given the choice of symmetric initial and goal positions as in \cref{initial_positions_three_robots} and \cref{goal_positions_three_robots} respectively and identical controller gains, the robots retain this symmetry and move along the vertices of an equilateral triangle.
\begin{lemma}
\label{lemma6}
All robots continue to move in an equilateral triangular configuration with centroid fixed at $\boldsymbol{c}(0)$
\end{lemma}
\hspace{-0.4cm}\textbf{Intuition:} Since the initial positions and goals were symmetric, and since all robots have identical controller gains, they move by the same distance per time step. So the symmetry is preserved every time which is why they retain the equilateral triangular configuration. \url{https://bit.ly/2VnDXdb} conveys this idea.
\begin{proof} \textcolor{black}{See appendix \ref{lemma6_appendix}.} \end{proof}
\textcolor{black}{Next we show that eventually, there comes a time when it is no longer feasible to use the nominal control $\boldsymbol{\hat{u}}_i(t)$ $\forall i$ at which point all the constraints of all robots become active. This again follows from the inevitability of collisions should all three robots continue using $\boldsymbol{\hat{u}}_i$ indefinitely. }
\begin{lemma}
\label{lemma7}
There exists a time $t_{ij}$ when the collision avoidance constraint of robot $i$ with robot $j$ becomes active. Furthermore, $t_{ij}$ is identical $\forall i \in \{1,2,3\}$, $\forall j \in \{1,2,3\}\backslash i$
\end{lemma}
\hspace{-0.4cm}\textbf{Intuition:} Eventually, the inter-robot distance (same for any pair) falls below the common critical distance (see \cref{fig:phases3} where the dark green curve (the inter-robot distance) intersects the blue curve (the critical distance)). Therefore, it is no longer feasible to use the nominal control. This occurs simultaneously for all robots due to the geometric symmetry and identical controller parameters. Let this common time be $t_1$ as shown in Fig \ref{fig:phases3}.
\begin{proof} \textcolor{black}{See appendix \ref{lemma7_appendix}.} \end{proof}
\subsubsection{\textbf{Phase 2:}}
For $t\geq t_1$, all the constraints of all robots have become active. Thus, the  dynamics of the robots are:
\begin{align}
\label{threecontrolopt}
\dot{\boldsymbol{p}}_i=\boldsymbol{u}^*_i = \hat{\boldsymbol{u}}_i - \frac{1}{2}\sum_{j\in \{1,2,3\}\backslash i}{\mu_{ij}\boldsymbol{a}_{ij}}
\end{align}
where the expressions for $\mu_{ij}\neq0$ can be derived by using the activeness of the constraints \textit{i.e.} $\boldsymbol{a}^T_{ij}\boldsymbol{u}^*_i=b_{ij}$. Following arguments in phase three for the two robot case, we can show that the positions of the robots can be written as:
\begin{align}
\label{initialpositions_attina2}
    \boldsymbol{p}_1(t) &= \eta(t)\boldsymbol{p}_1(0) + (1-\eta(t))\boldsymbol{p}_{d_1} \nonumber \\
    \boldsymbol{p}_2(t) &= \boldsymbol{p}_1(t) + D(t)\hat{\boldsymbol{e}}_{\alpha} \nonumber \\
    \boldsymbol{p}_3(t) &= \boldsymbol{p}_2(t) + D(t)\hat{\boldsymbol{e}}_{\alpha + \frac{2\pi}{3}}
\end{align}
for $t \geq t_{1}$ for some function $\eta(t)$ with $\eta(t_1)=e^{-k_pt_1}$. Here $D(t)$ is the inter-robot distance which is same for any two pairs. Using \cref{threecontrolopt}, \cref{initialpositions_attina2}, and the definition of $\boldsymbol{a}_{ij},b_{ij}$ from \cref{abdefs}, the dynamics of robot $i$ can be simplified to:
\begin{align}
\label{deadlock_12_rhs1}
    \dot{\boldsymbol{p}}_i  
    & = \sqrt{3}\gamma \frac{D^2(t)-D_s^2}{6D(t)} \hat{\boldsymbol{e}}_{\alpha + (4i-3)\frac{\pi}{6}}  \\
\implies
\label{deadlock_12_rhs2}
    \dot{\Delta \boldsymbol{p}}_{12} & = \gamma \frac{D^2(t)-D_s^2}{2D(t)} \hat{\boldsymbol{e}}_{\alpha} 
\end{align}
\textcolor{black}{Given this inter-robot dynamics, we show in \ref{lemma8_appendix} that the asymptotic solution of \cref{deadlock_12_rhs2} proves that $\norm{\Delta \boldsymbol{p}_{ij}(t) }$ \textit{i.e.} the inter-robot distance converges to $D_s$ and additionally, all three robots come to a halt while still being away from their goals. This establishes the incidence of deadlock.}
\begin{lemma}
\label{lemma8}
The distance between the robots converges to $D_s$ and the robots stop moving and fall in deadlock.
\end{lemma}
\begin{proof} \textcolor{black}{See appendix \ref{lemma8_appendix}.} \end{proof}
\textbf{Takeaway:} Deadlock happens here because each robot's velocity is always pointing towards its goal and not perpendicular to it. In phase 1, this velocity is along $\hat{\boldsymbol{u}}$ (which by definition is towards the goal). Likewise, we can show that the same holds true in phase 2. Additionally, all robots have identical speed at any given time as well. Hence, collisions would be inevitable at the centroid. Therefore, CBF-QPs cause the robots to stall which results in deadlock. 
\section{During Deadlock}
\label{in deadlock}
In the previous section, we analyzed cases of two and three robots and showed that geometric symmetry in the initial positions and goals results in deadlock. In this section, we assume that the system is already in deadlock and infer geometric properties of robot configurations that are valid in deadlock. The analysis in this section is general for $N$ robots, and not restricted to two or three robots. Recall that we defined deadlock as 
\begin{mydef}
\label{deadlock_definition2}
Robot $i$ is in deadlock if $\boldsymbol{u}^*_i=0$ and $ \boldsymbol{p}_i \neq \boldsymbol{p}_{d_i}$
\end{mydef}
We reformulate the KKT conditions \cref{stationarity1}-\cref{complimentarty slackness1} using this definition to derive geometric properties of the system in deadlock. The aim is to show that these properties are indeed consistent with what one would expect to intuitively hold true when the robots are in deadlock. 
\subsection{Reformulating KKT Conditions for Deadlock and their Mechanics Interpretation:}
\subsubsection{Primal Feasibility: }  Recall from \cref{primal_feasibility1} that $\boldsymbol{a}^T_{ij}\boldsymbol{u}^*_i \leq b_{ij} \mbox{  }\forall j \in \{1,2,\cdots,N\}\backslash i$. Since in deadlock, $\boldsymbol{u}_i^*=\boldsymbol{0}$, this means
\begin{align}
\label{primal_feasibility}
0 \leq b_{ij}   \mbox{   } \forall j \in [\mathbb{N}]\backslash i
\end{align}
From \cref{abdefs}, $b_{ij} = \frac{\gamma}{4}h_{ij} =  \frac{\gamma}{4} \big(\norm{\Delta \boldsymbol{p}_{ij}}^2 - D^2_s \big)$ which using \cref{primal_feasibility} implies
\begin{align}
\label{primal_feasibility2}
0 \leq h_{ij} \iff D_s \leq \norm{\Delta \boldsymbol{p}_{ij}}   \mbox{   } \forall j \in [\mathbb{N}]\backslash i
\end{align}
\textcolor{black}{This means that robot $i$, when in deadlock, is at-least $D_s$ distance away from every other robot and therefore safe.}
\subsubsection{Dual Feasibility: } From \cref{dual_feasibility1} recall that,
\begin{align}
\label{dual_feasibility}
{\mu^*_{ij}} \geq 0 \mbox{  }	 \forall j \in [\mathbb{N}]\backslash i
\end{align}
\subsubsection{Complementary Slackness: } From \cref{complimentarty slackness1}, recall that $\mu^*_{ij} \cdot (\boldsymbol{a}^T_{ij}\boldsymbol{u}^*_i -b_{ij}) = 0 
	 \mbox{   }\forall j \in [\mathbb{N}]\backslash i$. Substituting $\boldsymbol{u}_i^*=\boldsymbol{0}$ and $b_{ij} = \frac{\gamma}{4} h_{ij}=\frac{\gamma}{4} \big(\norm{\Delta \boldsymbol{p}_{ij}}^2 - D^2_s\big)$, we get
	 \begin{align}
	\label{complimentarty slackness}
	\mu^*_{ij} \cdot \big(\norm{\Delta \boldsymbol{p}_{ij}}^2 - D^2_s \big) =0 
	 \mbox{   }\forall j \in [\mathbb{N}]\backslash i
\end{align}
Thus, using \cref{dual_feasibility}, \cref{complimentarty slackness} and \cref{activeinactive}, the constraints that are active and inactive for robot $i$ in deadlock are
\begin{align}
\label{activeinactive1}
	\mathcal{A}(\boldsymbol{0}) &= \{j \in [\mathbb{N}]\backslash i \mid \norm{\Delta \boldsymbol{p}_{ij}}=D_s \iff  h_{ij}=0 \} \nonumber \\
	\mathcal{IA}(\boldsymbol{0})&= \{j \in [\mathbb{N}]\backslash i\mid \norm{\Delta \boldsymbol{p}_{ij}}>D_s \}
\end{align}
\textcolor{black}{Thus, active constraints refer to those robots that are exactly $D_s$ away from $i$. Alternatively, these are the robots $j$ whose positions are such that $h_{ij} =0$ \textit{i.e.} robot $i$'s safety is at the verge of being compromised thanks to these robots. In the next lemma, we show that if $i$ is in deadlock, then it can never be the case that there are no active robots.}
\begin{lemma}
\label{activesetisnonempty}
    $\mathcal{A}(\boldsymbol{0})$ is non-empty.
\end{lemma}
\hspace{-0.4cm}\textbf{Intuition:} When $i$ is in deadlock, $\mathcal{A}(\boldsymbol{0})$ is the set of robots preventing it from reaching its goal. Thus, this set cannot be empty as other wise the robot would be able to move freely to reach its goal and hence not be in deadlock.
\begin{proof}
    We prove this by contradiction. Suppose that $\mathcal{A}(\boldsymbol{0})$ is empty   or said another way $\mathcal{IA}(\boldsymbol{0})=[\mathbb{N}]\backslash i$. This means that $\mu^*_{ij}=0 \mbox{ }\forall j \in [\mathbb{N}]\backslash i$ and hence $\hat{\boldsymbol{u}}_i= \boldsymbol{u}^*_i$ from stationarity \cref{stationarity1}.  
    Furthermore, since $i$ is in deadlock, from Def. \eqref{deadlock_definition2} we have that $\boldsymbol{p}_i \neq \boldsymbol{p}_{d_i} \iff \hat{\boldsymbol{u}}_i \neq \boldsymbol{0}$ since $\hat{\boldsymbol{u}}_i = -k_{p_i}(\boldsymbol{p}_i - \boldsymbol{p}_{d_i})$ However, $\boldsymbol{u}^*_i = \boldsymbol{0}$ in deadlock. This means that $\boldsymbol{u}^*_i \neq \hat{\boldsymbol{u}}_i$ in deadlock, giving rise to the contradiction. \qed 
\end{proof}
Furthermore, we can interpret the multipliers $\mu_{ij}^*$ as spring-constants of springs connecting $i$ to $\mathcal{A}(\boldsymbol{0})$. This gives rise to the force-equilibrium interpretation of stationarity: 
\subsubsection{Stationarity:} Using \cref{kkt_general} and $\boldsymbol{u}^*_i= \boldsymbol{0}$, we get
\begin{align}
\label{stationarity}
	\hat{\boldsymbol{u}}_i - \frac{1}{2}\sum_{j \in \mathcal{A}(\boldsymbol{0})}\mu^*_{ij}\boldsymbol{a}_{ij} &= \boldsymbol{0} \nonumber \\
	\iff -k_{p_i}(\boldsymbol{p}-\boldsymbol{p}_{d_i}) +\frac{1}{2}\sum_{j \in \mathcal{A}(\boldsymbol{0})}\mu^*_{ij}\Delta \boldsymbol{p}_{ij}&= \boldsymbol{0}
\end{align}
In this equation, the first term $\hat{\boldsymbol{u}}_i = -k_{p_i}(\boldsymbol{p}-\boldsymbol{p}_{d_i})$ represents an attractive force pulling the robot $i$ towards its goal $\boldsymbol{p}_{d_i}$. The second term $+ \frac{1}{2}\sum_{j \in \mathcal{A}(\boldsymbol{u}^*_i)}\mu^*_{ij}\Delta \boldsymbol{p}_{ij}$ represents the resultant repulsive force from active robots \textit{i.e.} ones that are $D_s$ away from $i$. Thus, deadlock occurs when the attractive force to the goal is balanced by the net repulsion from active robots.
\begin{figure}[t]
    \centering
    \includegraphics[width=0.8\linewidth]{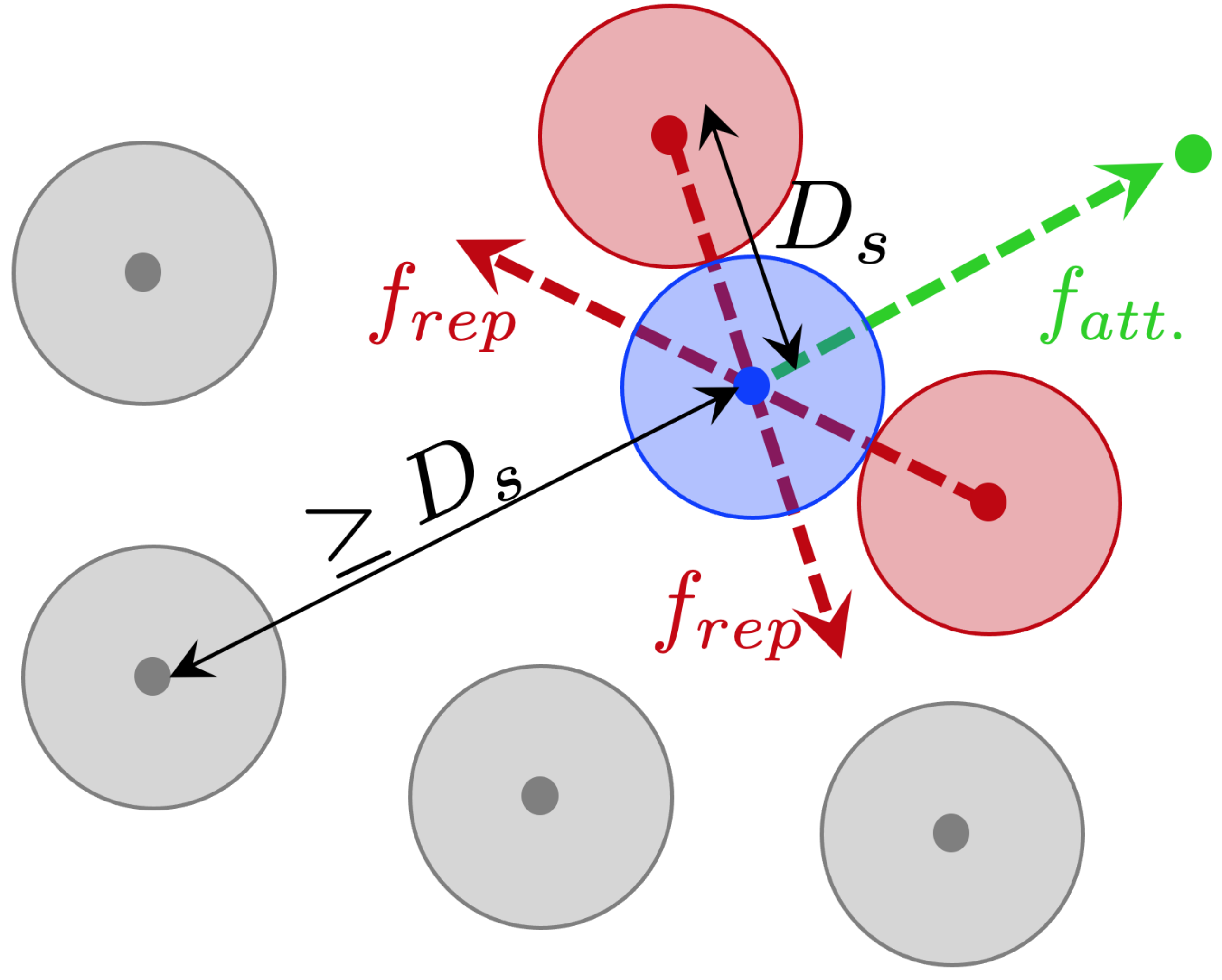}
    \captionof{figure}{An illustration of the mechanics interpretation of KKT conditions. The ego robot is shown in blue and its desired goal location in green. The green arrow represents the force attracting the ego robot towards its goal. The red robots are the ones in the active set since they are separated by exactly $D_s$, therefore, they repel the ego robot. The robots shown in gray represent the inactive constraints. }
    \label{fig:dcartoon}
\end{figure}

Based on these conditions, we motivate a set-theoretic interpretation of deadlock. Let the state of the ego robot be $\boldsymbol{p}_i \in \mathbb{R}^2$  and let the states of the remaining $N-1$ robots be denoted as $\boldsymbol{p}_{-i}= \{\boldsymbol{p}_j\}$ $\forall j \in [\mathbb{N}]\backslash i$. Define $\mathcal{D}_i$ as:
\begin{align}
\label{deadlock_def3}
	\mathcal{D}_i=\{\boldsymbol{p}_i \in \mathbb{R}^2 \mid \boldsymbol{u}^*_i(\boldsymbol{p}_i,\boldsymbol{p}_{-i}) = \boldsymbol{0}, \boldsymbol{p}_i \neq \boldsymbol{p}_{d_i}\}
\end{align}
The set $\mathcal{D}_i$ is defined as those states of robot $i$ which satisfy the criteria of $i$ being in deadlock.
We restate the $\boldsymbol{u}^*_i(\boldsymbol{p}_i,\boldsymbol{p}_{-i}) = \boldsymbol{0}$ criterion in this definition with the reformulated KKT conditions described above:
\begin{align}
\label{deadlock_def4}
	\mathcal{D}_i=&\bigg\{\boldsymbol{p}_i \in \mathbb{R}^2 \bigg| \norm{\Delta \boldsymbol{p}_{ij}} \geq D_s,\mu^*_{ij}\geq 0 \mbox{   } \forall j \in [\mathbb{N}]\backslash i, \nonumber\\ &\hat{\boldsymbol{u}}_i(\boldsymbol{p}_i) = \frac{1}{2}\sum_{j \in \mathcal{A}(\boldsymbol{0})}\mu^*_{ij} \boldsymbol{a}_{ij},\boldsymbol{p}_i \neq \boldsymbol{p}_{d_i}\bigg \}
\end{align}
The motivation behind stating this definition is to interpret deadlock as a bonafide set in the state space of the ego robot and derive a control strategy that makes the robot evade/exit this set. Building on this definition of robot $i$'s deadlock set, we motivate \textit{system deadlock} to be the  joint states of all robots for which all of them are in deadlock: 
\begin{align}
\label{sysdeadlockdefN}
\mathcal{D}_{system}
=\bigg\{(&\boldsymbol{p}_1,\cdots,\boldsymbol{p}_N)\in \mathbb{R}^{2N}\bigg| \boldsymbol{p}_i \in \mathcal{D}_i \mbox{  }, i \in \nonumber \\ & \{1,\cdots,N\} \bigg\} = \prod_{i=1}^N\mathcal{D}_i  
\end{align}
In the rest of the paper, we will focus on system deadlock. This is because the case where only a subset of robots are in deadlock can be decomposed into subproblems where a subset is in \textit{system deadlock} and the remaining robots are free to move. The next section focuses on complexity analysis of \textit{system deadlock}. 

\subsection{Graph Enumeration based Complexity Analysis}
\label{complexity}
The Lagrange multipliers $\mu^*_{ij}$  are in general, a nonlinear function of the states of the robots. Their values depend on which constraints are active/inactive. An active constraint will in-turn determine the set of possible geometric configurations that the robots can take when they are in deadlock and this in turn will guide the design of our deadlock resolution algorithm.  Therefore, we are interested in deriving all possible combinations of active/inactive constraints that the robots can assume once in deadlock. But first we derive upper and lower bounds for the number of valid configurations in \textit{system deadlock}. 

We can interpret an active constraint between robots $i$ and $j$ as an undirected edge between vertices $i$ and $j$ in a graph formed by $N$ labeled vertices, where each vertex represents a robot. The following property (which follows from symmetry) allows the edges to be undirected.
\begin{lemma}
\label{symm}
    If robot $i$ and $j$ are both in deadlock and  $i's$  constraint with $j$ is active (inactive), then  $j's$  constraint with $i$ is also active (inactive).
\end{lemma}
\hspace{-0.4cm}\textbf{Intuition:} If $i$ and $j$ are in deadlock and $i$ perceives that $j$ is at the verge of violating $i's$ safety, then $j$ also perceives that $i$ is at the verge of violating its safety.
\begin{proof}
Since $i$ is in deadlock and its constraint with $j$ is active, we have $\boldsymbol{a}^T_{ij}\boldsymbol{u}^*_i=b_{ij}=0 \iff b_{ji}=0$, since from, \cref{abdefs} we have that $b_{ij}=b_{ji}$. Since $j$ is in deadlock, $\boldsymbol{a}^T_{ji}\boldsymbol{u}^*_j=0 \implies b_{ji}=\boldsymbol{a}^T_{ji}\boldsymbol{u}^*_j$. This means $j's$ constraint with $i$ is also active. \qed
\end{proof}
\subsubsection{Upper Bound}
Given $N$ vertices, there are ${}^{N}C_{2}$ distinct pairs of edges possible. The overall system can have any subset of those edges. Since a set with ${}^{N}C_{2}$ members has $2^{{}^{N}C_{2}}$ subsets, we conclude that there are $2^{{}^{N}C_{2}}$ possible graphs. In other words, given $N$ robots, the number of configurations that are admissible in deadlock is $2^{{}^{N}C_{2}}$. However, this number is an upper bound because it includes cases where a given vertex can be disconnected from all other vertices, which is not valid in \textit{system deadlock} because $\mathcal{A}(\boldsymbol{0})$ cannot be empty (\textbf{Lemma} \ref{activesetisnonempty}).

\subsubsection{Lower Bound}
Since $\mathcal{A}(\boldsymbol{0})$ is non-empty,  each vertex in the graph will have at-least one edge \textit{i.e.} each robot will have at-least one constraint active with some other robot. From this observation, it follows that the set of graphs valid in \textit{system deadlock} is a superset of connected simple graphs that can be formed by $N$ labeled vertices. It is also possible that some simply connected graphs may not be feasible due to restrictions imposed by Euclidean geometry.

Graphs that are (a) simple and connected, (b) have $N$ labelled vertices (since each robot has an ID), (c) are embedded in $\mathbb{R}^2$ (environment is planar), (d) have Euclidean distance between connected vertices equal to $D_s$, (e) that between unconnected vertices greater than $D_s$, and (f) have one or two edges per vertex (possibly more), necessarily represent admissible geometric configurations of robots in \textit{system deadlock}. \textcolor{black}{This follows because every graph satisfying qualifiers (a)-(f) can be mapped to a geometric arrangement of $N$ robots in deadlock.}  A lower bound for graphs satisfying all qualifiers (a)-(f) can be shown to be $0.5(N+1)(N-1)!$ as follows ($N\geq 3$). 

Consider a cyclic graph whose each node is the vertex of an $N$ regular polygon with side $D_s$. Such a graph necessarily satisfies (a)-(f). Re-arrangements of its vertices gives rise to $0.5(N-1)!$ graphs. Likewise, a graph with nodes along an open chain also satisfies (a)-(f), and gives $0.5N!$ rearrangements. Thus, the total is $0.5(N-1)!+0.5N!=0.5(N+1)(N-1)!$. It is well known that factorial overtakes exponential, thus highlighting the increase in the number of geometric configurations. Our MATLAB simulations show that the exact number of configurations for $N=\{1,\mbox{ }2,\mbox{ }3,\mbox{ }4\}$ are $\{1,\mbox{ }1,\mbox{ }4,\mbox{ }18\}$ whereas our bound gives $\{1,\mbox{ }1,\mbox{ }4,\mbox{ }15\}$. This simulation demonstrates the explosion in possible geometric configurations that are admissible in \textit{system deadlock} with increasing number of robots. Therefore for further analysis, we will restrict to the case of two and three robots. 
\subsection{Characteristics of Two-Robot Deadlock}
\label{two_robot_analysis}
In this section, we give some specific properties of a two-robot system in deadlock which we will use later to resolve deadlocks in two robots. We show that these properties hold true regardless of the choice of initial conditions.

In the two-robot system, an individual robot by itself cannot be in deadlock \textit{i.e.} either both robots are in deadlock or neither. This  is because the sole collision avoidance constraint is symmetric due to \textbf{Lemma} \ref{symm}. From \textbf{Lemma} \ref{activesetisnonempty}, each robot perceives the other robot as active. Define the \textit{system deadlock} set $\mathcal{D}_{system}$ using \cref{deadlock_def4,sysdeadlockdefN}:
\begin{align}
\label{sysdeadlockdef}
\mathcal{D}_{system}
=&\bigg\{(\boldsymbol{p}_1,\boldsymbol{p}_2)\in \mathbb{R}^4 \bigg|\norm{\Delta \boldsymbol{p}_{12}} \geq D_s,\mu^*_{12},\mu^*_{21}\geq 0, \nonumber \\ 
&\hat{\boldsymbol{u}}_1(\boldsymbol{p}_1) = \frac{1}{2}\mu^*_{12} \boldsymbol{a}_{12},\hat{\boldsymbol{u}}_2(\boldsymbol{p}_2) = \frac{1}{2}\mu^*_{21} \boldsymbol{a}_{21}, \nonumber \\ 
&\boldsymbol{p}_1 \neq \boldsymbol{p}_{d_1},\boldsymbol{p}_2 \neq \boldsymbol{p}_{d_2}
\bigg \}.
\end{align}
where $\boldsymbol{a}_{12} =-\boldsymbol{a}_{21} = -(\boldsymbol{p}_1-\boldsymbol{p}_2)$. From \eqref{stationarity}, we have that $\hat{\boldsymbol{u}}_i= \frac{1}{2}\mu^*_{ij} \boldsymbol{a}_{ij}$. This gives
\begin{align}
\label{Lagrange Multipliers}
	\mu^*_{ij} = 2\frac{\boldsymbol{a}_{ij}^T\hat{\boldsymbol{u}}_i}{\norm{ \boldsymbol{a}_{ij}}^2}
\end{align}
\begin{theorem}\textbf{Safety Margin Apart.}
\label{touching}
In deadlock, the two robots are separated by the safety distance and the robots are on the verge of violating safety (Fig. \ref{fig:D4.png})
\end{theorem}
\begin{proof}
Since both robots are in deadlock, each perceives the other as active \textit{i.e.} $\mathcal{A}_i(\boldsymbol{0}) = \{1,2\}\backslash i$ for $i=\{1,2\}$. Then, straightforward application of  \cref{activeinactive1} gives this result. \qedsymbol
\end{proof}
\begin{figure}[t]
\centering      
\subfigure[Two Robot Equilibrium]{\label{fig:D4.png}\includegraphics[trim={0.0cm 15.8cm 11.94cm 2.7cm},clip,width=.49\columnwidth]{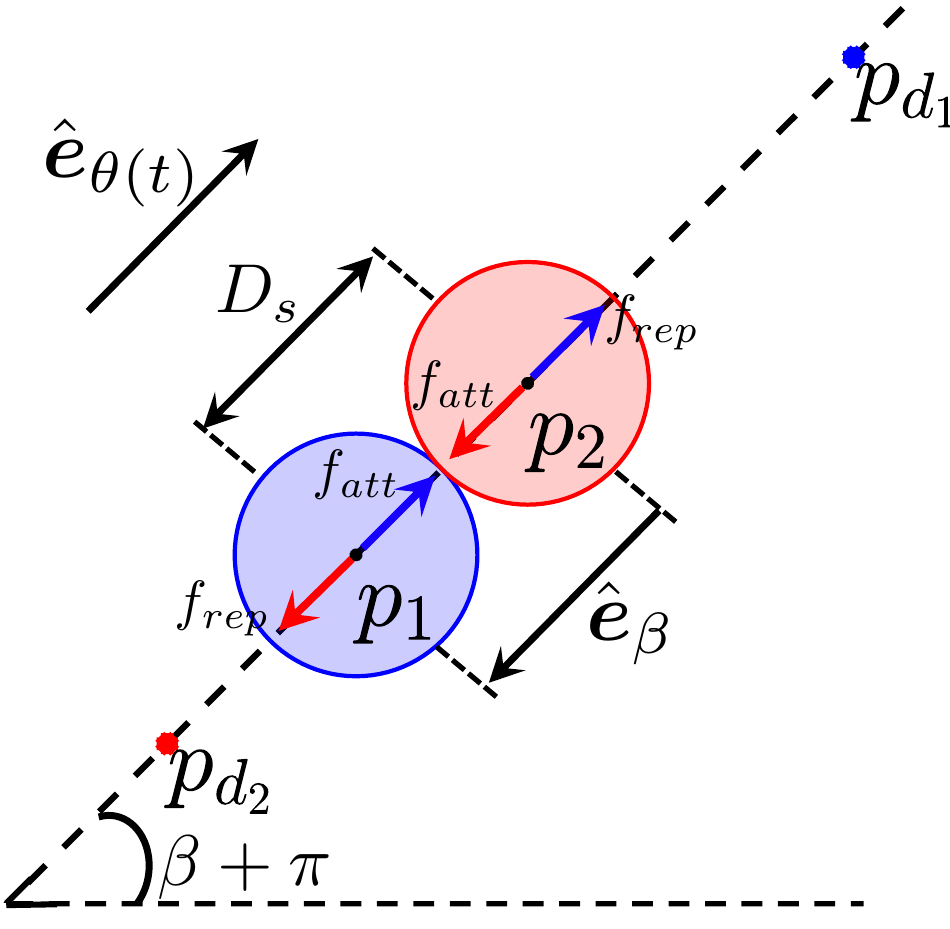}}
\subfigure[Three Robot Equilibrium]{\label{fig:D4.jpeg}\includegraphics[trim={11.2cm 15.4cm 0.0cm 2.7cm},clip,width=.49\columnwidth]{A.pdf}}
\caption{Force Equilibrium in Deadlock}
\label{equib._deadlock_two_three_robots}
\end{figure}
\begin{theorem}\textbf{$\mathcal{D}_{system}$ is Non-Empty.}
\label{nonempty}
$\forall\mbox{ } k_p,D_s>0,\exists$ a family of states $(\boldsymbol{p}^*_1,\boldsymbol{p}^*_2) \in \mathcal{D}_{system}$. These states are such that the robots and their goals are all collinear.
\end{theorem}
\begin{proof}
To prove this theorem, we propose a set of candidate states $(\boldsymbol{p}^*_1,\boldsymbol{p}^*_2)$ and show that they satisfy the definition of deadlock \cref{sysdeadlockdef}. See Fig. \ref{fig:D4.png} for an illustration of geometric quantities referred to in this proof. 

Let $\boldsymbol{p}^*_1=\alpha \boldsymbol{p}_{d_1} + (1-\alpha) \boldsymbol{p}_{d_2}$ and $\boldsymbol{p}^*_2= \boldsymbol{p}^*_1 - D_s\hat{e}_{\beta}$ where $\beta = \mbox{tan}^{-1}(\frac{y_{d_2}-y_{d_1}}{x_{d_2}-x_{d_1}})$ and $\alpha \in (0,1)$. Note that $\boldsymbol{p}^*_1,\boldsymbol{p}^*_2, \boldsymbol{p}_{d_1}, \boldsymbol{p}_{d_2}$ are collinear by construction.  Then we will show that $(\boldsymbol{p}^*_1,\boldsymbol{p}^*_2) \in \mathcal{D}_{system}$. 
Note that
\begin{align}
\centering
\label{cont}
\boldsymbol{a}_{12}&=-(\boldsymbol{p}^*_1-\boldsymbol{p}^*_2) =-D_{s}\hat{e}_{\beta} \nonumber \\
\hat{\boldsymbol{u}}_1&=-k_{p_1}(\boldsymbol{p}^*_1-\boldsymbol{p}_{d_1}) 
\end{align}
From definition, $\hat{e}_{\beta} 
    =\frac{1}{D_G}(x_{d_2}-x_{d_1},y_{d_2}-y_{d_1})$ where $D_G = \norm{\boldsymbol{p}_{d_2}-\boldsymbol{p}_{d_1}}$ is the distance between the goals. Therefore, we have 
\begin{align}
\label{etheta}
\boldsymbol{p}^*_1-\boldsymbol{p}_{d_1} &= 
-(1-\alpha) \boldsymbol{p}_{d_1} + (1-\alpha) \boldsymbol{p}_{d_2} \nonumber \\
&= (1-\alpha)D_{G}\hat{e}_{\beta}
\end{align}
Substituting \cref{etheta} in \cref{cont} gives 
\begin{align}
\label{lhs}
\hat{\boldsymbol{u}}_1=-k_{p_1}(1-\alpha)D_{G}\hat{e}_{\beta}
\end{align}
From \cref{Lagrange Multipliers}, \cref{cont} and \cref{lhs}, we deduce that  $\mu^*_{12} $
\begin{align}
\label{lmult}
 &\mu^*_{12} = 2\frac{\boldsymbol{a}^T_{12}\hat{\boldsymbol{u}}_1}{\norm{ \boldsymbol{a}_{12}}^2_2} = 2k_{p_1}(1-\alpha)\frac{D_G}{D_s} >0\mbox{  }\forall \alpha \in (0,1) 
\end{align}
 In \cref{lmult} we have shown that the Lagrange multiplier $\mu_{12}^*$ is positive, which is one condition in \cref{sysdeadlockdef}. We can similarly show that $\mu^*_{21}>0$. Further, $\hat{\boldsymbol{u}}_1=\frac{1}{2}\mu^*_{12} \boldsymbol{a}_{12}$ is trivially satisfied because of the way we computed the multiplier in \eqref{Lagrange Multipliers}. Likewise, $\hat{\boldsymbol{u}}_2=\frac{1}{2}\mu^*_{21} \boldsymbol{a}_{21}$ will be satisfied.  Restricting  $\alpha \in (0,1)$  ensures that $\boldsymbol{p}^*_i\neq\boldsymbol{p}_{d_i}$. Hence, the proposed states are always in deadlock as all conditions in \cref{sysdeadlockdef} are met. \qed
\end{proof}
\begin{theorem}\textbf{$\mathcal{D}_{system}$ is zero-measure.}
The \textit{system deadlock} set is measure zero.
\end{theorem}
\begin{proof}
Following the definition of $\boldsymbol{p}^*_{1}$ and $\boldsymbol{p}^*_{2}$  from \cref{touching} and \cref{nonempty}, we can show that when two robots are in deadlock, their positions satisfy
\begin{align}
\norm{\big(\boldsymbol{p}_{1} -\boldsymbol{p}_{d_1}\big)} +  \norm{\big(\boldsymbol{p}_{2} -\boldsymbol{p}_{d_2}\big)} = D_s+ D_G \nonumber
\end{align}
This can be verified by straightforward substitution. \qed
\end{proof}
\subsection{Characteristics of Three Robot Deadlock}
\label{three_robot_analysis}
Following the results shown in two robot deadlock, we now describe the three robot case. We will demonstrate that properties such as robots being on the verge of safety violation (\cref{touching3}) and non-emptiness (\cref{nonempty3}) are retained in this case as well. We are interested in analyzing \textit{system deadlock}, which occurs when $\boldsymbol{u}^*_{i} =  \boldsymbol{0}$ and $\hat{\boldsymbol{u}}_i \neq \boldsymbol{0}$ $\forall i \in \{1,2,3\}$. Since we are studying \textit{system deadlock}, each robot has at-least one active constraint  (each robot has two constraints in total). The \textit{system deadlock} set $\mathcal{D}_{system}$ for three robots is defined analogously to \cref{sysdeadlockdef}.
\begin{theorem} \textbf{Safety Margin Apart}
\label{touching3}
In system deadlock, either all three robots or  exactly two pairs of robots are separated by the safety margin.
\end{theorem}
\begin{proof}
The proof is kept brief because it is similar to the proof of \cref{touching}. Based on the number of constraints that are allowed to be active per robot, all geometric configurations can be clubbed in two categories :  \\
\textbf{Category A -}  This arises when all  constraints of each robot are active \textit{i.e.} $\boldsymbol{a}^T_{ij}\boldsymbol{u}^*_i=\hat{b}_{ij}=0 \iff \norm{\Delta \boldsymbol{p}_{ij}}=D_s \mbox{  }\forall j\in \{1,2,3\}\backslash i$ $\forall i \in \{1,2,3\} $. As a result, each robot is separated by $D_s$ from every other robot (\cref{fig:equilateral}).  
\\
\textbf{Category B -} This arises when there is exactly one robot with both its constraints active (robot $i$ in \cref{fig:nonequilateral}), and the remaining robots ($j$ and $k$) have exactly one constraint active each. Hence, robot $i$ is separated by $D_s$ from the others. \qedsymbol
\end{proof}
\begin{theorem} \textbf{Non-emptiness}
\label{nonempty3}
$\forall\mbox{ } k_p,D_s,R$ $>0$ and $\boldsymbol{p}_{d_i}=R\hat{\boldsymbol{e}}_{2\pi (i-1)/3}$ where $i=\{1,2,3\}$, $\exists$ $(\boldsymbol{p}^*_1,\boldsymbol{p}^*_2,\boldsymbol{p}^*_3) \in  \mathcal{D}_{system}$ where $\boldsymbol{p}^*_i$ are   \\
Category A: $\boldsymbol{p}^*_i=\frac{D_s}{\sqrt{3}}\hat{\boldsymbol{e}}_{\frac{2\pi (i-1)}{3} + \pi}  $ where $i=\{1,2,3\}$ \\
Category B: $\boldsymbol{p}^*_1 = D_s\hat{\boldsymbol{e}}_{\pi}$,
    $\boldsymbol{p}^*_2=\boldsymbol{0}$,   $\boldsymbol{p}^*_3=D_s\hat{\boldsymbol{e}}_{\frac{\pi}{3}}$ if robot $2$ has both constraints active.
\end{theorem}
\begin{proof}
This proof is similar to  \cref{nonempty} so it is skipped. \qedsymbol
\end{proof}
Note that in the statement of this theorem, we have predefined the desired goal positions unlike the statement of \cref{nonempty}. The candidate positions of the robots that we propose are in $\mathcal{D}_{system}$ are valid with respect to these given goals. Additionally, for category B, we proposed one set of positions that is valid in deadlock, however there is continuous family of positions that can be valid in category B which can be found in the appendix of \cite{grover2019deadlock}.  
 \begin{figure}[t]
 	 \setlength{\belowcaptionskip}{0pt}
	\centering      
	\subfigure[Category A ]{\label{fig:equilateral}\includegraphics[width=0.40\linewidth]{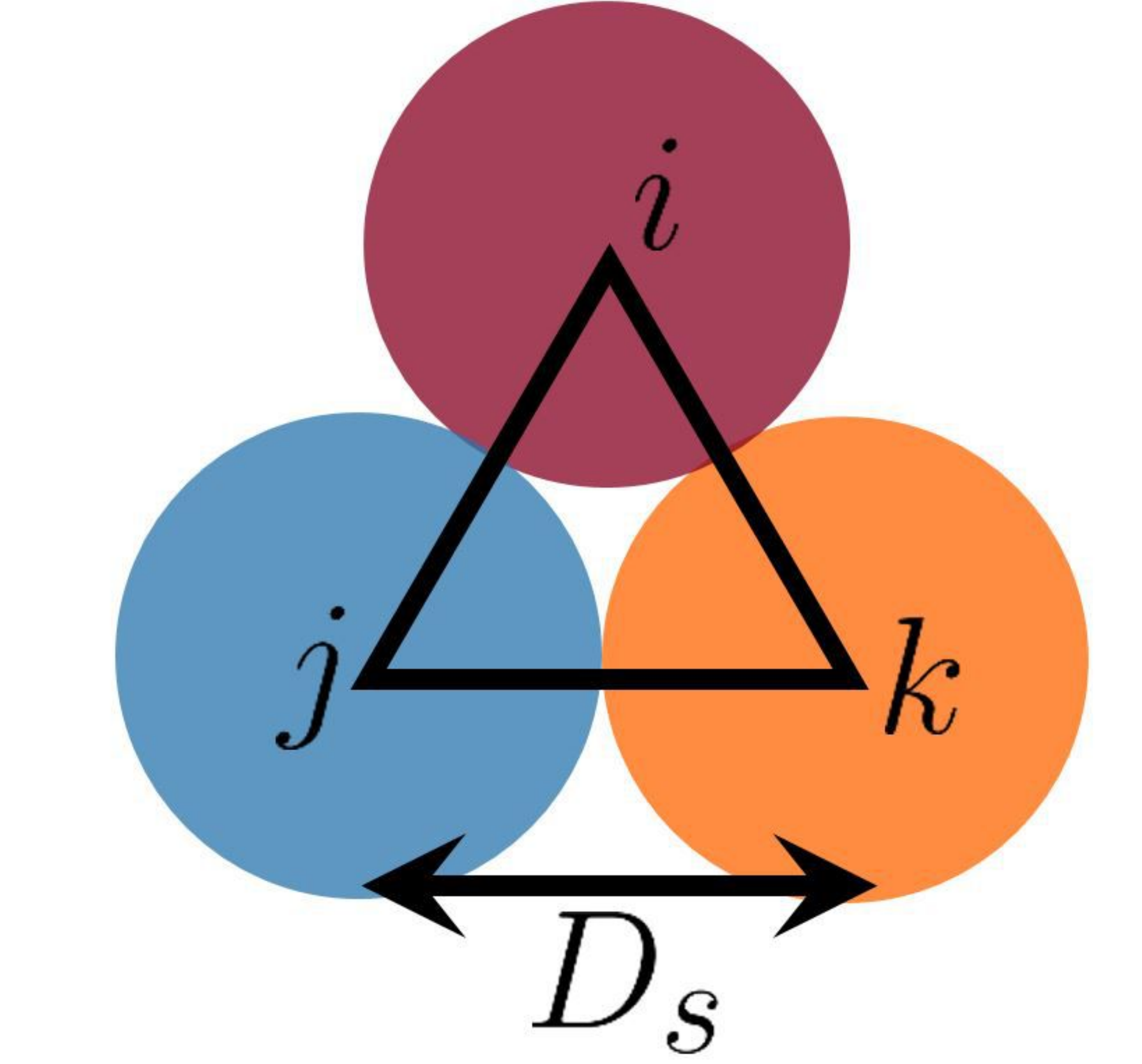}}
	\subfigure[Category B]{\label{fig:nonequilateral}\includegraphics[width=0.40\linewidth]{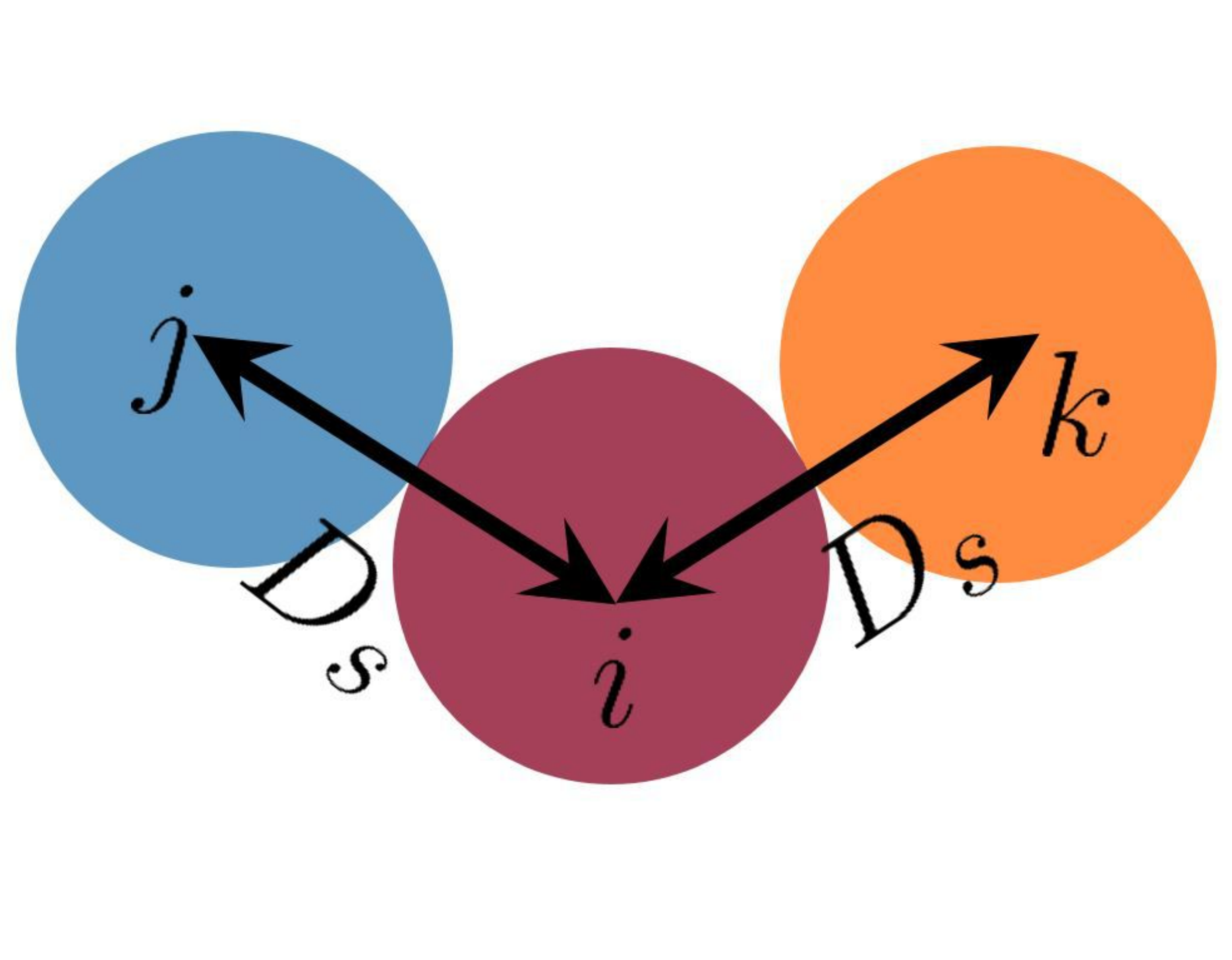}}
	\caption{Geometric configurations in three robot deadlock}
	\label{deadlock_three_robots}
\end{figure}
\section{After Deadlock: Deadlock Resolution}
\label{deadlock resolution}
We now use the properties of geometric configurations derived in \cref{two_robot_analysis} to synthesize a strategy that (1) gets the robots out of deadlock, (2) ensures their safety and (3) makes them converge to their goals.  One approach to achieve these objectives is to detect the incidence of deadlock while the  CBF-QP controller is running on the robots and once detected, any small non-zero perturbation to the control will instantaneously give a non-zero velocity to the robots. Thereafter, CBF-QPs can take charge again and we can hope that using this controller the system state will come out of deadlock at-least for a short time. However, there is no guarantee that deadlock will not relapse because it was the CBF-QP controller that led to deadlock in the first place. Secondly, perturbations can violate safety and even lead to degraded performance. Therefore, given these limitations, we propose a controller which ensures that goal stabilization, safety and deadlock resolution are met with guarantees. We demonstrate this algorithm for the two and three robot cases. Refer to  Fig. \ref{deadlock_res_two_three_robots} for a schematic of our approach. This algorithm is described here:
\begin{figure*}
\centering      
 \setlength{\belowcaptionskip}{0pt}
\subfigure[Deadlock resolution for two robots]{\label{fig:drestwo}\includegraphics[trim={0.0cm 20cm 0.0cm 3cm},clip,width=1.4\columnwidth]{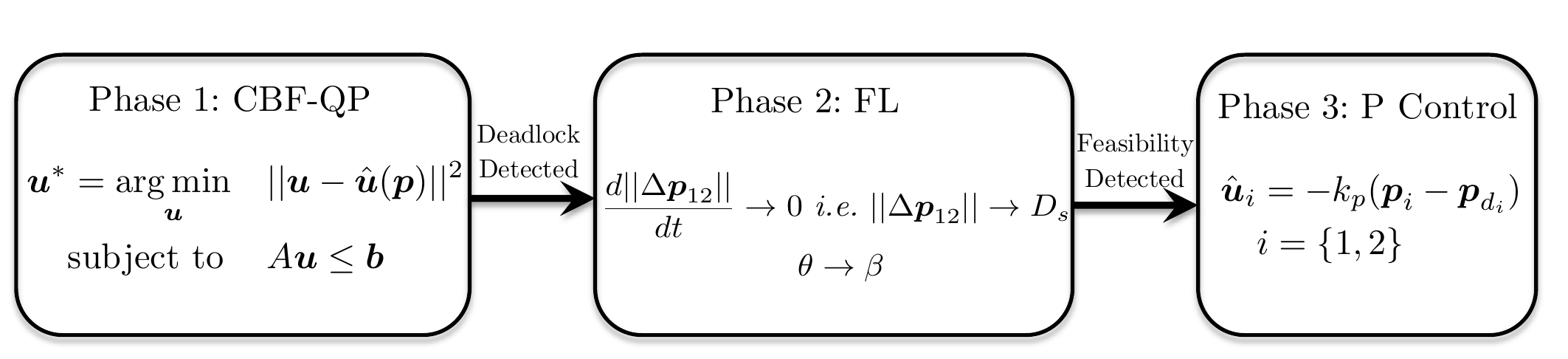}}
\subfigure[Deadlock resolution for three robots]{\label{fig:dresthree}\includegraphics[trim={0.0cm 14.2cm 0.0cm 9cm},clip,width=1.4\columnwidth]{Schematic.pdf}}
\caption{Deadlock Resolution Algorithm Schematic}
\label{deadlock_res_two_three_robots}
\end{figure*}
\begin{enumerate}
    \item The algorithm starts by executing controls derived from CBF-QP in Phase 1. This ensures movement of robots to the goals and safety by construction. To detect the incidence of deadlock, we continuously compare $\norm{\boldsymbol{u}^*}, \norm{\boldsymbol{p}-\boldsymbol{p}_{d}}$ against small thresholds. Once these thresholds are small enough, we declare deadlock detected and switch control to Phase 2.
    \item In this phase, we rotate the robots around each other to swap positions while maintaining the safe distance. Refer to the appendix of \cite{grover2019deadlock} for the controller derivation.
    \begin{enumerate}
        \item For the two-robot case, we calculate $\boldsymbol{u}^1_{fl}(t)$ and $\boldsymbol{u}^2_{fl}(t)$ using feedback linearization to ensure that $\norm{\Delta \boldsymbol{p}_{12}}=D_s$ and rotation ($\dot{\theta}=-k_p(\theta-\beta) \implies \Delta \boldsymbol{p}^T_{12}\Delta \boldsymbol{v}_{12}=0$) (See Fig. \ref{fig:D4.png} for $\theta,\beta$). This rotation and distance invariance guarantees safety \textit{i.e.} $\ h_{12}=0$. Adding an extra constraint $\boldsymbol{u}^1_{fl}+\boldsymbol{u}^2_{fl}=\boldsymbol{0}$ still ensures that the problem is well posed and additionally makes the centroid static. 
        \item For the three-robot case, we calculate $\boldsymbol{u}^1_{fl}(t),\boldsymbol{u}^3_{fl}(t),\boldsymbol{u}^3_{fl}(t)$ to ensure that $\norm{\Delta \boldsymbol{p}_{12}}=\norm{\Delta \boldsymbol{p}_{23}}=\norm{\Delta \boldsymbol{p}_{31}}=D_s$ and rotation ($\dot{\theta}=-k_p(\theta-\beta) \implies \Delta \boldsymbol{p}^T_{12}\Delta \boldsymbol{v}_{12}=0$) (See Fig. \ref{fig:D4.jpeg} for $\theta,\beta$). This guarantees safety \textit{i.e.} $ h_{12}=h_{23}=h_{31}=0$. Similarly, we impose $\boldsymbol{u}^1_{fl}+\boldsymbol{u}^2_{fl}+\boldsymbol{u}^3_{fl}=\boldsymbol{0}$ to make the centroid static. Note that both in simulations and experiments, we  observed incidence of only Category A deadlock so our resolution algorithm is specific for this category. 
    \end{enumerate}
    \item Once the robots swap their positions, their new positions will ensure that prescribed proportional controllers will be feasible in the future. Thus, after convergence of Phase 2, control switches to Phase 3, which simply uses the prescribed controllers. This phase guarantees that the distance between robots is non-decreasing and safety is maintained as we prove in \cref{finalthm}.
\end{enumerate}
\begin{figure}[t]
	\centering
	\subfigure[Robot 1 Position]{\label{fig:cc}\includegraphics[trim={24.0cm 1.0cm 25cm 3.6cm},clip, width=.49\linewidth]{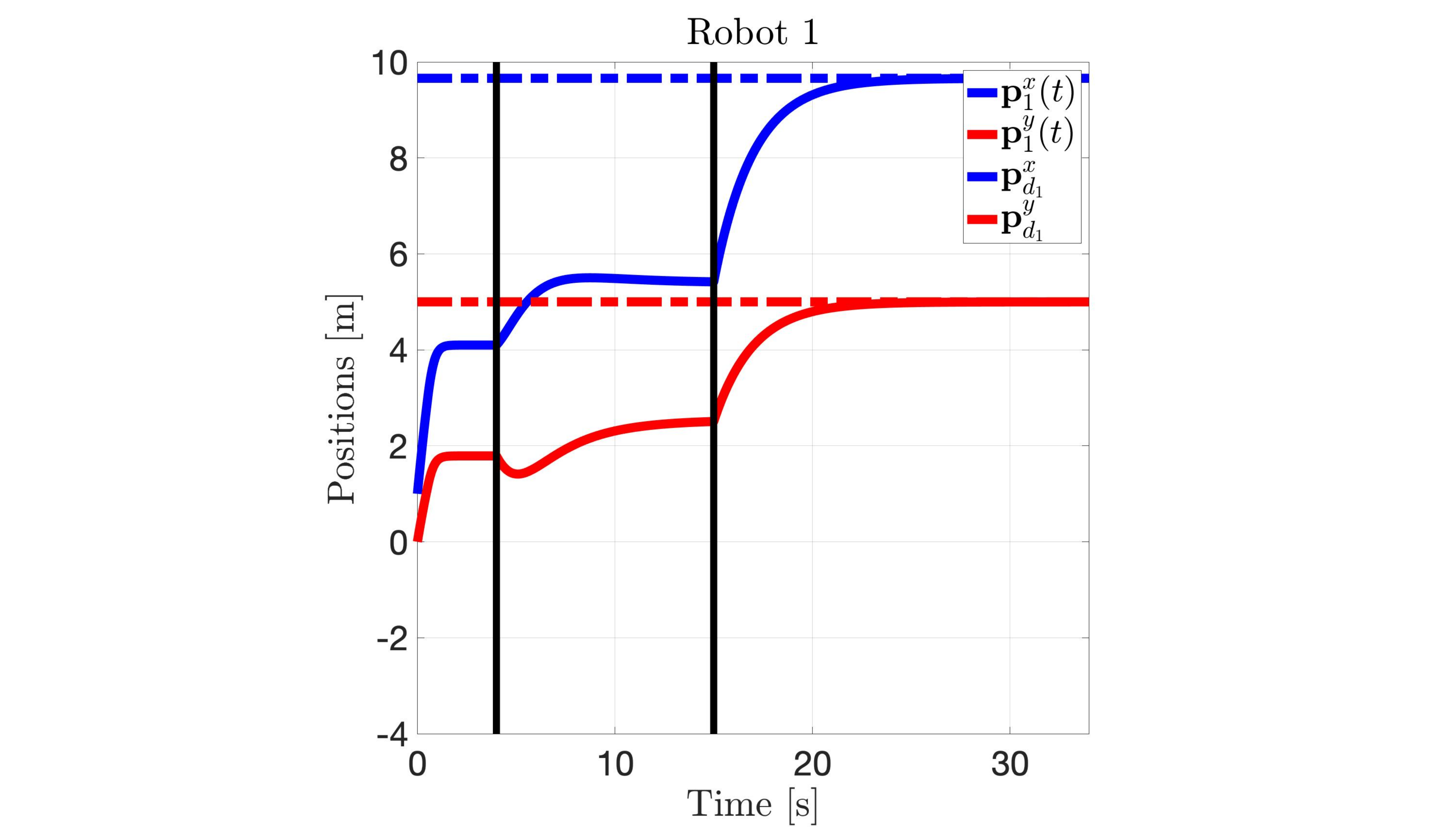}}
	\subfigure[Robot 2 Position]{\label{fig:dd}\includegraphics[trim={24.0cm 1.0cm 25cm 3.6cm},clip, width=.49\linewidth]{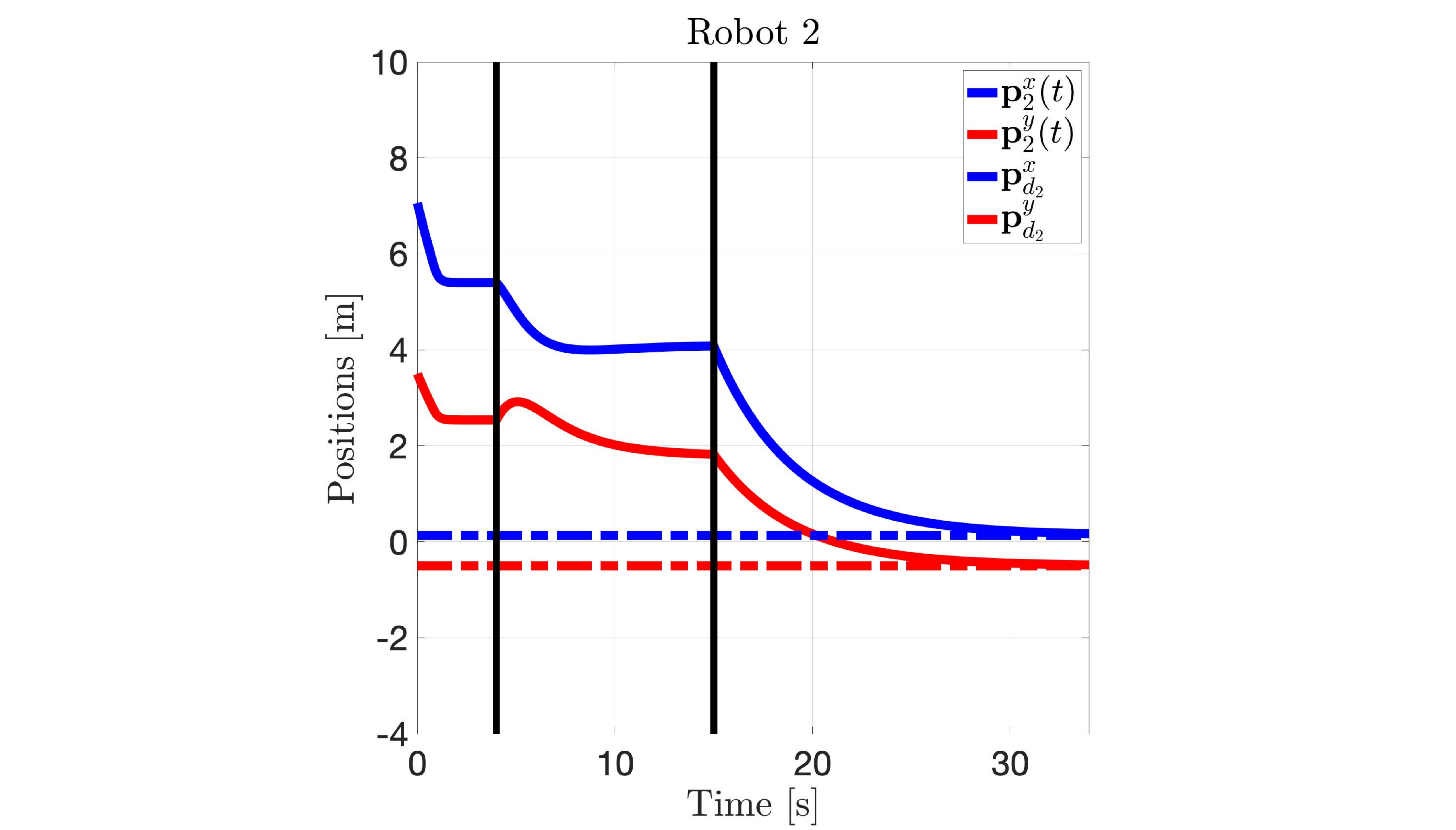}}
	\setlength{\belowcaptionskip}{0pt}
	\caption{Positions of robots from deadlock resolution algorithm for two robots. The vertical bars separate the three phases.  In all figures, final positions converge to desired positions in Phase 3. Simulation video at \url{https://youtu.be/nrWXdn_3nI4} and experimental at \url{https://bit.ly/3hBJIew} }
	\label{fig:simexpres2}
\end{figure}
\begin{figure}[t]
	\centering
	\includegraphics[width=\columnwidth]{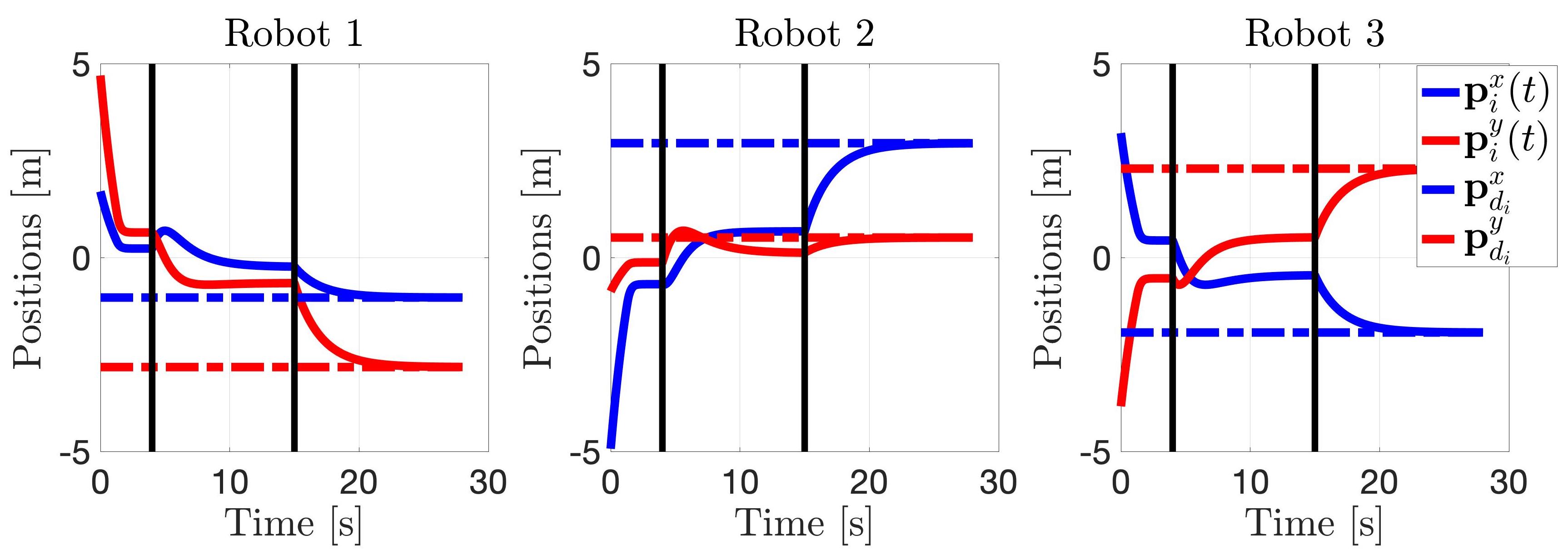}
	\setlength{\belowcaptionskip}{0pt}
	\caption{Positions of robots from deadlock resolution algorithm for three robots. The vertical bars separate the three phases. Final positions converge to desired positions in Phase 3. Simulation video at \url{https://youtu.be/5xbx4mk27xc} and experimental validation at \url{https://youtu.be/IqYmkQUjrBI}}
	\label{fig:simexpres3}
\end{figure}
Fig. \ref{fig:simexpres2} shows simulation  results from running this strategy on two robots. We also conducted an experimental validation of this algorithm using Khepera 4 nonholonomic robots. The videos for experimental results can be found at \url{https://bit.ly/3hBJIew}.  Note that for nonholnomic robots, we noticed from experiments and simulations (in two robot case) that deadlock only occurs if the body frames of both robots are perfectly aligned with one another at $t=0$. Since this alignment is difficult to establish in experiments due to sensor noise, we simulated a virtual deadlock at $t=0$ \textit{i.e.} assumed that the initial position of robots are in deadlock.

\begin{figure*}
	\centering
	\subfigure[$t=0s$, Phase 1]{\label{fig:rr1}\includegraphics[trim={32.0cm 7.0cm 29cm 7.6cm},clip, width=.24\linewidth]{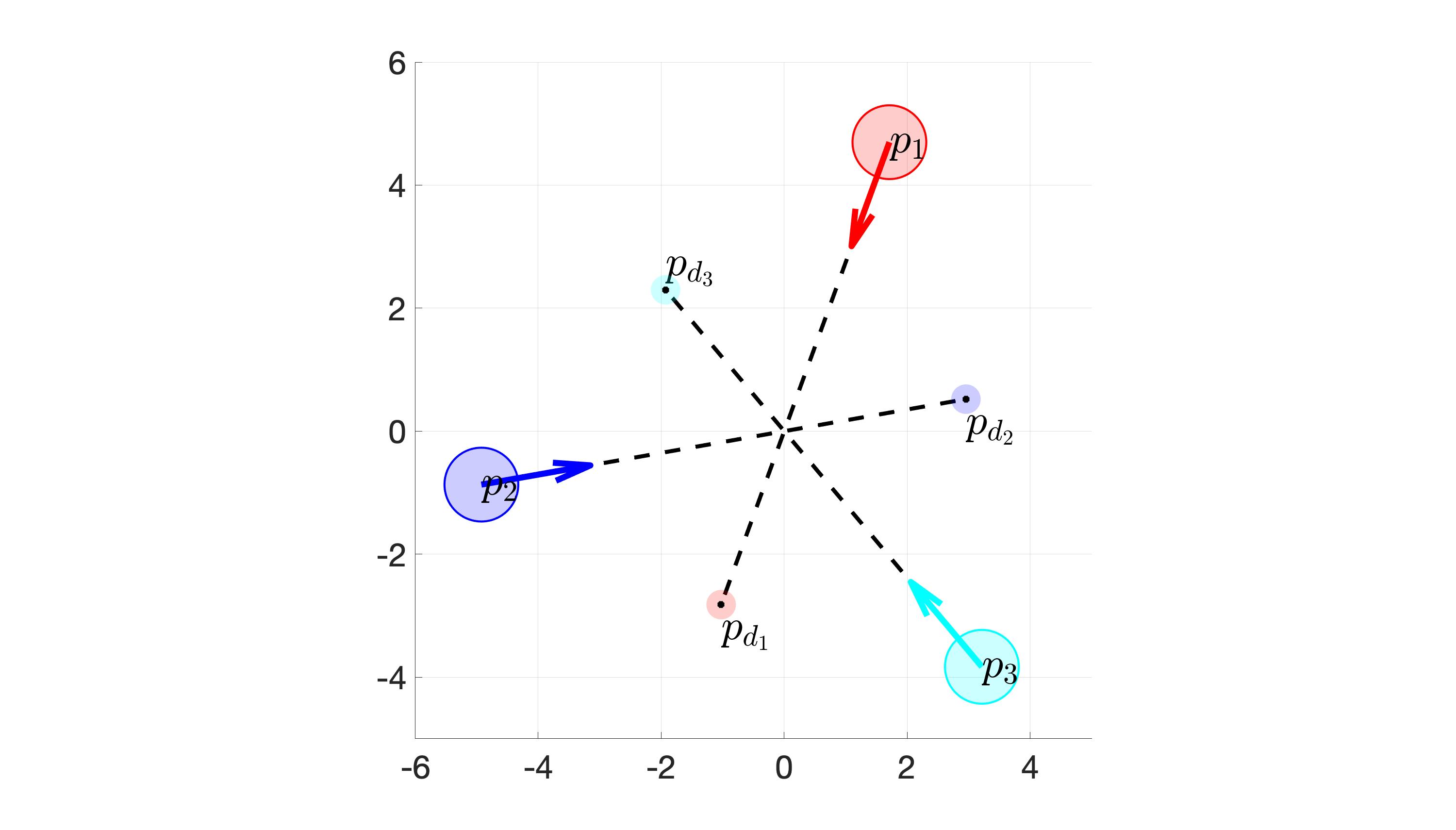}}
	\subfigure[$t=0.5s$, Phase 1]{\label{fig:rr2}\includegraphics[trim={32.0cm 7.0cm 29cm 7.6cm},clip, width=.24\linewidth]{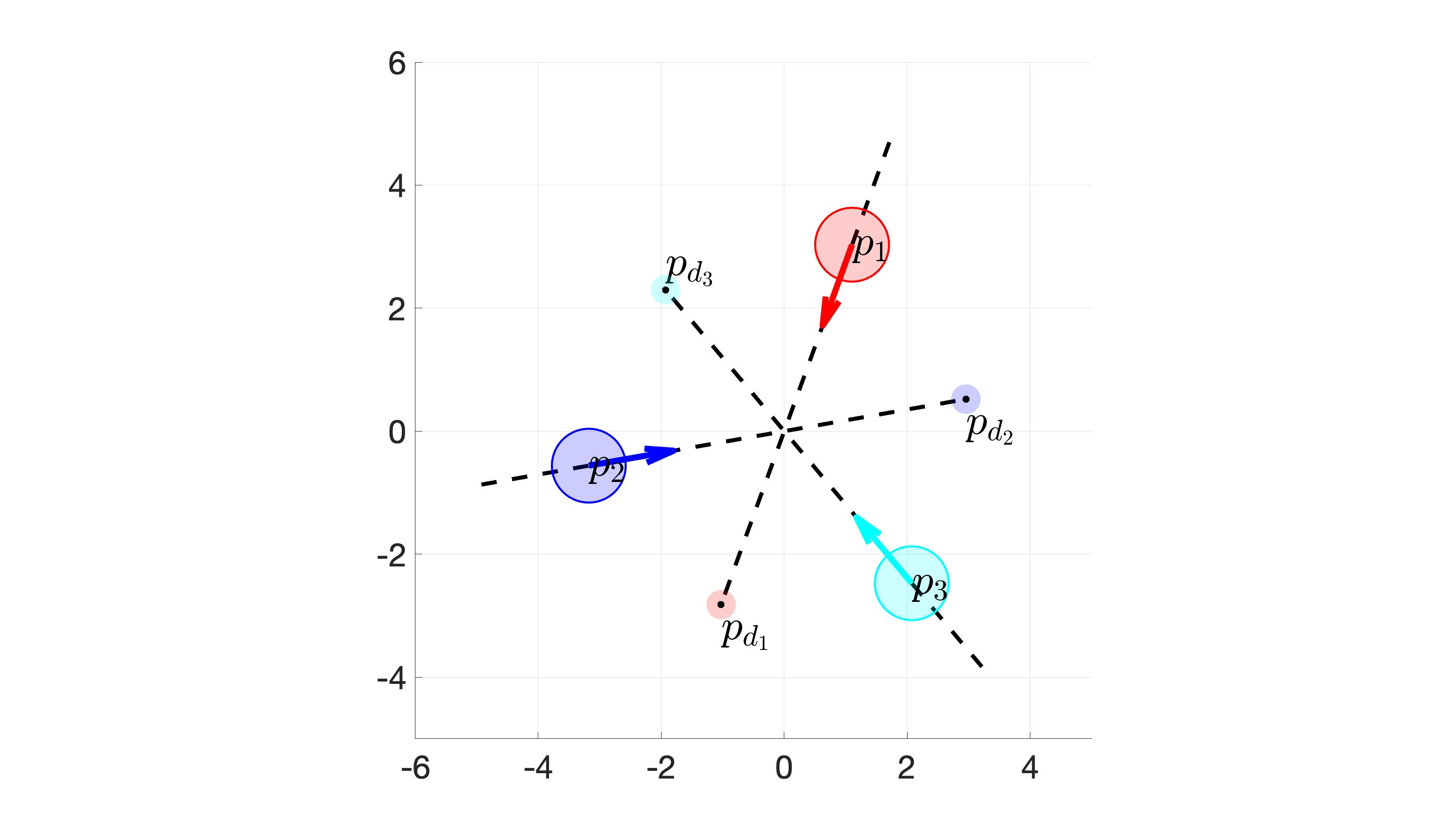}}
	\subfigure[$t=1.0s$, Phase 1]{\label{fig:rr3}\includegraphics[trim={32.0cm 7.0cm 29cm 7.6cm},clip, width=.24\linewidth]{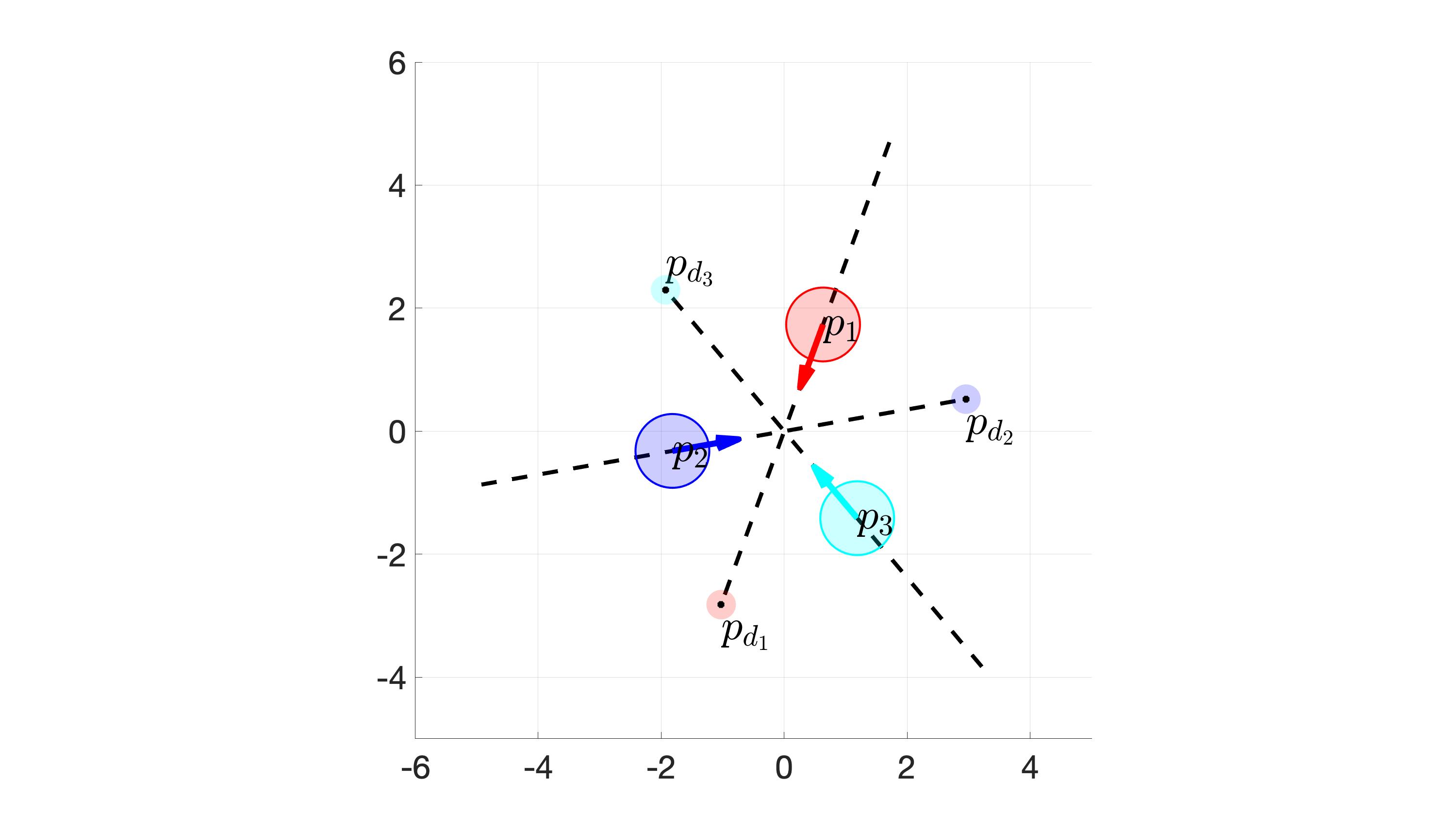}}
	\subfigure[$t=2.5s$, Phase 1]{\label{fig:rr4}\includegraphics[trim={32.0cm 7.0cm 29cm 7.6cm},clip, width=.24\linewidth]{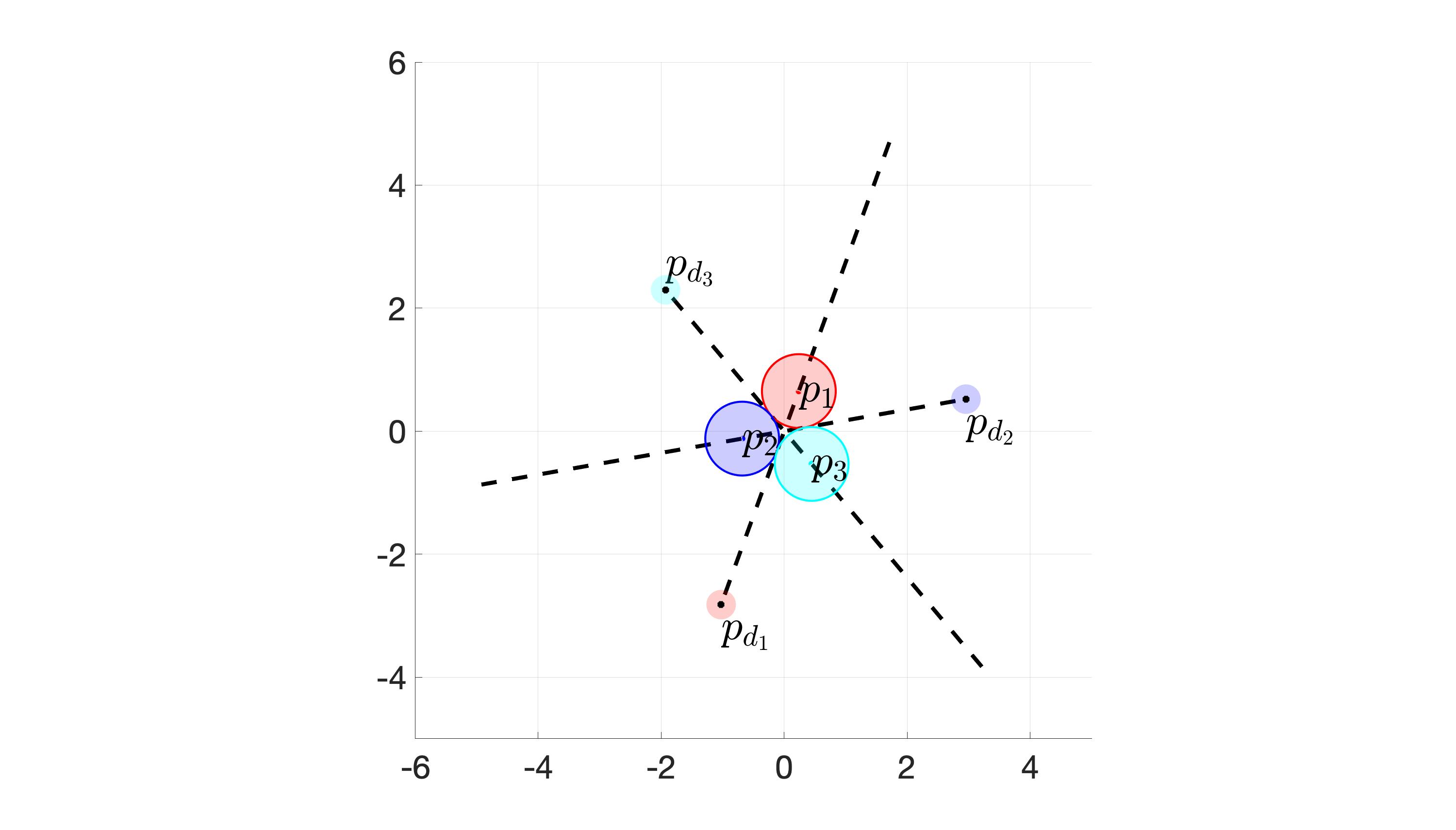}}

	\subfigure[$t=4.0s$, Phase 2]{\label{fig:rr5}\includegraphics[trim={32.0cm 7.0cm 29cm 7.6cm},clip, width=.24\linewidth]{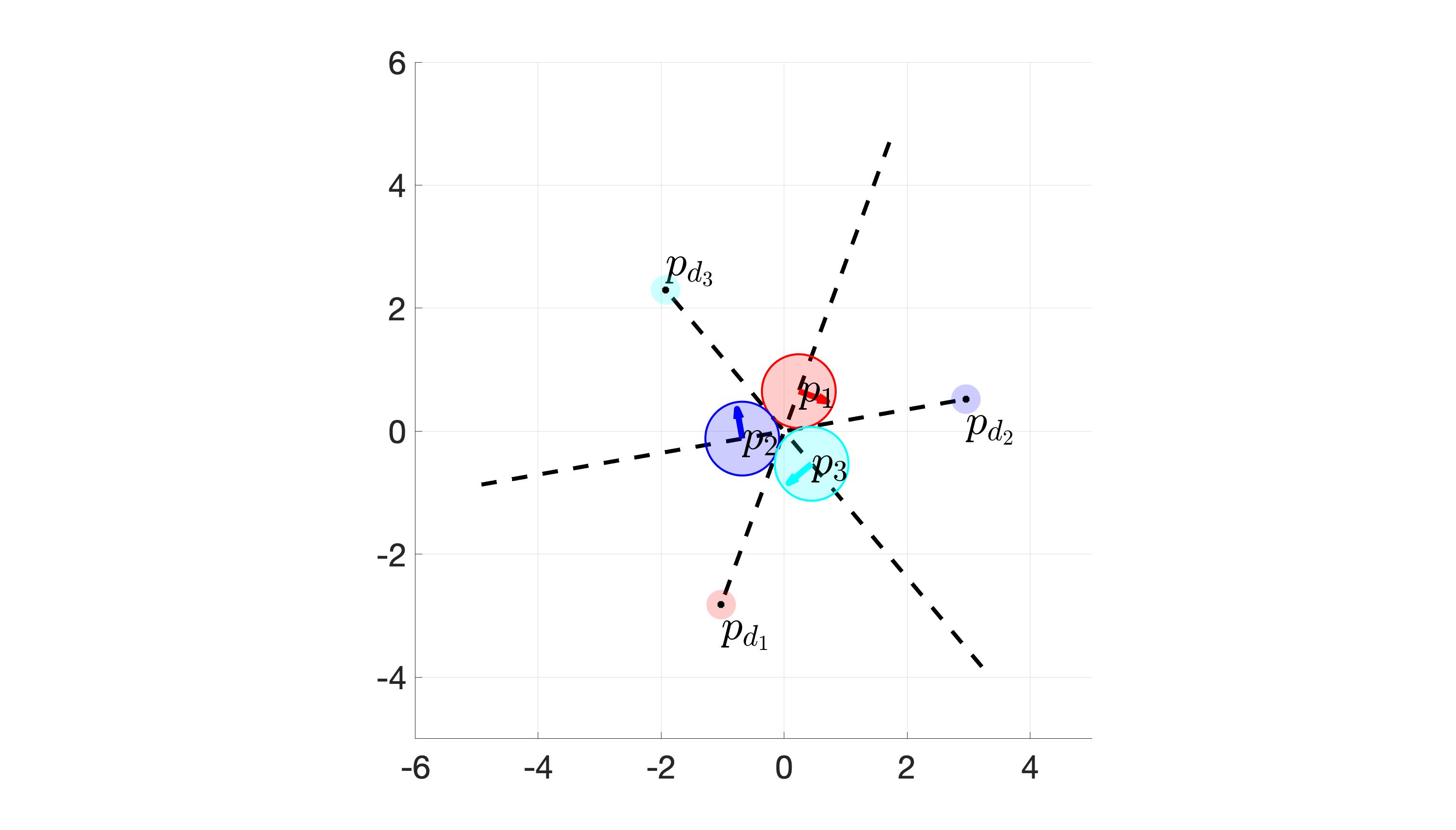}}
	\subfigure[$t=5.0s$, Phase 2]{\label{fig:rr6}\includegraphics[trim={32.0cm 7.0cm 29cm 7.6cm},clip, width=.24\linewidth]{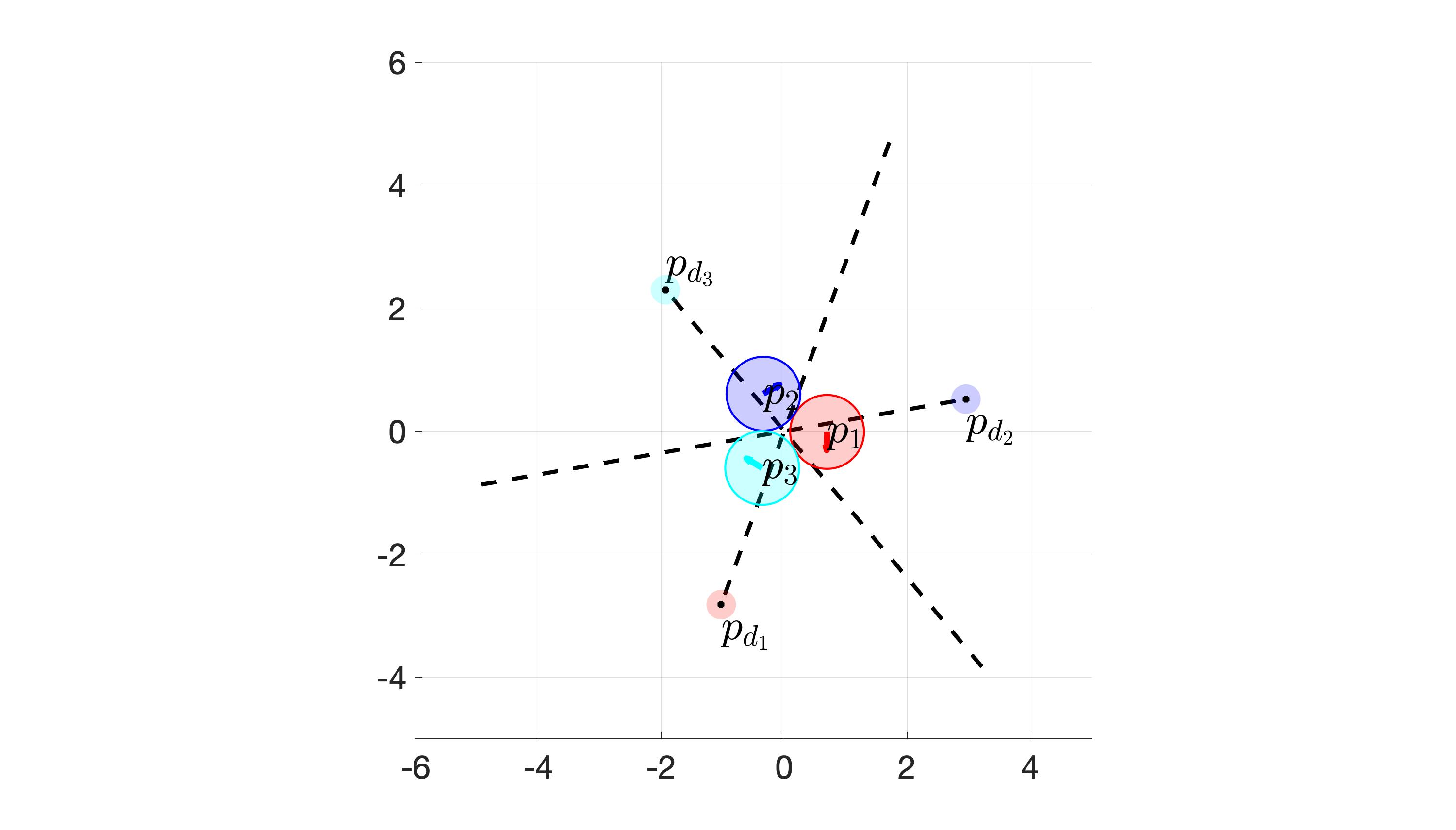}}
	\subfigure[$t=6.0s$, Phase 2]{\label{fig:rr7}\includegraphics[trim={32.0cm 7.0cm 29cm 7.6cm},clip, width=.24\linewidth]{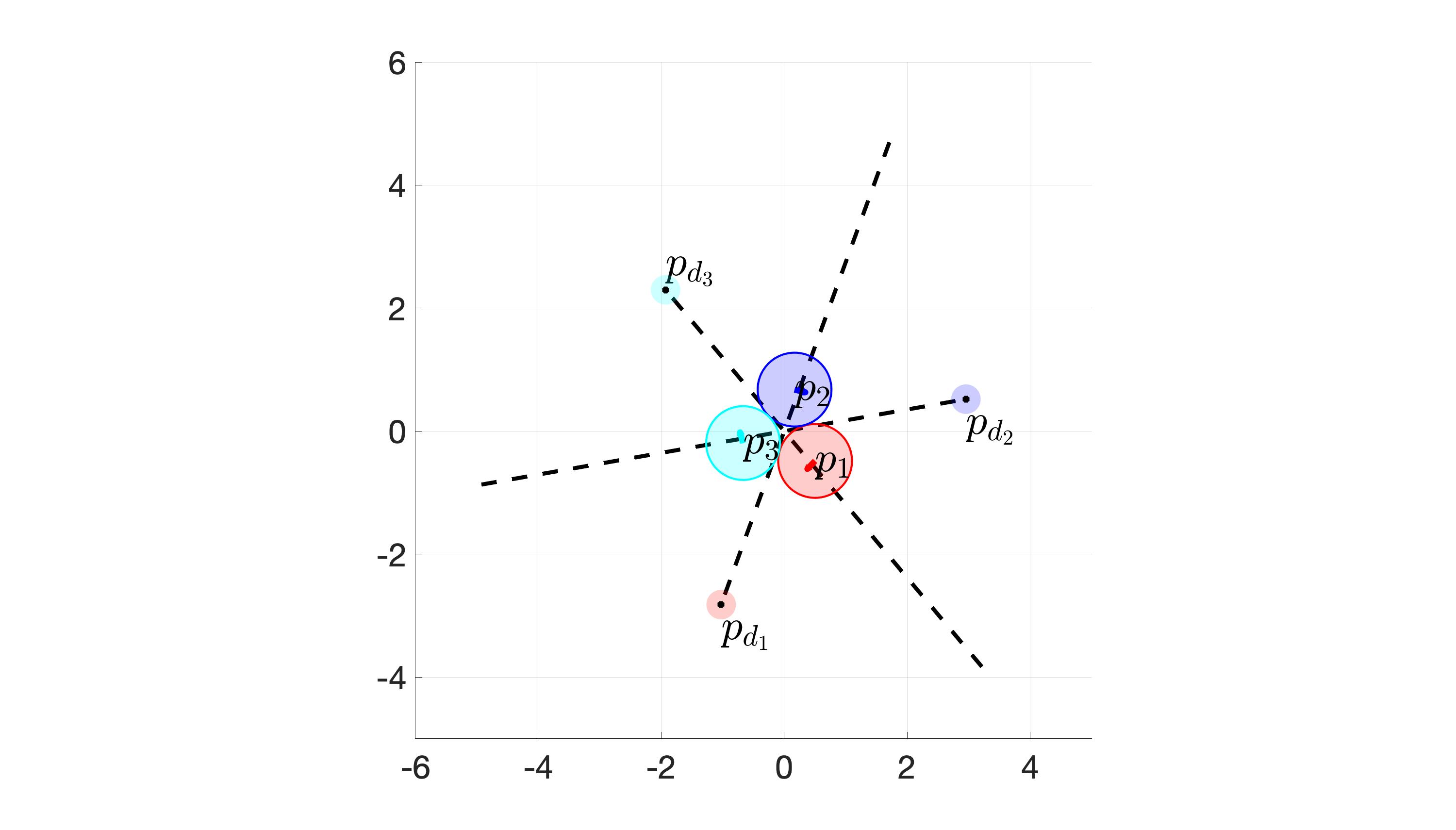}}
	\subfigure[$t=7.5s$, Phase 2]{\label{fig:rr8}\includegraphics[trim={32.0cm 7.0cm 29cm 7.6cm},clip, width=.24\linewidth]{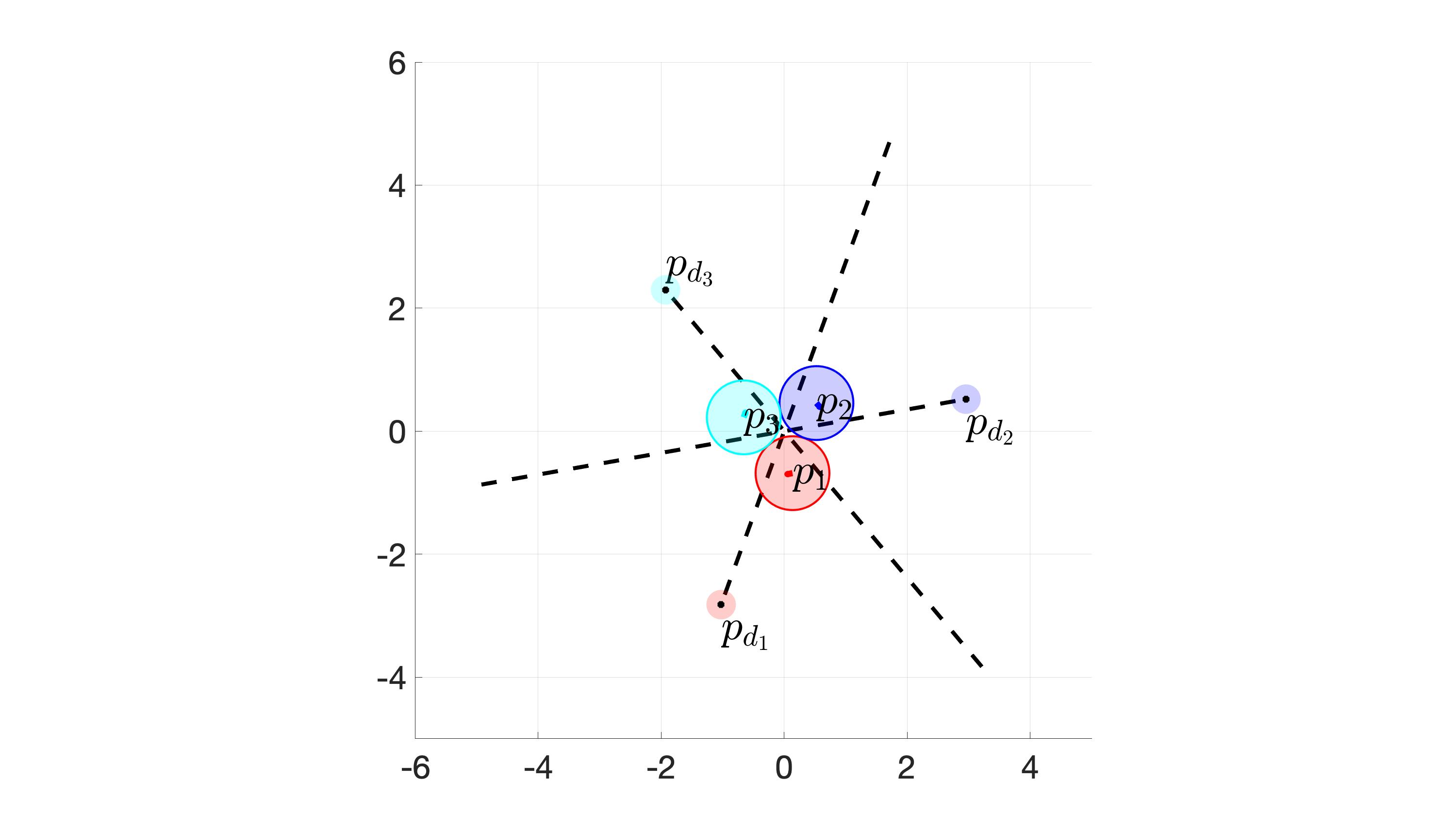}}

	\subfigure[$t=15s$, Phase 3]{\label{fig:rr9}\includegraphics[trim={32.0cm 7.0cm 29cm 7.6cm},clip, width=.24\linewidth]{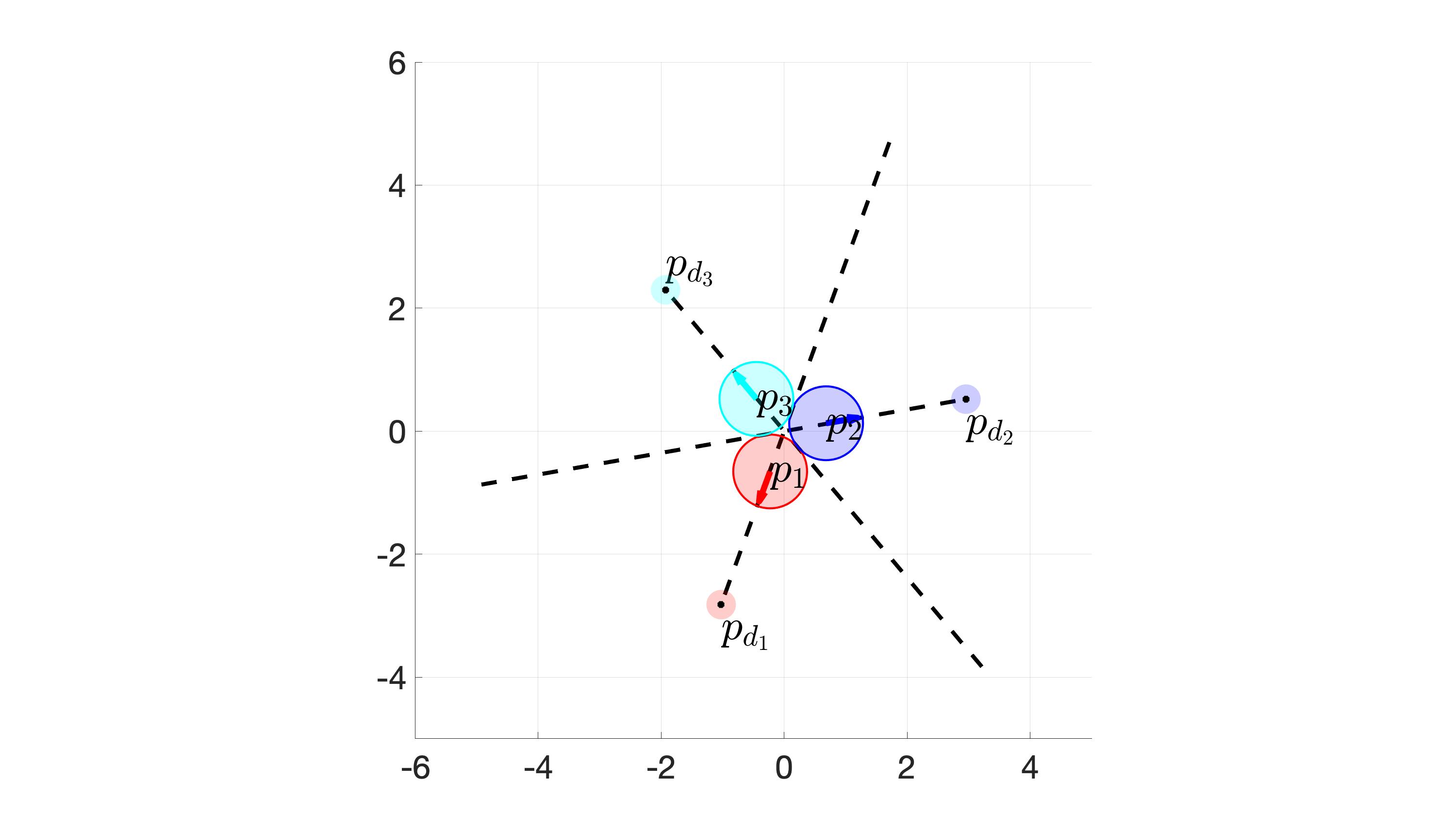}}
	\subfigure[$t=16s$, Phase 3]{\label{fig:rr10}\includegraphics[trim={32.0cm 7.0cm 29cm 7.6cm},clip, width=.24\linewidth]{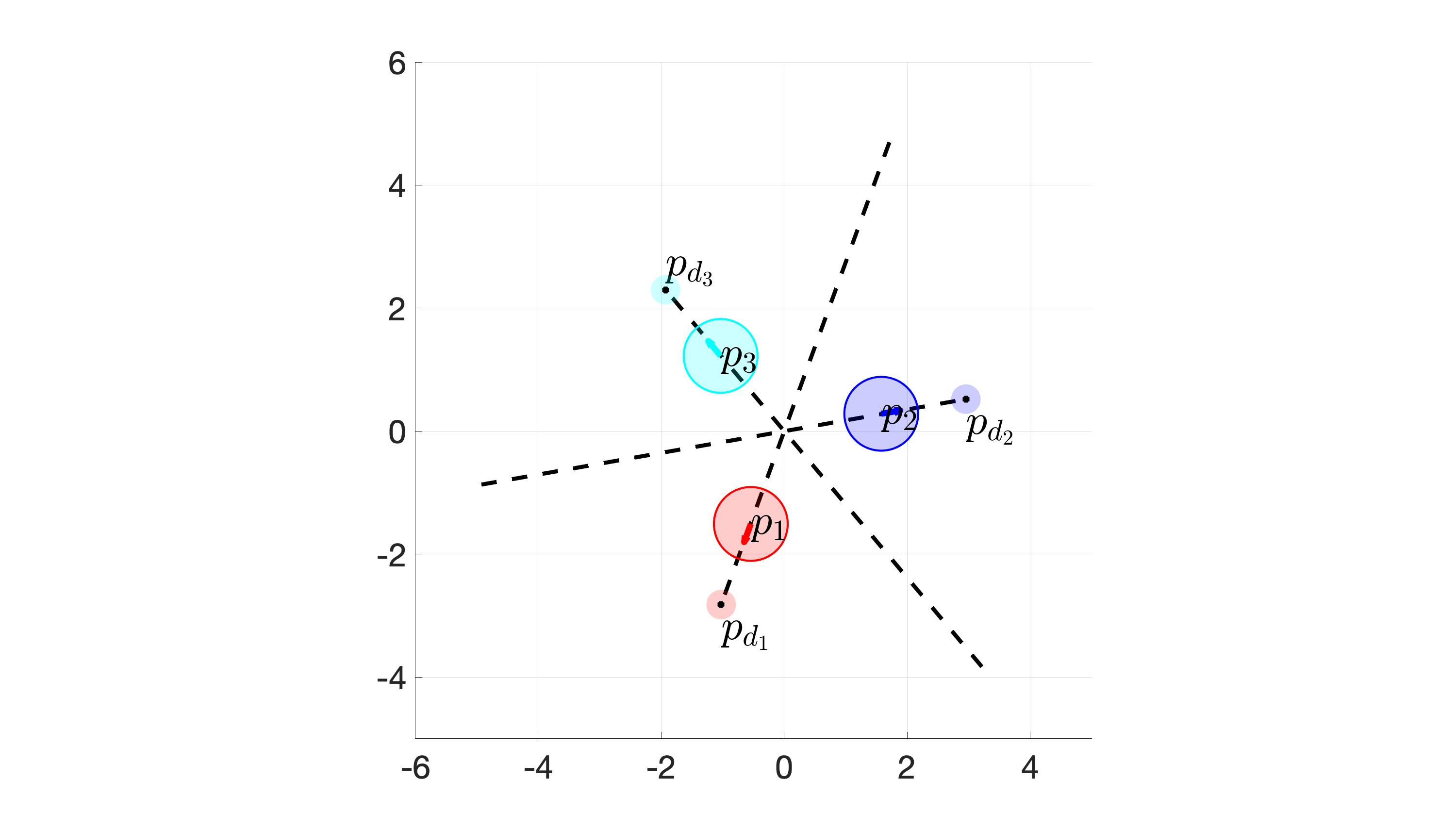}}
	\subfigure[$t=18s$, Phase 3]{\label{fig:rr11}\includegraphics[trim={32.0cm 7.0cm 29cm 7.6cm},clip, width=.24\linewidth]{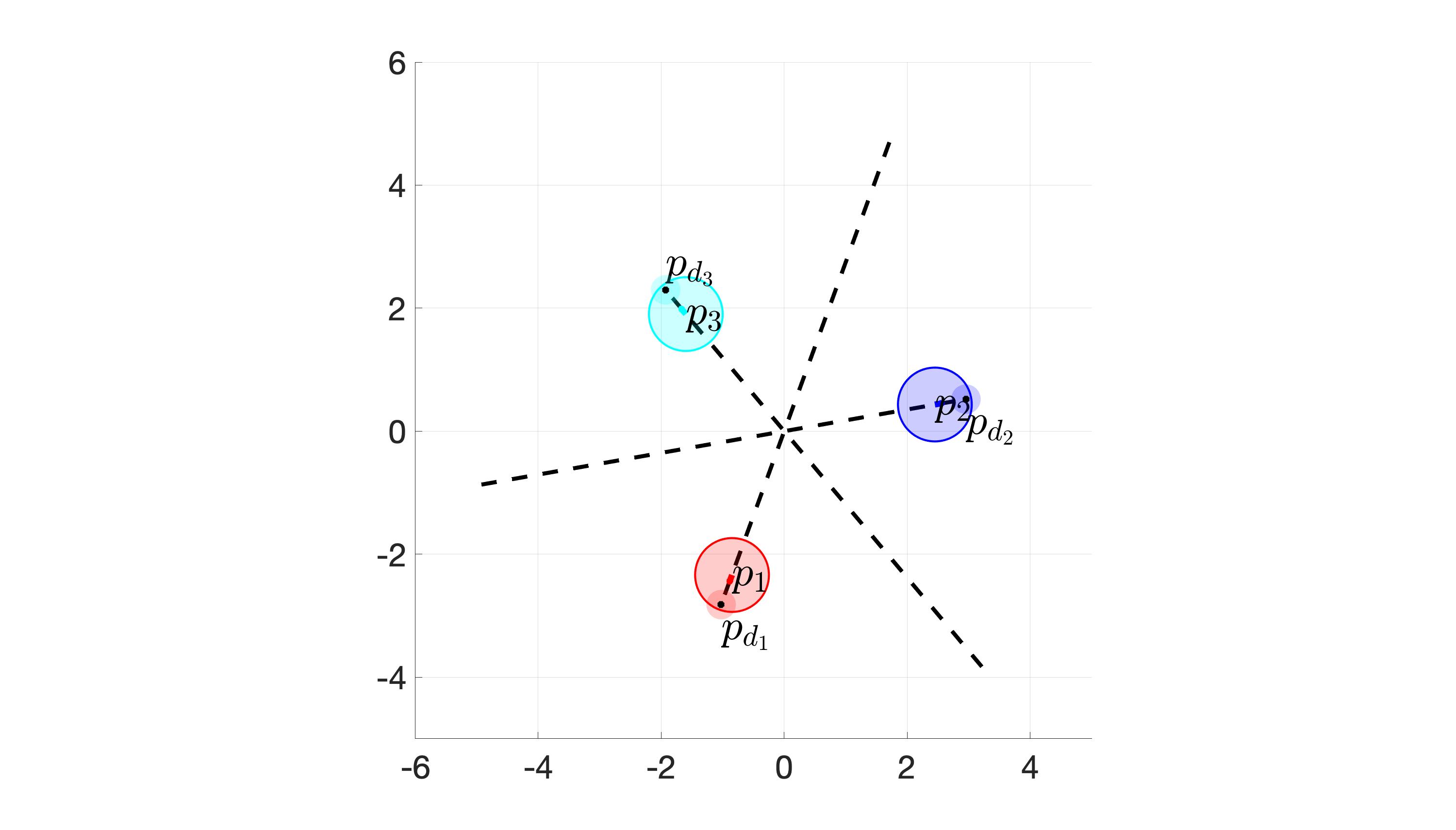}}
	\subfigure[$t=19.5s$, Phase 3]{\label{fig:rr12}\includegraphics[trim={32.0cm 7.0cm 29cm 7.6cm},clip, width=.24\linewidth]{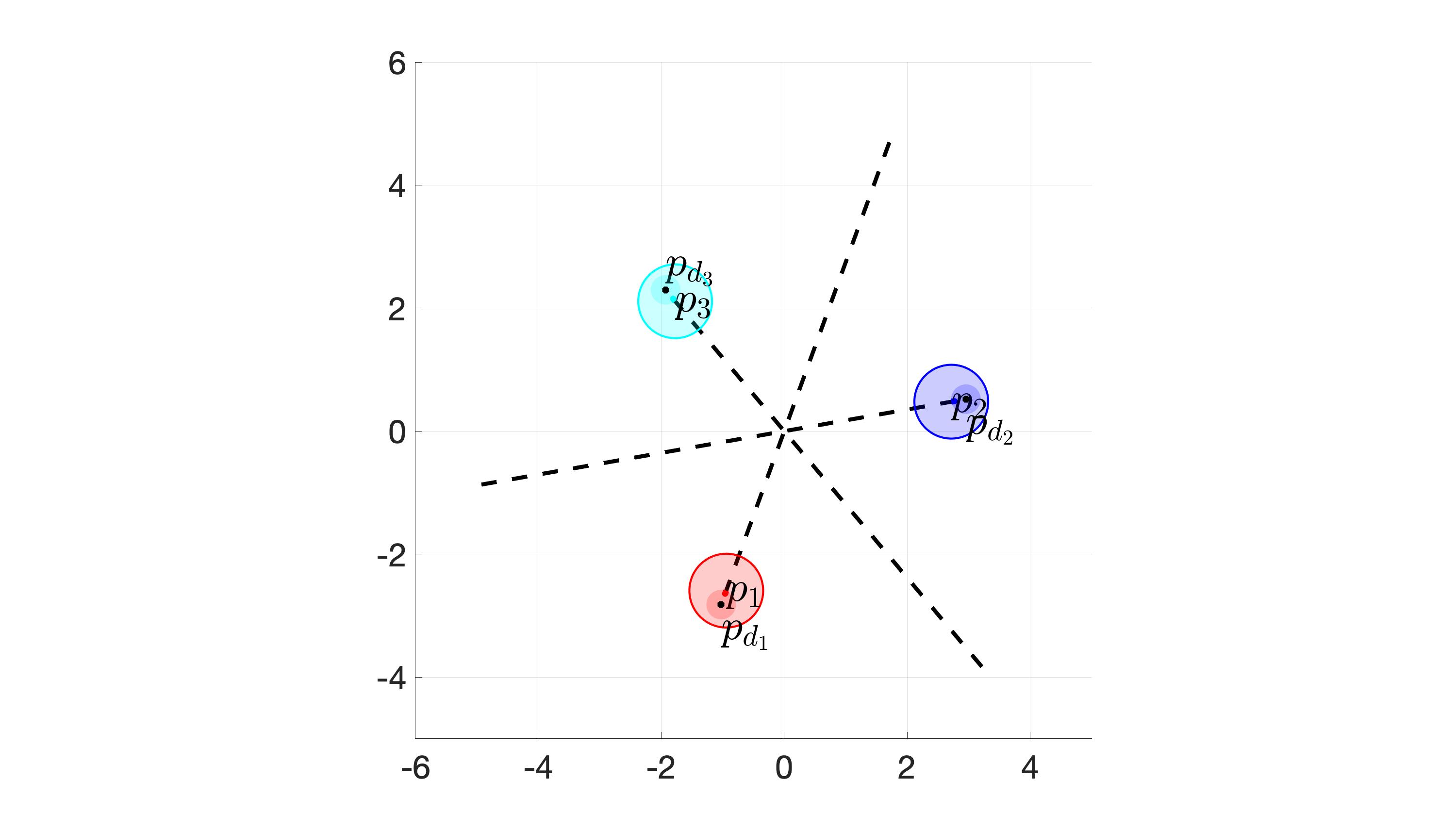}}
	\setlength{\belowcaptionskip}{0pt}
	\caption{Snapshots from running the deadlock resolution algorithm on three robots. The top row corresponds to Phase 1 where we use CBF-QPs. The middle row corresponds to Phase 2 that uses feedback linearization for rotating the robots and the last row corresponds to Phase 3 that uses proportional controls. The simulation video can be watched at \url{https://youtu.be/5xbx4mk27xc}}
	\label{fig:threeressnapshots}
\end{figure*}

Figs. \ref{fig:simexpres3} show the simulation  results from running this strategy on three robots. \textcolor{black}{Fig. \ref{fig:threeressnapshots} shows snapshots of this video during different phases of the execution of the deadlock resolution algorithm.} In three robots, we \textbf{did} observe deadlock experimentally as can be seen here \url{https://youtu.be/e6eOeuh7Uec}. We only observed incidence of Category A deadlock in both simulations and experiments so our resolution algorithm is designed for this category. The videos for experimental results for validation of the resolution algorithm can be found at \url{https://youtu.be/IqYmkQUjrBI}. The video demonstrates that our resolution algorithm was able to safely deliver the robots to their goals. 
We next prove that this strategy ensures resolution of deadlock and convergence of robots to their goals by exploiting  the properties we derived in \cref{touching} and \cref{nonempty}. We prove this theorem for $N=2$ since  $N=3$ is a trivial extension of $N=2$.
\begin{theorem}
\label{finalthm}
The three phase strategy ensures that the robots will not fall back in deadlock and will converge to their goals.
\end{theorem}

\begin{proof} See appendix \ref{finalthm_appendix}. \end{proof}
\section{Conclusions}
\label{conclusions}
\textcolor{black}{In this paper, we analyzed the characteristic properties of deadlock that results from using CBF based QPs for avoidance control in multirobot systems.  We broke our analysis into three phases consistent with the chronology of deadlock incidence. For before deadlock, we demonstrated that symmetry in initial and goal positions leads to deadlocks and heterogeneity in controller parameters may not be enough to evade deadlock. The key lesson that we learned from the examples we considered is that symmetry will always lead to deadlock resulting in robots grazing each other asymptotically. Thus a symmetry breaking controller with some random noise can potentially prevent such configurations from being attained. However even with randomness, deadlock avoidance cannot be taken for granted.  Thus, a guaranteed way to resolve deadlock is to allow it to occur and then correct it using a reactive mechanism. However, devising this corrective strategy necessitates knowing the geometric characteristics of the robot configurations when the system is in deadlock. }

\textcolor{black}{In that spirit, we used the KKT conditions of the CBF-QPs and analyzed the geometric properties of an $N$ robot arrangement when it has fallen in deadlock. We demonstrated how to interpret deadlock as a subset of the state space and proved that in deadlock, the robots are on the verge of violating safety. Additionally, we showed that this set is non-empty and bounded. We also demonstrated that the number of valid geometric configurations in deadlock increases approximately exponentially with the number of robots which makes the analysis and resolution for $N \geq 4$ complex. Finally, by using these properties, we devised a corrective control algorithm to force the robots out of deadlock and ensure task completion.}

\textcolor{black}{There are several directions we would like to explore in future. Firstly, we want to extend this to $N\geq 4$ case. In the $N=4$ case, we determined a large number of admissible geometric configurations. We found that there exist bijections among some of these configurations depending on the number of total active constraints. We believe this property can be exploited to reduce the complexity down to the equivalence classes of these bijections.  We will exploit this line of approach to simplify analysis for $N\geq 4$ cases. Secondly, while in this work we focused on CBF-QPs for analysis, our developed theory can be used to analyze properties of deadlock resulting from other optimization based such as velocity obstacles and the safe-set method. So in future, we will evaluate the generalizabilty of our derived results specifically the geometric properties of robot arrangements across the spectrum of reactive optimization based controllers, and explore properties that make a particular algorithm immune to deadlock.} 
\appendix
\section{Appendix}
\subsection{Two Robot Deadlock Results}
\begin{lemma}
\label{lemma1_appendix}
If at $t=0$, $D_{init}>\beta^i_{+}$, then $\boldsymbol{u}^*_i(0) = \boldsymbol{\hat{u}}_i(0)$, \textcolor{black}{\textit{i.e.} the collision avoidance constraint of robot $i$ is inactive.}
\end{lemma}
\begin{proof}
To determine the control returned by the optimization for robot one at $t=0$, we compute the value of the flag $f_{12}(\hat{\boldsymbol{u}}_1(0))$. \textcolor{black}{This flag evaluates whether $\hat{\boldsymbol{u}}_1(0)$ is a feasible controller \textit{i.e.} it prevents collision of robot 1 with 2. Based on the initial positions and goals from  \eqref{initial_positions_two_robots}-\eqref{goal_positions_two_robots}, we have}
\begin{align}
\label{expressionforg12_two_robot}
    f_{12}(\hat{\boldsymbol{u}}_1(0)) &= \boldsymbol{a}_{12}^T(0)\boldsymbol{\hat{u}}_1(0) -b_{12}(0) \nonumber \\
    &=k_{p_1}(\boldsymbol{p}_{1}(0)-\boldsymbol{p}_2(0))^T(\boldsymbol{p}_{1}(0)-\boldsymbol{p}_{d_1}) \nonumber \\
    &- \frac{\gamma}{4}(D^2_{init.} - D_s^2) \nonumber \\
    &= -\frac{\gamma}{4}D_{init}^2 + D_{G_1}k_{p_1}D_{init} + \frac{\gamma}{4}D_s^2 \nonumber \\
    &\coloneqq g_{12}(D_{init.})
\end{align}
Define $  g_{12}(D_{init.}) \coloneqq f_{12}(\hat{\boldsymbol{u}}_1(0))$ to emphasize dependence on $D_{init.}$, the initial distance between the robots. \eqref{expressionforg12_two_robot} is a quadratic polynomial in $D_{init}$  and has two zeros at
\begin{align}
\label{dcrit1}
    \beta^1_\pm = \frac{2D_{G_1}k_{p_1}}{\gamma} \pm \sqrt{\bigg(\frac{2D_{G_1}k_{p_1}}{\gamma}\bigg)^2  + D_s^2}   
\end{align}
where the subscript of $\beta$ indicates the sign of $\beta$. Now, since the graph of $g_{12}(D_{init})$ is a downward facing parabola, we know that $g_{12}(D_{init})<0$ $\forall$ $D_{init} \in (-\infty,\beta^1_{-})\cup (\beta^1_{+},\infty) \implies g_{12}(D_{init}) <0$  $\forall$ $D_{init} \in (\beta^1_{+},\infty)$. We call $\beta^{1}_{+}$ to be the critical distance for robot 1. If at $t=0$, the distance between the robots is such that $D_{init}>\beta^{1}_{+}$, then at $t=0$, $g_{12}(D_{init.})<0 \iff f_{12}(\hat{\boldsymbol{u}}_1(0)) <0 \implies \boldsymbol{u}^*_1(0) = \boldsymbol{\hat{u}}_1(0)$.  Similarly, we can compute $g_{21}(D_{init.}) \coloneqq  f_{21}(\hat{\boldsymbol{u}}_2(0))$ which has roots at $\beta^2_{\pm}$. Hence, if at $t=0$, $ D_{init}>\beta^2_+$, then $\boldsymbol{u}^*_2(0)=\boldsymbol{\hat{u}}_{2}(0)$.  \qed
\end{proof}

\begin{lemma}
\label{lemma2_appendix}
$\exists$ a finite time $t_1$ as described in Def. \eqref{time_def}, until which the collision avoidance constraints of both robots are simultaneously inactive. 
\end{lemma}
\begin{proof}
Since at $t=0$, both robots use their prescribed controls $\boldsymbol{\hat{u}}_i$, let us assume that they continue to do so for the interim. Therefore, the dynamics of the robots are
\begin{align}
\label{dynamics2}
    \dot{\boldsymbol{p}}_i = \boldsymbol{\hat{u}}_i= -k_{p_i}(\boldsymbol{p}_i-\boldsymbol{p}_{d_i}).
    \end{align}
 \eqref{dynamics2} can be integrated analytically using the initial and desired positions of the two robots from  \eqref{initial_positions_two_robots} and \eqref{goal_positions_two_robots} to give
\begin{align}
\label{postime}
\boldsymbol{p}_1(t) &=\boldsymbol{p}_{d_1}-D_{G_1}e^{-k_{p_1}t}\hat{\boldsymbol{e}}_\alpha \nonumber \\
 \boldsymbol{p}_2(t) &=\boldsymbol{p}_{d_2}+D_{G_2}e^{-k_{p_2}t}\hat{\boldsymbol{e}}_\alpha
 \end{align}
We can compute the relative position between 1 and 2 as:
 \begin{align}
 \label{Doft1}
 \Delta \boldsymbol{p}_{21}(t)  &=(D_{G_1}e^{-k_{p_1}t}+ D_{G_2}e^{-k_{p_2}t}-K) \hat{\boldsymbol{e}}_\alpha \nonumber \\
 &=D(t)\hat{\boldsymbol{e}}_\alpha 
 \end{align}
 where $K\coloneqq D_{G_1}+D_{G_2}-D_{init}$. \textcolor{black}{$D(t)$ denotes the inter-robot distance as a function of time for the duration in which both robots continue to use their proportional controllers. From \cref{Doft1}, it is given by}
 \begin{align}
 \label{distancetime}
     D(t) \coloneqq (D_{G_1}e^{-k_{p_1}t}+ D_{G_2}e^{-k_{p_2}t})-K,
 \end{align}
\textcolor{black}{The nominal controls can are given by}
 \begin{align}
 \label{prescribedalongalpha1}
  \boldsymbol{\hat{u}}_1(t) &=-k_{p_1}(\boldsymbol{p}_1(t)-\boldsymbol{p}_{d_1}) &= +k_{p_1} D_{G_1}e^{-k_{p_1}t}\hat{\boldsymbol{e}}_\alpha \\
  \label{prescribedalongalpha2}
  \boldsymbol{\hat{u}}_2(t) &=-k_{p_2}(\boldsymbol{p}_2(t)-\boldsymbol{p}_{d_2}) &= -k_{p_2} D_{G_2}e^{-k_{p_2}t}\hat{\boldsymbol{e}}_\alpha 
 \end{align}
\textcolor{black}{Now, we use \cref{distancetime}-\cref{prescribedalongalpha2} to compute the feasibility flag for robot 1 using \eqref{constraints_time1}}
\begin{align}
\label{parabola}
    f_{12}(\boldsymbol{\hat{u}}_1(t)) &=\boldsymbol{a}^T_{12}(t)\boldsymbol{\hat{u}}_1(t) - b_{12}(t) \nonumber \\
    &=-\frac{\gamma}{4}D^2(t)+k_{p_1}D_{G_1}e^{-k_{p_1}t}D(t)+\frac{\gamma}{4}D^2_s 
\end{align}
which is a downwards facing parabola in $D(t)$. Thus, $f_{12}(\boldsymbol{\hat{u}}_1(t))<0$ for $D(t)>\beta^1_{+}(t)$ where
\begin{align}
\label{time_crit_1}
\beta^1_{+}(t)=\frac{2D_{G_1}k_{p_1}e^{-k_{p_1}t}}{\gamma} + \sqrt{\bigg(\frac{2D_{G_1}k_{p_1}e^{-k_{p_1}t}}{\gamma} \bigg)^2 + D_s^2}    
\end{align}
and likewise we can define $\beta^2_{+}(t)$ for robot two by replacing $k_{p_1}, D_{G_1}$ with $k_{p_2}, D_{G_2}$ in  \eqref{time_crit_1}. Note that $D(t)$, $\beta^1_{+}(t)$ and $\beta^2_{+}(t)$ are monotonically decreasing with time. Additionally, recall from  \eqref{initial_positions_two_robots} that $D(0)=D_{init}$ and from Assumption \ref{ass2} that $D_{init.}>\beta^1_+$. However, while $D(t)$ converges to $-(D_{G_1}+D_{G_2}-D_{init})<0$, $\beta^1_{+}(t),\beta^2_{+}(t)$ converge to $D_s>0$. Therefore, there exists a time $t_a$ at which $D(t_a)=\beta^1_{+}(t_a) \implies f_{12}(\boldsymbol{\hat{u}}_1(t_a))=0$ and a time $t_b$ at which $D(t_b)=\beta^2_{+}(t_b) \implies f_{21}(\boldsymbol{\hat{u}}_2(t_b))=0$. 
We assume WLOG that $\beta^1_+(t_a)>\beta^2_+(t_a) \implies t_a<t_b$ and hence we define:
\begin{align}
    t_1 &\coloneqq \min\{t_a,t_b\}  =t_a
\end{align}
$t_a$ is the time at which the collision avoidance constraint of robot one becomes active \textit{i.e.} $\boldsymbol{a}^T_{12}(t_a)\boldsymbol{u}^*_1(t_a)=b_{12}(t_a)$ while the collision avoidance constraint of robot two is still inactive \textit{i.e.} $\boldsymbol{a}^T_{21}(t_a)\boldsymbol{\hat{u}}_2(t_a)<b_{21}(t_a)$. Thus $t_a$ is the time in Def. \ref{time_def}, \textit{i.e.} the maximum time until which both robots' constraints are inactive. See  \cref{fig:phases} for an illustration of phase 1 which ends at $t=t_1$ when $D(t)=\beta^1_+(t)$. Finally, $D(t)$ is monotonically decreasing $\forall t$ in \eqref{Doft1} and $D(0)=D_{init}>0$, $D(t_1)=\beta^1_+(t_1)>0$, therefore, $D(t)>0$ $\forall t \in [0,t_1)$. From \eqref{Doft1}, note that the distance between the robots is $\norm{\Delta \boldsymbol{p}_{21}(t)}\coloneqq |D(t)|=D(t)$ $\forall t \in [0,t_1)$. Therefore, $\norm{\Delta \boldsymbol{p}_{21}(t)}$ is also monotonically \textit{decreasing} until $t_1$. \qed
\end{proof}

\begin{lemma}
\label{lemma3_appendix}
$\exists$ a finite time $t_2$ as in Def. \eqref{time_def2}, until which the constraint  of robot two stays inactive.
\end{lemma}
\begin{proof}
Let us assume that there is a time $t_2$ until which the collision avoidance constraint of robot two remains inactive. \textcolor{black}{Therefore, the dynamics of robot two are governed by the prescribed nominal control.} Thus,
\begin{align}
\label{pos2inactive}
    \dot{\boldsymbol{p}}_2 &= \hat{\boldsymbol{u}}_{2} = -k_{p_2}(\boldsymbol{p}_2  -\boldsymbol{p}_{d_2}) \nonumber \\
\implies  \boldsymbol{p}_2(t) &=\boldsymbol{p}_{d_2}+D_{G_2}e^{-k_{p_2}t}\hat{\boldsymbol{e}}_\alpha
 \end{align}
On the other hand, the control input to robot one is
\begin{align}
\label{robot1active}
   \boldsymbol{u}^*_{1} = \hat{\boldsymbol{u}}_{1} - \frac{1}{2}\mu_{12}\boldsymbol{a}_{12},
\end{align}
where $\mu_{12} \neq 0$ because the collision avoidance constraint of robot one became active at $t_1$ \textit{i.e.} $ \boldsymbol{a}^T_{12}(t)\boldsymbol{u}^*_1(t)=b_{12}(t)$ $\forall$ $t\geq t_1$. It's expression is given as follows
\begin{align}
    \mu_{12} = 2\frac{\boldsymbol{a}^T_{12}\boldsymbol{\hat{u}}_{1}-b_{12}}{\norm{\boldsymbol{a}_{12}}^2}
\end{align}
Therefore, the dynamics of robot one $\forall t \geq t_1$ are
\begin{align}
\label{robot1phase2}
    \dot{\boldsymbol{p}}_1 &= \hat{\boldsymbol{u}}_{1} -  \frac{\boldsymbol{a}^T_{12}\boldsymbol{\hat{u}}_{1}-b_{12}}{\norm{\boldsymbol{a}_{12}}^2}\boldsymbol{a}_{12} \nonumber \\
    &= \underbrace{\hat{\boldsymbol{u}}_{1}- \frac{\boldsymbol{a}^T_{12}\boldsymbol{\hat{u}}_{1}}{\norm{\boldsymbol{a}_{12}}^2}\boldsymbol{a}_{12}}_{\boldsymbol{u}_{\perp}} + \underbrace{\frac{b_{12}}{\norm{\boldsymbol{a}_{12}}^2}\boldsymbol{a}_{12}}_{\boldsymbol{u}_{\parallel} }
\end{align}
On closer inspection, note that $\boldsymbol{u}_{\perp} \perp \boldsymbol{a}_{12} \iff \boldsymbol{u}_{\perp} \perp -\Delta \boldsymbol{p}_{12}$. 
\textcolor{black}{We want to derive an analytical expression of $\boldsymbol{p}_1(t)$ just like we did for robot 2 in \cref{pos2inactive}. However, it is difficult to analytically integrate  \eqref{robot1phase2}. Instead, we will show via recursion that the $\boldsymbol{p}_1(t)$  can be expressed as}
 \begin{align}
 \label{robot1generalform}
   \boldsymbol{p}_1(t) &=\boldsymbol{p}_{d_1}-D_{G_1}\eta(t)\hat{\boldsymbol{e}}_\alpha
 \end{align}
for some function $\eta(t)$ $\forall t \geq t_1$. This expression is valid at $t=t_1$ for $\eta(t_1)=e^{-k_{p_1}t_1}$ as shown in  \eqref{postime}. This representation highlights that the position of robot one is confined along $\hat{\boldsymbol{e}}_\alpha$. We will show that this property is  maintained throughout phase 2 because the component of velocity input to robot one perpendicular to $\hat{\boldsymbol{e}}_\alpha $ will vanish (which will become the cause of \textit{deadlock}).  Recall at $t=t_1$
\begin{align}
 \label{Doft}
\boldsymbol{a}_{12}(t_1)= \Delta \boldsymbol{p}_{21}(t_1) 
 &=D(t_1)\hat{\boldsymbol{e}}_\alpha 
 \end{align}
\begin{align}
\label{uperp}
   \mbox{Therefore, }\boldsymbol{u}_{\perp}(t_1) &=   \hat{\boldsymbol{u}}_{1}- \frac{\boldsymbol{a}^T_{12}\boldsymbol{\hat{u}}_{1}}{\norm{\boldsymbol{a}_{12}}^2}\boldsymbol{a}_{12} \nonumber \\
   &= k_{p_1}D_{G_1}\eta(t_1)\hat{\boldsymbol{e}}_\alpha \nonumber  \\ 
   &- \frac{D(t_1)\hat{\boldsymbol{e}}^T_\alpha\big(k_{p_1}D_{G_1}\eta(t_1)\hat{\boldsymbol{e}}_\alpha\big)}{D^2(t_1)}(D(t_1)\hat{\boldsymbol{e}}_\alpha) \nonumber \\
   &= \boldsymbol{0} \nonumber \\
   \boldsymbol{u}_{\parallel}(t_1) &= \frac{b_{12}}{\norm{\boldsymbol{a}_{12}}^2}\boldsymbol{a}_{12}   \nonumber \\
   &= \gamma \frac{D^2(t_1)-D_s^2}{4D(t_1)}\hat{\boldsymbol{e}}_\alpha
\end{align}
Thus, integrating the velocity for a small time step gives
\begin{align}
\label{recursive}
\boldsymbol{p}_1(t_1+\Delta t) &=   \boldsymbol{p}_1(t_1) + \Delta t \big(\boldsymbol{u}_{\perp}(t_1)+\boldsymbol{u}_{\parallel}(t_1)\big) \nonumber \\
    &= \boldsymbol{p}_{d_1}-D_{G_1}\bigg(\eta(t_1)-\Delta t \gamma \frac{D^2(t_1)-D_s^2}{4D_{G_1}D(t_1)}\bigg)\hat{\boldsymbol{e}}_\alpha \nonumber \\
    & = \boldsymbol{p}_{d_1}-D_{G_1}\eta(t_1+\Delta t)\hat{\boldsymbol{e}}_\alpha 
\end{align}
Through  \eqref{recursive}, we have demonstrated that the updated position of robot one admits the general form given by  \eqref{robot1generalform} because the perpendicular component of the velocity vanishes. As a result, the robot never acquires any displacement along the perpendicular component. Hence, the dynamics of robot one are always
\begin{align}
  \dot{\boldsymbol{p}}_1 = \boldsymbol{u}_{\parallel}=\gamma \frac{\norm{\Delta \boldsymbol{p}_{21}}^2 -D_s^2}{4\norm{\Delta \boldsymbol{p}_{21}}^2}\Delta \boldsymbol{p}_{21},
\end{align}
$\forall t\geq t_1$. The relative dynamics are
\begin{align}
    \dot{\Delta \boldsymbol{p}_{21}}&=-\gamma \frac{\norm{\Delta \boldsymbol{p}_{21}}^2 -D_s^2}{4\norm{\Delta \boldsymbol{p}_{21}}^2}\Delta \boldsymbol{p}_{21} - k_{p_2}(\boldsymbol{p_2}-\boldsymbol{p}_{d_2})
\end{align}
Let $\Delta \boldsymbol{p}_{21}(t)=D(t)\hat{\boldsymbol{e}}_\alpha$, then we get for $t\geq t_1$
\begin{align}
    \dot{D}(t) = -\gamma \frac{D^2(t) -D_s^2}{4D(t)} -k_{p_2}D_{G_2}e^{-k_{p_2}t}
\end{align}
where $D(t_1)=\beta^1_{+}(t_1)>D_s$. Note that for $D(t)>D_s$, $\dot{D}(t)<0 \implies D(t)$ is monotonically decreasing at least until $t_c$ where $t_c \coloneqq \{t|D(t)=D_s\}$. Recall from phase 1 that the collision avoidance constraint of robot two is inactive \textit{i.e.} $f_{21}(\boldsymbol{\hat{u}}_2(t))<0$ as long as $D(t)>\beta^2_{+}(t)$. Additionally, recall that $\beta^2_{+}(t)$ is monotonically decreasing with respect to time and converges to $D_s$. Moreover,  from phase 1, recall that $D(t_1)=\beta^1_{+}(t_1)>\beta^2_{+}(t_1)$. Hence, there exists a time $t_2\leq t_c$ at which $D(t)=\beta^2_{+}(t)$. This time $t_2$ is precisely the time in \eqref{time_def2}. See  \cref{fig:phases} for an illustration of phase 2 which ends at $t=t_2$ when $D(t)=\beta^2_+(t)$.   This concludes phase 2 and marks the start of phase 3. \qed 
\end{proof}

\begin{lemma}
\label{lemma4_appendix}
The distance between robots converges to the safety margin $D_s$ at which point they stop moving and fall in deadlock.
\end{lemma}
\begin{proof}
Letting $\Delta \boldsymbol{p}_{21}(t)=D(t)\hat{\boldsymbol{e}}_\alpha$  in  \eqref{relativedynamicsphase3}
\begin{align}
    \dot{D}(t) = -\gamma \frac{D^2(t) -D_s^2}{2D(t)}
\end{align}
where $D(t_2)=\beta^2_{+}(t_2)>D_s$. The solution to this  is 
\begin{align}
\label{dtsolution2}
    D(t) = \sqrt{(D^2(t_2)-D_s^2)e^{-\gamma (t-t_2)} + D_s^2}
\end{align}
Therefore $D(t)\longrightarrow D_s$ and so $\norm{\Delta \boldsymbol{p}_{12}(t)}\longrightarrow D_s$. Therefore, the robots will be just at the verge of colliding.  Moreover, note from  \eqref{phasethreedynamics} that as $\norm{\Delta \boldsymbol{p}_{12}(t)}\longrightarrow D_s$, $\dot{\boldsymbol{p}}_1 \longrightarrow \boldsymbol{0}$, $\dot{\boldsymbol{p}}_2 \longrightarrow \boldsymbol{0}$, \textit{i.e.} the robots stop moving.  However, once they stop, note that $\lim_{t \longrightarrow \infty}\Delta \boldsymbol{p}_{21}(t)=D_s\hat{\boldsymbol{e}}_{\alpha}$, yet $\Delta \boldsymbol{p}_{d_{21}}=-(D_{G_1}+D_{G_2}-D_{init.})\hat{\boldsymbol{e}}_{\alpha}$. Therefore, the robots are not at their goals (because the goal vector is anti-parallel to the vector connecting the robots). Hence, we conclude that the robots have fallen in \textit{deadlock}.\qed
\end{proof}
\subsection{Three Robot Deadlock Results}
\begin{lemma}
\label{lemma5_appendix}
If  $D_{init.}>\beta_{+}$, then $\boldsymbol{u}^*_i(0) = \boldsymbol{\hat{u}}_i(0)$ $\forall i$
\end{lemma}
\begin{proof}
To keep calculations brief, we compute the flag $f_{ij}(\hat{\boldsymbol{u}}_i(0))$ for $i=1$, $j=2$. Using  \eqref{initial_positions_three_robots}-\eqref{goal_positions_three_robots}, we have
\begin{align}
    f_{12}(\hat{\boldsymbol{u}}_1(0))) &= \boldsymbol{a}_{12}^T(0)\boldsymbol{\hat{u}}_1(0) -b_{12}(0) \nonumber \\
    &= -\frac{\gamma}{4}D_{init}^2 + \frac{\sqrt{3}D_Gk_p}{2}D_{init} + \frac{\gamma}{4}D_s^2  \nonumber \\
    &\coloneqq g(D_{init.})
\end{align}
We can show that $f_{ij}(\hat{\boldsymbol{u}}_i) = g(D_{init.}) \mbox{ }  \forall i \in \{1,2,3\}, j \neq i$ because of the equilateral configuration at $t=0$. $g(D_{init.})$ is a quadratic polynomial in $D_{init.}$ with a zero at $\beta_+$: 
\begin{align}
\label{crit3_app}
    \beta_+ =\frac{\sqrt{3}D_Gk_p}{\gamma} + \sqrt{\bigg(\frac{\sqrt{3}D_Gk_p}{\gamma}\bigg)^2 + D_s^2}
\end{align}
Therefore, $D_{init}> \beta_+ \implies g(D_{init.})<0$ $\iff f_{ij}(\hat{\boldsymbol{u}}_i(0)) <0 \mbox{ } \forall i \in \{1,2,3\}   \mbox{, }  j \in \{1,2,3\}\backslash i$,  $\iff {\boldsymbol{u}}^*_i(0) = \hat{\boldsymbol{u}}_i(0)$ $\forall$ $i \in \{1,2,3\}$. \qed
\end{proof}
\begin{lemma}
\label{lemma6_appendix}
All robots continue to move in an equilateral triangular configuration with centroid fixed at $\boldsymbol{c}(0)$
\end{lemma}
\begin{proof}
Using  \eqref{threerobotpositions}, the centroid's position is given by
\begin{align}
    \boldsymbol{c}(t) &=  \underbrace{\frac{1}{3}\sum_{i=1}^3 (\boldsymbol{p}_i(0)-\boldsymbol{p}_{d_i})}_{\boldsymbol{0}}e^{-k_pt} + \underbrace{\frac{1}{3}\sum_{i=1}^3 \boldsymbol{p}_{d_i}}_{\boldsymbol{c}(0)} \nonumber \\
     &= \boldsymbol{c}(0)
\end{align}
\textcolor{black}{Here the two terms are zero because of the way we defined the initial position and goal locations in \eqref{initial_positions_three_robots}-\eqref{goal_positions_three_robots}. Next, to show that the three robots move along the vertices of an equilateral triangle, we compute the distance between robots $i$ and $j$ and show that it identical for every pair.
From \eqref{threerobotpositions}, }
\begin{align}
\label{sidelength}
    \norm{\Delta \boldsymbol{p}_{ij}(t)} 
    = &\sqrt{\underbrace{\norm{\Delta \boldsymbol{p}_{ij}(0)}^2}_{\mbox{term 1}}e^{-2k_pt} +  \underbrace{\norm{\Delta \boldsymbol{p}_{d_{ij}}}^2}_{\mbox{term 2}}(1-e^{-k_pt})^2}  \nonumber \\ 
    &+\overline{2 \underbrace{\Delta \boldsymbol{p}^T_{ij}(0)\Delta \boldsymbol{p}_{d_{ij}}}_{\mbox{term 3}}(e^{-k_pt}-e^{-2k_pt}) }
\end{align}
One can show that term 1, term 2 and term 3 are identical for all $i,j \in \{1,2,3\}, j\neq i$ using  \eqref{initial_positions_three_robots}, \eqref{goal_positions_three_robots}. Therefore, the distance of robot $i$ from robot $j \neq i$ is same for all $i\in \{1,2,3\}$, hence the robots move along the vertices of an equilateral triangle.  A second invariant is the angle made by the vector connecting robots $i,j$ with the $X_{w}$ axis of the world. This can be shown by demonstrating that $\Delta \boldsymbol{p}_{ij}(t)$ remains parallel to $\Delta \boldsymbol{p}_{ij}(0)$ as follows
\begin{align}
\label{parallel}
     \Delta \boldsymbol{p}_{ij}(t) \times \Delta \boldsymbol{p}_{ij}(0)& =  \underbrace{\big(\Delta \boldsymbol{p}_{ij}(0) \times \Delta \boldsymbol{p}_{ij}(0)\big)}_{\boldsymbol{0}}e^{-k_pt} \nonumber \\ 
    &+ \underbrace{\big(\Delta \boldsymbol{p}_{d_{ij}} \times \Delta \boldsymbol{p}_{ij}(0)\big)}_{\mbox{Term 1}} (1-e^{-k_pt}) \nonumber \\
    &= \boldsymbol{0}
\end{align}
Term 1 vanishes because  $\Delta \boldsymbol{p}_{d_{ij}}$ is anti-parallel to  $\Delta \boldsymbol{p}_{ij}(0)$ using  \eqref{initial_positions_three_robots}, \eqref{goal_positions_three_robots}. \textcolor{black}{Since now we have shown that the three robots move along the vertices of an equilateral triangle, the positions of the robots can still be written in the form similar to the way we defined their initial positions in 
\eqref{initial_positions_three_robots}}:
\begin{align}
\label{positionsasafunctionoftimethree}
    \boldsymbol{p}_1(t) &= e^{-k_pt}\boldsymbol{p}_1(0) + (1-e^{-k_p t})\boldsymbol{p}_{d_1} \nonumber \\
    \boldsymbol{p}_2(t) &= \boldsymbol{p}_1(t) + \tilde{D}(t)\hat{\boldsymbol{e}}_{\alpha} \nonumber \\
    \boldsymbol{p}_3(t) &= \boldsymbol{p}_2(t) + \tilde{D}(t)\hat{\boldsymbol{e}}_{\alpha + \frac{2\pi}{3}},
\end{align}
where $\tilde{D}(0)=D_{init.}$. \textcolor{black}{$\tilde{D}(t)$ denotes the inter-robot distance between any pair of robots.} Here, the angle between $\Delta \boldsymbol{p}_{21}$ and $X_w$ is still $\alpha$ as in  \eqref{initial_positions_three_robots} because of  \eqref{parallel}. Thus, the robots move along the vertices of an equliateral triangle using $\hat{\boldsymbol{u}}_i(t)$ $\forall i \in \{1,2,3\}$. \textit{This symmetry is because (1)  $\hat{\boldsymbol{u}}_i(t)$ has identical gains (\textit{i.e.} $k_p$) $\forall i$ and (2) the distance of initial position to goal is identical for all robots ($D_G$)}. \qedsymbol
\end{proof}

\hspace{-0.4cm}Note that in  \eqref{positionsasafunctionoftimethree}, $\tilde{D}(t)$ is
\begin{align}
\label{Dtilde}
    \tilde{D}(t) &=\hat{\boldsymbol{e}}^T_{\alpha}\Delta \boldsymbol{p}_{21}(t)  \nonumber \\
    &= \hat{\boldsymbol{e}}^T_{\alpha}\bigg(\Delta \boldsymbol{p}_{21}(0)e^{-k_pt} + \Delta \boldsymbol{p}_{d_{21}}(1-e^{-k_pt})\bigg) \nonumber \\
    & = (D_{init.}-\sqrt{3}D_G) + \sqrt{3}D_Ge^{-k_pt}
\end{align}
Using this definition of $\tilde{D}(t)$, we now demonstrate that there exists a time when all the collision avoidance constraints will inevitably become active.
\begin{lemma}
\label{lemma7_appendix}
There exists a time $t_{ij}$ when $ f_{ij}(\boldsymbol{\hat{u}}_i(t)) =\boldsymbol{a}^T_{ij}(t)\boldsymbol{\hat{u}}_i(t) - b_{ij}(t)=0$ \textit{i.e.} when the collision avoidance constraint of robot $i$ with robot $j$ becomes active. Furthermore, $t_{ij}$ is identical $\forall i \in \{1,2,3\}, j \in \{1,2,3\}\backslash i$
\end{lemma}
\begin{proof}
Using \eqref{threerobotpositions}, we can find that $f_{ij}(\hat{\boldsymbol{u}}_i(t))$ (as a function of time) is identical for$ \mbox{ }\forall i \in \{1,2,3\},\mbox{ } j\in \{1,2,3\}\backslash i$. This is again due to the symmetry in the positions at time $t$. Now, we evaluate  $f_{ij}(\boldsymbol{\hat{u}}_i)$ as a function of $\tilde{D}$. For brevity, we evaluate this for $i=1$, $j=2$:
\begin{align}
\label{parabola2}
    g(\tilde{D})&\coloneqq f_{12}(\boldsymbol{\hat{u}}_1) =\boldsymbol{a}^T_{12}\boldsymbol{\hat{u}}_1 - b_{12} \nonumber \\
    &=-\frac{\gamma}{4}\tilde{D}^2(t)+\frac{\sqrt{3}}{2}k_{p}D_{G}e^{-k_pt}\tilde{D}(t)+\frac{\gamma}{4}D^2_s
\end{align}
Note that $g(\tilde{D})$ is quadratic in $\tilde{D}$ with a zero at
\begin{align}
\label{beta3}
    \beta_+(t) =\frac{\sqrt{3}D_Gk_pe^{-k_pt}}{\gamma} + \sqrt{ \bigg( \frac{\sqrt{3}D_Gk_pe^{-k_pt}}{\gamma}\bigg)^2 + D_s^2}  
\end{align}
Therefore, if $\tilde{D}(t) \leq \beta_{+}(t) \implies f_{12}(\hat{\boldsymbol{u}}_1(t))\geq0$ . Now, note from  \eqref{Dtilde} that $\tilde{D}(t)$ is monotonically decreasing with $t$, $\tilde{D}(0)=D_{init.}>\beta_{+}$ (from Assumption \eqref{ass4}). Also, note that $\lim_{t \longrightarrow \infty}\tilde{D}(t)=D_{init.}-\sqrt{3}D_G<0$ from Assumption \eqref{ass3}. Similarly, $\beta_{+}(t)$ is monotonically decreasing. $\beta_+(0)<D_{init.}$ and $\lim_{t \longrightarrow \infty}\beta_{+}(t)=D_s>0$.  Therefore, there exists a time when $\tilde{D}(t)$ intersects $\beta_{+}(t)$ \textit{i.e.} $t_1=\{t|\tilde{D}(t)=\beta_{+}(t)\}$. This is equivalent to $t_1=\{t|f_{ij}(\hat{\boldsymbol{u}}_i(t))=0 \mbox{ }\forall i \in \{1,2,3\},\mbox{ } j \in \{1,2,3\}\backslash i\}$.  This is precisely when constraints of \textbf{all} robots become active. \qed
\end{proof}

\begin{lemma}
\label{lemma8_appendix}
The distance between the robots converges to the safety margin $D_s$ and the robots stop moving and fall in deadlock.
\end{lemma}
\begin{proof}
Noting from  \eqref{initialpositions_attina2} that $\Delta \boldsymbol{p}_{12}(t)=-D(t)\hat{\boldsymbol{e}}_{\alpha}$ and combining with  \eqref{deadlock_12_rhs2}, we deduce that
\begin{align}
\label{dtsolution}
    \dot{D}(t) &= -\gamma \frac{D^2(t)-D_s^2}{2D(t)} \hspace{0.5cm}\mbox{      }\forall t\geq t_1 \nonumber \\
\implies  D(t) &= \sqrt{(D^2(t_1)-D_s^2)e^{-\gamma (t-t_1)} + D_s^2} \nonumber \\
\implies \lim_{t \longrightarrow \infty}D(t)&=D_s
\end{align}
Hence, it follows from \eqref{deadlock_12_rhs1} that $\dot{\boldsymbol{p}_i}\longrightarrow \boldsymbol{0}$ $\forall i \in \{1,2,3\}$. Moreover, note that $\lim_{t \longrightarrow \infty}\Delta \boldsymbol{p}_{12}(t)=-D_s\hat{\boldsymbol{e}}_{\alpha}$ yet $\Delta \boldsymbol{p}_{d_{12}}=(\sqrt{3}D_G-D_{init.})\hat{\boldsymbol{e}}_{\alpha} \neq -D_s\hat{\boldsymbol{e}}_{\alpha}$  (from \eqref{goal_positions_three_robots} and Assumption \ref{ass3}). Therefore, the robots are not at their goals, and static, thus they have fallen in \textit{deadlock}. \qed
\end{proof}
\subsection{Deadlock Resolution Algorithm Proof}
\begin{theorem}
\label{finalthm_appendix}
The three phase strategy ensures that the robots will not fall back in deadlock and will converge to their goals.
\end{theorem}
\begin{proof}
We would like to show that once phase three control begins, the robots will never fall back in deadlock. We will do this by showing that the distance between the robots is non-decreasing, once phase three control starts. We make one assumption which is needed for a technicality in the proof but is easily achieved in practice.
\begin{assumption}
\label{intergoal}
$D_G>D_s$ \textit{i.e.} the inter-goal distance is greater than the safety margin.
\end{assumption}
This is required because otherwise, the robots will never be at-least $D_s$ apart while at their goals which will result in safety violation. We break this proof into three parts consistent with the three phases shown in \cref{fig:drestwo}: \\ \\
\textbf{Phase 1 $\rightarrow$ Phase 2:}
Let $t=t_1$ be the time at which phase 1 ends (and phase 2 starts) \textit{i.e.} when robots fall in deadlock. In \cref{touching} we showed that in deadlock $\norm{\Delta \boldsymbol{p}_{21}}=D_s$, and in \cref{nonempty} we showed that the positions of robots and their goals are collinear. So at the end of phase 1, $\Delta \boldsymbol{p}_{21}(t_{1})=D_s\boldsymbol{\hat{e}}_{\beta + \pi}$. The goal vector $\Delta \boldsymbol{p}_{d_{21}}(t)\coloneqq \boldsymbol{p}_{d_2}-\boldsymbol{p}_{d_1}=D_G\boldsymbol{\hat{e}}_{\beta} \mbox{ }\forall t>0$. \\  \\
\textbf{Phase 2 $\longrightarrow$ Phase 3:}
The initial condition of phase two is the final condition of phase one  \textit{i.e.} $\Delta \boldsymbol{p}_{21}(t_{1})=D_s\hat{e}_{\beta +\pi}$. In phase two, we use feedback linearization to rotate the assembly of robots making sure that the distance between them stays at $D_s$, until the orientation of the vector $\Delta \boldsymbol{p}_{21}(t)=D_s \boldsymbol{\hat{e}}_{\theta(t)}$ aligns with  $\Delta \boldsymbol{p}_{d_{21}}=D_G\boldsymbol{\hat{e}}_{\beta}$. (Such a controller is guaranteed to exist (see \cite{grover2019deadlock}). Once done, $\exists$ a time $t_2$ at which $\theta(t_2)=\beta$. Moreover, at $t=t_2$, the robots are no longer moving and $\Delta \boldsymbol{p}_{21}(t_2)=D_s \boldsymbol{\hat{e}}_{\beta}$. These states are the final condition for phase 2 and initial condition for phase 3.\\ \\
\textbf{Phase 3 $\longrightarrow \infty:$}
In this phase, the initial condition is $\Delta \boldsymbol{p}_{21}(t_2)=D_s \boldsymbol{\hat{e}}_{\beta}$. Also, note that the dynamics of phase 3 control are specified by a proportional controller. The dynamics of relative positions and velocities are:
\begin{align}
\Delta \dot{\boldsymbol{p}}_{21} = -k_p(\Delta \boldsymbol{p}_{21}- \Delta \boldsymbol{p}_{d_{21}})
\end{align}
where $\Delta \boldsymbol{p}_{d_{21}}=D_G\boldsymbol{\hat{e}}_{\beta}$.  Now, we will do a coordinate change as described next. Let $\Delta \tilde{\boldsymbol{p}}_{21} \coloneqq R_{-\beta} \Delta {\boldsymbol{p}}_{21}$. The initial conditions in these coordinates are $\Delta \tilde{\boldsymbol{p}}_{21}(t_2)= R_{-\beta} D_s\boldsymbol{\hat{e}}_{\beta}=(D_s,0)$ \textit{i.e.} $\Delta \tilde{p}_{21}^x(t_2)=D_s,\Delta \tilde{p}_{21}^y(t_2)=0$. The dynamics in new coordinates are:
\begin{align}
\Delta \dot{\tilde{\boldsymbol{p}}}_{21} = -k_p(\Delta \tilde{\boldsymbol{p}}_{21}- R_{-\beta}\Delta \boldsymbol{p}_{d_{21}})
\end{align}
Using these coordinates, note that $R_{-\beta}\Delta \boldsymbol{p}_{d_{21}}=(D_G,0)$. 
Note from the dynamics and the initial conditions for the $y$ components of relative position that the only solution is the zero solution \textit{i.e.} $\Delta \tilde{p}_{21}^y(t) \equiv 0$ $\forall$ $t\geq t_2$. As for the $x$ component, we can compute the solution to be $\Delta \tilde{p}_{21}^x(t)=D_G + (D_s-D_G)e^{-k_p(t-t_2)}$.
Finally, note that $\frac{d \norm{\Delta \boldsymbol{p}_{12}(t)}}{dt} = -k_p(D_s-D_G)e^{-k_p(t-t_2)}  > 0$ (from Assumption \ref{intergoal}). Hence, the distance between the robots is non-decreasing \textit{i.e.} the robots never fall in deadlock. Additionally, since the robots use a proportional controller, their positions exponentially stabilize to their goals. \qed
\end{proof} 
\bibliographystyle{agsm}
\bibliography{main}

\end{document}